\documentclass{article}

\usepackage{times}
\usepackage{graphicx} 
\usepackage{subcaption}

\usepackage{natbib}

\usepackage{algorithm}
\usepackage{algorithmic}

\usepackage{hyperref}

\usepackage[accepted]{icml2017}

\newenvironment{packed_enum}{
  \begin{enumerate}
    \setlength{\itemsep}{1pt}
    \setlength{\parskip}{-1pt}
    \setlength{\parsep}{0pt}
}{\end{enumerate}}

\newcounter{qcounter}
 {\end{list}}

\usepackage[utf8]{inputenc} 
\usepackage[T1]{fontenc}    
\usepackage{booktabs}       
\usepackage{microtype}      
\usepackage{xspace}

\usepackage{amsmath,amsthm,amssymb,amsfonts,bbm}
\usepackage{mathtools}
\usepackage{xr-hyper}
\usepackage{enumitem}
\usepackage{verbatim,float,url}

\usepackage{authblk}

\usepackage[]{color-edits}
\addauthor{md}{Cerulean}
\addauthor{yw}{Purple}
\addauthor{aa}{Red}

\newtheorem{theorem}{Theorem}
\newtheorem{lemma}{Lemma}
\newtheorem{assumption}{Assumption}
\newtheorem{corollary}{Corollary}
\newtheorem{proposition}{Proposition}
\theoremstyle{remark}
\newtheorem{remark}{Remark}
\theoremstyle{definition}

\newcommand{\argmin}{\mathop{\mathrm{argmin}}}
\newcommand{\argmax}{\mathop{\mathrm{argmax}}}
\newcommand{\minimize}{\mathop{\mathrm{minimize}}}
\newcommand{\maximize}{\mathop{\mathrm{maximize}}}

\def\E{\mathbb{E}}
\def\P{\mathbb{P}}

\def\Var{\mathrm{Var}}

\def\R{\mathbb{R}}
\def\cA{\mathcal{A}}

\def\cD{\mathcal{D}}

\def\cF{\mathcal{F}}

\def\cI{\mathcal{I}}

\def\cL{\mathcal{L}}

\def\cN{\mathcal{N}}
\def\cP{\mathcal{P}}

\def\cR{\mathcal{R}}

\def\cU{\mathcal{U}}

\def\cX{\mathcal{X}}

\def\cZ{\mathcal{Z}}

\def\1{\mathbf{1}}

\def\MSE{\mathrm{MSE}}

\newcommand{\figsqueeze}{\vspace{-16pt}}
\newcommand{\figsqueezeB}{\vspace{-3pt}}

\newcommand{\red}[1]{\textcolor{red}{#1}}
\newcommand{\blue}[1]{\textcolor{blue}{#1}}
\newcommand{\todo}[1]{}
\newcommand{\alekh}[1]{}

\newcommand{\dx}{\ensuremath{\lambda}}
\newcommand{\er}{\ensuremath{\eta}}
\newcommand{\rmax}{\ensuremath{R_{\max}}}
\def\cZna{\mathcal{Z}^{\textup{na}}}
\def\cZa{\mathcal{Z}^{\textup{a}}}
\newcommand{\Z}{\mathbb{Z}}
\newcommand{\cJ}{\mathcal{J}}

\newcommand{\set}[1]{\{#1\}}

\newcommand{\Braces}[1]{\left\{#1\right\}}
\newcommand{\bigBraces}[1]{\bigl\{#1\bigr\}}
\newcommand{\BigBraces}[1]{\Bigl\{#1\Bigr\}}
\newcommand{\bracks}[1]{[#1]}
\newcommand{\bigBracks}[1]{\bigl[#1\bigr]}
\newcommand{\BigBracks}[1]{\Bigl[#1\Bigr]}
\newcommand{\biggBracks}[1]{\biggl[#1\biggr]}
\newcommand{\BiggBracks}[1]{\Biggl[#1\Biggr]}
\newcommand{\Bracks}[1]{\left[#1\right]}
\newcommand{\Parens}[1]{\left(#1\right)}
\newcommand{\bigParens}[1]{\bigl(#1\bigr)}
\newcommand{\BigParens}[1]{\Bigl(#1\Bigr)}
\newcommand{\biggParens}[1]{\biggl(#1\biggr)}

\newcommand{\given}{\mathbin{\vert}}
\newcommand{\bigGiven}{\mathbin{\bigm\vert}}

\newcommand{\Abs}[1]{\left\lvert#1\right\rvert}
\newcommand{\bigAbs}[1]{\bigl\lvert#1\bigr\rvert}

\newcommand{\Thm}[1]{Theorem~\ref{thm:#1}}
\newcommand{\Sec}[1]{Section~\ref{sec:#1}}
\newcommand{\Eq}[1]{Eq.~\eqref{eq:#1}}
\newcommand{\Lem}[1]{Lemma~\ref{lem:#1}}

\newcommand{\OA}{\textsc{switch}\xspace}
\newcommand{\OADR}{\textsc{switch-DR}\xspace}
\newcommand{\term}{\ensuremath{\mathcal{T}}}
\newcommand{\order}{\ensuremath{\mathcal{O}}}
\newcommand{\lmax}{\gamma}
\newcommand{\hR}{\hat{R}}
\newcommand{\hrho}{\hat{\rho}}

\graphicspath{{../plots/}}

\begin{document}

\twocolumn[

\icmltitle{Optimal and Adaptive Off-policy Evaluation in Contextual
  Bandits}

\begin{icmlauthorlist}
\icmlauthor{Yu-Xiang Wang}{cmu}
\icmlauthor{Alekh Agarwal}{msr}
\icmlauthor{Miroslav Dud{\'\i}k}{msr}
\end{icmlauthorlist}
\icmlaffiliation{cmu}{Carnegie Mellon University, Pittsburgh, PA}
\icmlaffiliation{msr}{Microsoft Research, New York, NY}

\icmlcorrespondingauthor{Yu-Xiang Wang}{yuxiangw@cs.cmu.edu}
\icmlcorrespondingauthor{Alekh Agarwal}{alekha@microsoft.com}
\icmlcorrespondingauthor{Miroslav Dud\'ik}{mdudik@microsoft.com}

\icmlkeywords{contextual bandits, off-policy evaluation, minimax risk, adaptive estimator}

\vskip 0.3in
]

\printAffiliationsAndNotice{}

\begin{abstract}
  We study the off-policy evaluation problem---estimating the value of
  a target policy using data collected by another policy---under the
  contextual bandit model. We consider the general (agnostic) setting
  without access to a consistent model of rewards and establish a
  minimax lower bound on the mean squared error (MSE). The bound is
  matched up to constants by the inverse propensity scoring (IPS) and
  doubly robust (DR) estimators. This highlights the difficulty of the
  agnostic contextual setting, in contrast with multi-armed
  bandits and contextual bandits with access to a consistent
    reward model, where IPS is suboptimal.  We then propose the \OA
  estimator, which can use an existing reward model (not
    necessarily consistent) to achieve a better bias-variance
  tradeoff than IPS and DR. We prove an upper bound on its MSE and
  demonstrate its benefits empirically on a diverse collection of
  data sets, often outperforming prior work by orders of
  magnitude.
\end{abstract}


\section{Introduction}
\label{sec:intro}

Contextual bandits refer to a learning setting where the learner
repeatedly observes a context, takes an action and observes a reward
for the chosen action in the observed context, \emph{but no feedback
  on any other action}. An example is movie recommendation, where the
context describes a user, actions are candidate movies and the reward
measures if the user enjoys the recommended movie. The learner
produces a policy, meaning a mapping from contexts to actions. A common
question in such settings is, given a \emph{target policy}, what is
its expected reward? By letting the policy choose actions
(e.g., recommend movies to users), we can compute its reward. Such
\emph{online evaluation} is typically costly since it exposes users to
an untested experimental policy, and does not scale to evaluating many
different target policies.

\emph{Off-policy evaluation} is an alternative paradigm for the same
question. Given logs from the existing system, which might be choosing
actions according to a very different \emph{logging policy} than the one we
seek to evaluate, can we estimate the expected reward of the
\emph{target policy}? There are three classes of approaches to address
this question: the \emph{direct method} (DM), also known as regression
adjustment, \emph{inverse propensity scoring}
(IPS)~\citep{horvitz1952generalization} and \emph{doubly robust} (DR)
estimators~\citep{robins1995semiparametric,bang2005doubly,dudik2011doubly,dudik2014doubly}.

Our first goal in this paper is to study the optimality of these three classes of approaches (or lack thereof), and more
fundamentally, to quantify the statistical hardness of off-policy
evaluation.
This problem was previously studied for multi-armed bandits \citep{li2015toward}
and is related to a large body of work on asymptotically optimal estimators of average treatment effects (ATE)
\citep{hahn1998role,hirano2003efficient,ImbensNeRi07,rothe2016value}, which can be viewed as a special
case of off-policy evaluation.
In both settings, a major underlying assumption is that rewards can be consistently estimated from the features (i.e.,
covariates) describing contexts and actions, either via a parametric
model or non-parametrically.
Under such consistency assumptions, it
has been shown that DM and/or DR are optimal~\citep{ImbensNeRi07,li2015toward,rothe2016value},\footnote{The
  precise assumptions vary for each estimator, and are somewhat weaker
  for DR than for DM.} whereas standard IPS is
not~\citep{hahn1998role,li2015toward}, but it becomes (asymptotically)
optimal when the true propensity scores are replaced by suitable
estimates~\citep{hirano2003efficient}.

Unfortunately, consistency of a reward model can be difficult to
achieve in practice. Parametric models tend to suffer from a large
bias (see, e.g., the empirical evaluation of
\citealp{dudik2011doubly}) and non-parametric models are limited to
small dimensions, otherwise non-asymptotic terms become too large
(see, e.g., the analysis of non-parametric regression by
\citealp{bertin2004asymptotically}).  Therefore, here we ask:
\emph{What can be said about hardness of policy evaluation in the
  absence of reward-model consistency?}

In this pursuit, we provide the first rate-optimal lower bound on the
mean-squared error (MSE) for off-policy evaluation in contextual
bandits without consistency assumptions. Our lower bound matches the
upper bounds of IPS and DR up to constants, when given a
non-degenerate context distributions. This result is in contrast with the
suboptimality of IPS under previously studied consistency assumptions,
which implies that the two settings are qualitatively different.

Whereas IPS and DR are both minimax optimal, our experiments (similar
to prior work) show that IPS is readily outperformed by DR, even
when using a simple parametric regression model that is not asymptotically
consistent. We attribute this to a lower variance of
the DR estimator. We also empirically observe that while DR is generally
highly competitive, it is sometimes substantially outperformed by DM.
We therefore ask whether it
is possible to achieve an even better bias-variance tradeoff than DR.
%
%
%
We answer affirmatively and
propose a new class of estimators, called the \OA estimators, that
\emph{adaptively interpolate} between DM and DR (or IPS). We show that
\OA has MSE no worse than DR (or IPS) in the worst case, but is robust
to large importance weights and can achieve a substantially
smaller variance than DR or IPS.

We empirically evaluate the \OA estimators against a number of strong
baselines from prior work, using a previously used experimental setup
to simulate contextual bandit problems on real-world multiclass
classification data. The results affirm the superior bias-variance
tradeoff of \OA estimators, with substantial improvements across a
number of problems.

In summary, the first part of our paper initiates the study of optimal
estimators in a finite-sample setting and without making strong
modeling assumptions, while the second part shows how to practically
exploit domain knowledge by building better estimators.

\section{Setup}
\label{sec:formulation}

In contextual bandit problems, the learning agent observes a context
$x$, takes an action $a$ and observes a scalar reward $r$ for the action
chosen in the context. Here the context $x$ is a feature vector from
some domain $\cX \subseteq \R^d$, drawn according to a distribution
$\dx$. Actions $a$ are drawn from a finite set~$\cA$. Rewards $r$ have
a distribution conditioned on $x$ and $a$ denoted by $D(r\given x,a)$. The
decision rule of the agent is called a policy, which maps contexts to
distributions over actions, allowing for randomization in the action choice. We
write $\mu(a\given x)$ and $\pi(a\given x)$ to denote the \emph{logging} and
\emph{target} policies respectively. Given a policy $\pi$, we extend
it to a joint distribution over $(x,a,r)$, where $x \sim \dx$, action
$a \sim \pi(a\given x)$, and $r\sim D(r\given x,a)$. With this notation, given $n$
i.i.d.\ samples $(x_i, a_i, r_i)\sim\mu$,
we wish to compute the \emph{value of $\pi$}:
\begin{equation}
  v^\pi = \E_\pi[r] = \E_{x \sim \dx} \E_{a \sim\pi(\cdot\given x)} \E_{r
    \sim D(\cdot\given a,x)}[r].
  \label{eq:vpi}
\end{equation}
In order to correct for the mismatch in the action distributions under
$\mu$ and $\pi$, it is typical to use \emph{importance weights},
defined as $\rho(x,a)\,{\coloneqq}\,
\pi(a\given x)/\mu(a\given x)$. For consistent estimation, it is standard to assume
that $\rho(x,a) \ne \infty$, corresponding to \emph{absolute continuity} of $\pi$ with
respect to $\mu$, meaning that whenever $\pi(a\given x) > 0$,
then also $\mu(a\given x)> 0$. We make this assumption throughout
the paper. In the remainder of the setup we present three common estimators of
$v^\pi$.

The first is the inverse
propensity scoring (IPS) estimator \citep{horvitz1952generalization}, defined
as
\begin{equation}
  \hat{v}^\pi_{\mathrm{IPS}} = \sum_{i=1}^n \rho(x_i, a_i) r_i.
  \label{eq:ips}
\end{equation}
IPS is unbiased and makes no assumptions about how rewards might
depend on contexts and actions. When such information is available, it
is natural to posit a parametric or non-parametric model of
$\E[r\given x,a]$ and fit it on the logged data to obtain a reward
estimator $\hat{r}(x,a)$. Policy evaluation can now simply be performed by
scoring $\pi$ according to $\hat{r}$ as
%
\begin{equation}
\label{eq:DM}
\hat{v}^\pi_{\text{DM}} = \frac{1}{n}\sum_{i=1}^n \sum_{a\in \cA}
\pi(a\given x_i) \hat{r}(x_i,a),
\end{equation}
where the DM stands for \emph{direct method}~\citep{dudik2011doubly},
also known as \emph{regression adjustment} or
\emph{imputation}~\citep{rothe2016value}. IPS can have a large
variance when the target and logging policies differ substantially,
and parametric variants of DM can be inconsistent, leading to a large
bias. Therefore, both in theory and practice, it is beneficial to combine the
approaches into a \emph{doubly robust}
estimator~\citep{cassel1976some,robins1995semiparametric,dudik2011doubly},
such as the following variant,
\begin{align}
\notag
\hat{v}^\pi_{\text{DR}}
&=\frac{1}{n} \sum_{i=1}^n \biggl[ \rho(x_i,a_i)\bigParens{r_i -\hat{r}(x_i, a_i)}
\\[-6pt]
\label{eq:DR}
&\qquad\qquad\qquad{}
   + \sum_{a \in \cA} \pi(a\given x_i)
   \hat{r}(x_i, a)\biggr].
\end{align}
Note that IPS is a special case of DR with $\hat{r}\equiv 0$. In the
sequel, we mostly focus on IPS and DR, and then suggest how to improve
them by further interpolating with DM.

\section{Limits of Off-policy Evaluation}
\label{sec:minimax}

In this section, we study the off-policy evaluation problem in a
minimax setup. After setting up the framework, we present our lower
bound and the matching upper bounds for IPS and DR under appropriate
conditions.

While minimax optimality is standard
in statistical estimations, it is not the only notion of optimality.
An alternative framework is that of asymptotic
optimality, which establishes Cramer-Rao style bounds on the asymptotic variance
of estimators. We use the minimax framework, because
it is the most amenable to finite-sample lower bounds, and
is complementary to previous asymptotic results, as we discuss
after presenting our main results.

\subsection{Minimax Framework}

Off-policy evaluation is a statistical estimation problem, where the
goal is to estimate $v^\pi$ given $n$ i.i.d.\ samples generated according
to a policy $\mu$.  We study this problem in a standard minimax
framework and seek to answer the following question. What is the
smallest MSE that \emph{any} estimator can achieve in the worst case
over a large class of contextual bandit problems? As is usual in the
minimax setting, we want the class of problems to be rich enough so
that the estimation problem is not trivial, and to be small enough
so that the lower bounds are not driven by complete pathologies. In
our problem, we
fix $\dx$, $\mu$ and $\pi$, and
only take worst case over a class of reward distributions. This allows
the upper and lower bounds
to depend on $\dx$, $\mu$ and $\pi$, highlighting how these ground-truth
parameters influence the problem difficulty. The family
of reward distributions $D(r\given x,a)$ that we study is a natural
generalization of the class studied by \citet{li2015toward} for
multi-armed bandits.
We assume we are given maps $R_{\max}~:~\cX\times \cA \to
\mathbb{R}_+$ and $\sigma~:~\cX\times \cA \to \mathbb{R}_+$, and
define the class of reward distributions $\cR(\sigma,R_{\max})$
as\footnote{Technically, the inequalities in the definition
  of $\cR(\sigma,R_{\max})$
  need to hold almost surely with $x \sim \dx$ and $a \sim \mu(\cdot\given x)$.}
\begin{align*}
  \cR(\sigma,R_{\max})
  \!\coloneqq\!
  \Bigl\{
    D(r|x,a):\:
    0\,{\leq}\,\E_D[r|x,a]\,{\leq}\,R_{\max}(x,a)
\\
    \text{ and }
    \Var_D[r|x,a]\,{\leq}\,\sigma^2(x,a)
    \text{ for all }x,a
  \Bigr\}
.
\end{align*}
Note that $\sigma$ and $R_{\max}$ are allowed to change over contexts
and actions.  Formally, an estimator is any function $\hat{v}:
(\cX\times\cA\times\R)^n \rightarrow \R$ that takes $n$ data points
collected by $\mu$ and outputs an estimate of $v^\pi$. The
\emph{minimax risk} of off-policy evaluation over the class
$\cR(\sigma,R_{\max})$, denoted by
$R_n(\pi;\dx,\mu,\sigma,R_{\max})$, is defined as
\begin{equation}
\label{eq:minimax_def}
\adjustlimits\inf_{\hat{v}}\sup_{\;\;D(r|x,a) \in \cR(\sigma,R_{\max})\;\;}
\E\left[(\hat{v}-v^\pi)^2\right].
\end{equation}
Recall that the expectation is taken over the $n$ samples drawn from
$\mu$, along with any randomness in the estimator. The main goal of
this section is to obtain a lower bound on the minimax risk. To state
our bound, recall that $\rho(x,a)=\pi(a\given x)/\mu(a\given
x)<\infty$ is an importance weight at $(x,a)$. We make the following
technical assumption on our problem instances, described by tuples of
the form $(\pi,\dx,\mu,\sigma,R_{\max})$:
\begin{assumption}
\label{ass:moment}
There exists $\epsilon>0$ such that
$\E_{\mu}\Bracks{(\rho\sigma)^{2+\epsilon}}$
and
$\E_{\mu}\Bracks{(\rho\rmax)^{2+\epsilon}}$
are finite.
\end{assumption}
This assumption is only a slight strengthening of the assumption that
$\E_{\mu}[(\rho\sigma)^2]$ and $\E_{\mu}[(\rho\rmax)^2]$ be finite,
which is required for consistency of IPS (see,
e.g.,~\citealp{dudik2014doubly}). Our assumption holds for instance
when the context space is finite, because then both $\rho$ and $\rmax$
are bounded.

\subsection{Minimax Lower Bound for Off-policy Evaluation}
\label{sec:lb}

With the minimax setup in place, we now give our main lower bound on
the minimax risk for off-policy evaluation and discuss its
consequences. Our bound depends on a parameter $\lmax\in[0,1]$
and a derived indicator random variable
$\xi_{\lmax}(x,a)\coloneqq\1(\mu(x,a)\le\lmax)$, which
  separates out the pairs $(x,a)$ that appear ``frequently'' under
  $\mu$.%
\footnote{Formally, $\mu(x,a)$ corresponds to $\mu(\,\set{(x,a)}\,)$, i.e.,
  the measure under $\mu$ of the set $\set{(x,a)}$. For example, when
  $\dx$ is a continuous distribution then $\mu(x,a)=0$ everywhere.}
As we will see, the
``frequent'' pairs $(x,a)$ (where $\xi_{\lmax}=0$) correspond to the
intrinsically realizable part of the problem, where consistent reward
models can be constructed. The ``infrequent'' pairs
(where $\xi_{\lmax}=1$) constitute the part that is non-realizable in the
worst-case. When $\cX\subseteq\R^d$ and $\dx$ is continuous with
respect to the Lebesgue measure, then $\xi_{\lmax}(x,a)=1$ for all
$\lmax\in[0,1]$, so the problem is non-realizable everywhere in the
worst-case. Our result uses the following problem-dependent constant (defined with the
convention $0/0=0$):
%
\[
C_{\lmax}\coloneqq 2^{2+\epsilon} \max\bigg\{
  \frac{\E_\mu[(\rho\sigma)^{2+\epsilon}]^2}{\E_\mu[(\rho\sigma)^2]^{2+\epsilon}}
  ,\,
  \frac{\E_\mu[\xi_{\lmax}(\rho\rmax)^{2+\epsilon}]^2}{\E_\mu[\xi_{\lmax}(\rho\rmax)^2]^{2+\epsilon}}
  \bigg\}
.
\]
%
\begin{theorem}
\label{thm:minimax}
  Assume that a problem instance satisfies Assumption~\ref{ass:moment}
  with some $\epsilon>0$. Then for any $\lmax\in[0,1]$ and any
  $n\ge\max\bigBraces{16C_{\lmax}^{1/\epsilon},\,
    2C_{\lmax}^{2/\epsilon}\E_\mu[\sigma^2/\rmax^2] }$, the minimax
  risk $R_n(\pi;\dx,\mu,\sigma,R_{\max})$ satisfies the lower bound
\[
   \frac{
    \E_\mu\bigBracks{\rho^2\sigma^2} + \E_\mu\bigBracks{\xi_{\lmax}\rho^2\rmax^2}
    \BigParens{1-350n\lmax\log(5/\lmax)}
   }
   {700n}
\enspace.
\]
\end{theorem}

The bound holds for every $\lmax\in [0,1]$, and we can take the
maximum over $\lmax$. In particular, we get the following simple
corollary under continuous context distributions.

\begin{corollary}\label{cor:continuous}
  Under conditions of Theorem~\ref{thm:minimax}, assume further that
  $\dx$ has a density relative to Lebesgue measure. Then
  \[
  R_n(\pi;\dx,\mu,\sigma,R_{\max}) \geq
  \frac{ \E_\mu\bigBracks{\rho^2\sigma^2} +
    \E_\mu\bigBracks{\rho^2\rmax^2}
  }
  {700n}
  \enspace.
\]
\end{corollary}
If $\dx$ is a mixture of a density and point masses, then $\lmax=0$
will exclude the point masses from the second term of the lower bound.
In general, choosing $\gamma=\order\bigParens{1/(n\log n)}$
  excludes the contexts likely to appear multiple times, and ensures
  that the second term in \Thm{minimax} remains non-trivial (when
  $\mu(x,a)\le\gamma$ with positive probability).


Before sketching the proof of \Thm{minimax}, we discuss its preconditions
and implications.

\textbf{Preconditions of the theorem:} The theorem assumes the
existence of a (problem-dependent) constant $C_{\lmax}$ which depends
on the constant $\lmax$ and various moments of the importance-weighted rewards.
When $\rmax$ and $\sigma$ are bounded (a common situation), $C_{\lmax}$ measures how
heavy-tailed the importance weights are.
Note that $C_{\lmax} < \infty$ for all $\lmax \in [0,1]$
whenever Assumption~\ref{ass:moment} holds, and so
the condition on $n$ in \Thm{minimax} is eventually satisfied as long as the
random variable $\sigma/\rmax$ has a bounded second moment.
This is quite reasonable since in typical applications the \emph{a priori} bound on expected
rewards is on the same order or larger than the \emph{a priori} bound on the reward noise.
For the remainder of the discussion, we assume that $n$
is appropriately large so the preconditions of the theorem hold.

\textbf{Comparison with upper bounds:} The setting of
Corollary~\ref{cor:continuous} is typical of many contextual bandit
applications. In this setting both IPS and DR achieve the minimax risk
up to a multiplicative constant. Let $r^*(x,a)\coloneqq\E[r\given
  x,a]$. Recall that DR is using an estimator $\hat{r}(x,a)$ of
$r^*(x,a)$, and IPS can be viewed as a special case of DR with
$\hat{r}\equiv 0$. By Lemma 3.3(i) of \citet{dudik2014doubly}, the MSE
of DR is
\begin{align}
\notag
&
\E\bracks{(\hat{v}^\pi_{\mathrm{DR}}-v^\pi)^2}
\\
\notag
&\quad{}
 =\frac1n
  \Bigl(
    \E_\mu[ \rho^2\sigma^2]
    + \Var_{x\sim D}\E_{a\sim\mu(\cdot\given x)}[ \rho r^*]
\\[-3pt]
\label{eq:DR:risk}
&\qquad\qquad + \E_{x\sim D}\Var_{a\sim\mu(\cdot\given x)}[ \rho (\hat{r}-r^*)]
  \Bigr).
\end{align}
Note that $0\le r^*\le\rmax$, so if the estimator $\hat{r}$ also satisfies
$0\le\hat{r}\le\rmax$, we obtain that
the risk of DR (with IPS as a special case)
is at most $\order\bigParens{\frac1n(\E_\mu[ \rho^2\sigma^2]+\E_\mu[ \rho^2\rmax^2])}$.
This means that IPS and DR are unimprovable, in the worst case, beyond
constant factors. Another implication is that the lower bound of  Corollary~\ref{cor:continuous}
is sharp, and the minimax risk is precisely
$\Theta\bigParens{\frac1n(\E_\mu[ \rho^2\sigma^2]+\E_\mu[ \rho^2\rmax^2])}$.
While IPS and DR exhibit the same minimax rates, \Eq{DR:risk} also immediately shows
that DR will be better than IPS whenever $\hat{r}$ is even moderately good (better than $\hat{r}\equiv 0$).

\textbf{Comparison with asymptotic optimality results:} As discussed
in \Sec{intro}, previous work on optimal off-policy evaluation,
specifically the average treatment estimation, assumes
that it is possible to consistently estimate $r^*(x,a)=\E[r\given
  x,a]$.  Under such an assumption it is possible to (asymptotically)
match the risk of DR with the perfect reward estimator $\hat{r}\equiv
r^\star$, and this is the best possible asymptotic
risk~\citep{hahn1998role}.  This optimal risk is
$\frac1n\bigParens{\E_\mu[ \rho^2\sigma^2]+\Var_{x\sim D}\E_\pi[
    r^*\given x]}$, corresponding to the first two terms of
\Eq{DR:risk}, with no dependence on $\rmax$. Several estimators
achieve this risk, \emph{including the multiplicative constant}, under
various consistency assumptions
\citep{hahn1998role,hirano2003efficient,ImbensNeRi07,rothe2016value}.
Note that this is strictly below our lower bound for continuous
$\dx$. That is, consistency assumptions yield a better asymptotic risk
than possible in the agnostic setting. The gap in
constants between our upper and lower bounds is due to the
finite-sample setting, where lower-order terms cannot be ignored, but
have to be explicitly bounded. Indeed, apart from the result
of~\citet{li2015toward}, discussed below, ours is the first
finite-sample lower bound for off-policy evaluation.

\textbf{Comparison with multi-armed bandits:} For multi-armed bandits,
equivalent to contextual bandits with a single context,
\citet{li2015toward} show that the minimax risk equals
$\Theta(\E_\mu[\rho^2\sigma^2]/n)$ and is achieved, e.g., by DM,
whereas IPS is suboptimal. They also obtain a similar result for
contextual bandits, assuming that each context appears with a
large-enough probability to estimate its associated rewards by
empirical averages (amounting to realizability). While we obtain a
larger lower bound, this is not a contradiction, because we allow
arbitrarily small probabilities of individual contexts and even
continuous distributions, where the probability of any single context
is zero.

On a closer inspection, the first term of our bound in \Thm{minimax} coincides with the lower
bound of \citet{li2015toward} (up to constants).  The second term
(optimized over $\gamma$) is non-zero only if there are contexts with
small probabilities relative to the number of samples. In multi-armed
bandits, we recover the bound of \citet{li2015toward}.  When the
context distribution is continuous, or the probability of seeing
repeated contexts in a data set of size $n$ is small, we get the
minimax optimality of IPS.

One of our key contributions is to highlight this
\emph{agnostic contextual} regime where IPS is optimal.
In the \emph{non-contextual} regime, where each context appears frequently,
the rewards for each context-action pair can be consistently estimated by empirical averages.
Similarly, the asymptotic results discussed earlier focus on a setting where rewards
can be consistently estimated thanks to parametric assumptions or smoothness (for non-parametric
estimation), with the goal of asymptotic efficiency.
%
%
%
Our work complements that line of research. In many practical
situations, we wish to evaluate policies on high-dimensional context
spaces, where the consistent estimation of rewards is not a feasible
option. In other words, the agnostic contextual regime
dominates.

%
The distinction between the contextual and non-contextual regime is
also present in our proof, which combines
a non-contextual lower bound
due to the reward noise, similar to the analysis of~\citet{li2015toward}, and
an additional bound arising for non-degenerate context distributions. This
latter result is a key technical novelty of our paper.

\paragraph{Proof sketch:}
We only sketch some of the main ideas here and defer the full proof to
Appendix~\ref{sec:proofs_lower}.
For
simplicity, we discuss the case where $\dx$ is a continuous
distribution. We consider two separate problem instances corresponding
to the two terms in Theorem~\ref{thm:minimax}. The first part is relatively
straightforward and reduces the problem to Gaussian mean
estimation. We focus on the second part which depends on $\rmax$. Our
construction defines a prior over the reward distributions,
$D(r\given x,a)$. Given any $(x,a)$, a problem instance is given by
\begin{equation}
\notag
\label{eqn:rewards}
  \E[r\given x,a] = \er(x,a) =
\begin{cases}
  \rmax(x,a)
  &\text{w.p.\ $\theta(x,a)$,}
\\
  0
  &\text{w.p.\ $1-\theta(x,a)$,}
\end{cases}
\end{equation}
for $\theta(x,a)$ to be appropriately chosen. Once $\er$ is drawn, we
consider a problem instance defined by $\er$ where the rewards are
deterministic and the only randomness is in the contexts. In order to
lower bound the MSE across all problems, it suffices to lower bound
$\E_\theta[\textrm{MSE$_\eta$}(\hat{v})]$. That is, we can compute the
MSE of an estimator for each individual $\eta$, and take expectation
of the MSEs under the prior prescribed by $\theta$. If the expectation
is large, we know that there is a problem instance where the estimator
incurs a large MSE.

A key insight in our proof is that this expectation can be lower
bounded by
$\textrm{MSE}_{\E_\theta [\er(x,a)]}(\hat{v})$, corresponding to
the MSE of a single problem instance
with the actual \emph{rewards},
rather than $\eta(x,a)$, drawn according to $\theta$ and with the mean
reward function $\E_\theta [\er(x,a)]$.
This is
powerful, since this new problem instance has stochastic rewards,
just like Gaussian mean estimation, and is amenable to
standard techniques. The lower bound by $\textrm{MSE}_{\E_\theta [\er(x,a)]}(\hat{v})$
is only valid when the
context distribution $\dx$ is rich enough (e.g., continuous). In that case,
our reasoning shows that with enough randomness in the context
distribution, a problem with even a deterministic reward function is
extremely challenging.
%

\section{Incorporating Reward Models}
\label{sec:use_dm}

As discussed in the previous section, it is generally possible
to beat our minimax bound when consistent reward models exist.
We also argued that even in the absence of a consistent model, when DR and IPS
both achieve optimal risk rates, the performance of DR on
finite samples will be better than IPS as long
as the reward model is even moderately good (see Eq.~\ref{eq:DR:risk}).
However,
under a large reward noise $\sigma$, DR may still suffer from high variance when the importance weights are large, even when given a perfect reward model.
In this section, we derive a class of estimators that leverage reward models to directly address this source of high variance,
in a manner very different from the standard DR approach.

\subsection{The \OA Estimators}
\label{sec:adaptive_est}

Our starting
point is the observation that insistence on maintaining
unbiasedness puts the DR estimator at one extreme end of the
bias-variance tradeoff. Prior works have considered ideas such as
truncating the rewards or importance weights when the importance
weights are large (see, e.g., \citealt{bottou2013counterfactual}), which can dramatically reduce the variance
at the cost of a little bias. We take the intuition a step further
and propose to estimate the rewards for actions by two distinct strategies, based on
whether they have a large or a small importance weight in a given
context. When importance weights are small, we continue to use our
favorite unbiased estimators, but switch to directly applying the (potentially
biased) reward model on actions with large importance
weights. Here, ``small'' and ``large'' are defined via a
\emph{threshold parameter} $\tau$. Varying this parameter between $0$ and $\infty$
leads to a family of estimators which we call the \OA estimators as
they switch between
an agnostic approach (such as DR or IPS) and the direct method.

We now formalize this intuition, and begin by decomposing $v^\pi$
according to importance weights:
\begin{align*}
\E_{\pi}[r] &=
  \E_\pi [r \1(\rho\leq\tau) ]
+ \E_\pi [r \1(\rho>\tau) ]
\\
  &= \E_\mu[\rho r \1(\rho\leq\tau) ]
\\
&\quad{}
+
  \E_{x\sim\dx}\BigBracks{
  \sum_{a\in \cA} \E_D[r\given x,a]\, \pi(a\given x)\, \mathbf{1}(\rho(x,a){>}\tau)
  }
.
\end{align*}
Conceptually, we split our problem into two. The first problem always
has small importance weights, so we can use unbiased estimators such
as IPS or DR. The second problem, where importance weights
are large, is addressed by DM. Writing this out leads to
the following estimator:
\begin{align}
&\hat{v}_{\mathrm{\OA}}
= \frac{1}{n}\sum_{i=1}^n \left[r_i \rho_i
  \mathbf{1}(\rho_i\leq \tau) \right] \nonumber
\\
&\qquad{}+
\frac{1}{n}\sum_{i=1}^n\sum_{a
  \in \cA} \hat{r}(x_i,a)\pi(a\given x_i) \1
(\rho(x_i,a)>\tau). \label{eq:OA}
\end{align}
Note that the above estimator specifically uses IPS on the first part
of the problem. When DR is used instead of IPS, we refer to the resulting
estimator as \OADR. The reward model used within the DR part of
the \OADR estimator can be the same or different from the reward model
used to impute rewards in the second part.
We next present a
bound on the MSE of the \OA estimator using IPS.
A similar bound holds for \OADR.

\begin{theorem}\label{thm:MSEbound}
  Let $\epsilon(a,x):= \hat{r}(a,x)-\E[r|a,x]$ be the bias of
  $\hat{r}$ and assume $\hat{r}(x,a) \in [0,\rmax(x,a)]$ almost
  surely. Then for $\hat{v}_{\mathrm{\OA}}$, with $\tau > 0$,
  the MSE is at most
  \begin{align*}
     &\frac{2}{n}\BigBraces{\,
      \E_\mu\bigBracks{
              \left(\sigma^2{+}R_{\max}^2\right)\rho^2 \1(\rho\,{\leq}\,\tau)
      }
      +
      \E_\pi\bigBracks{ R_{\max}^2\1(\rho\,{>}\,\tau) }
      \,}
\\&\qquad{}
      +
      \E_{\pi}\bigBracks{\epsilon\1(\rho\,{>}\,\tau)}^2
  .
  \end{align*}
\end{theorem}

The proposed estimator interpolates between DM and IPS.
For $\tau = 0$, \OA coincides with DM,
while $\tau \rightarrow \infty$ yields IPS. Consequently, \OA estimator
is minimax optimal when $\tau$ is appropriately chosen. However, unlike
IPS and DR, the \OA and \OADR estimators are by design more robust to
large (or heavy-tailed) importance weights. Several estimators related
to \OA have been previously studied:

\begin{packed_enum}
\item \citet{bottou2013counterfactual} consider a special case
  of \OA with $\hat{r} \equiv 0$, meaning that all
  the actions with large importance weights are eliminated
  from IPS.
  We refer to this method as \emph{Trimmed IPS}.
\item \citet{thomas2016data} study an estimator similar to \OA in the more
  general setting of reinforcement learning. Their \emph{MAGIC}
  estimator can be
  seen as using several candidate thresholds $\tau$ and then
  evaluating the policy by a weighted sum of the estimators
  corresponding to each $\tau$. Similar to our approach of automatically
  determining $\tau$, they determine the weighting of estimators
  via optimization (as we discuss below).
\end{packed_enum}

\subsection{Automatic Parameter Tuning}\label{sec:autotuning}

So far we have discussed the properties of the \OA estimators
assuming that the parameter $\tau$ is chosen well. Our goal
is to obtain the best of IPS and DM, but
a poor choice of $\tau$ might easily give us the worst
of the two estimators. Therefore, a method for selecting $\tau$
plays an essential role. A natural criterion would be to
pick $\tau$ that minimizes the MSE of the resulting estimator. Since we
do not know the precise MSE (as $v^\pi$ is unknown), an alternative is
to minimize its data-dependent estimate. Recalling that the MSE
can be written as the sum of variance and squared bias, we estimate
and bound the terms individually.

Recall that we are working with a data set $(x_i, a_i, r_i)$ and
$\rho_i\coloneqq\pi(a_i\given x_i)/\mu(a_i\given x_i)$. Using this data, it is
straightforward to estimate the
variance of the \OA estimator. Let
$Y_i(\tau)$ denote the estimated value that $\pi$ obtains on the data
point $x_i$ according to the \OA estimator with the threshold $\tau$,
that is
\[
Y_i(\tau)\coloneqq
r_i \rho_i \1(\rho_i{\leq}\tau) + \!\sum_{a \in \cA}\!
\hat{r}(x_i,a)\pi(a\given x_i) \1(\rho(x_i,a){>}\tau),
\]
and $\bar{Y}(\tau)\coloneqq\frac{1}{n} \sum_{i=1}^n Y_i(\tau)$. Since
$\hat{v}_{\mathrm{\OA}} = \bar{Y}(\tau)$ and the $x_i$ are
i.i.d., the variance can be estimated as
\begin{align}
\Var(\bar{Y}(\tau)) \approx \frac{1}{n^2} \sum_{i=1}^n
(Y_i(\tau)-\bar{Y}(\tau))^2 =: \widehat{\Var}_\tau,
\label{eq:var-tune}
\end{align}
where the approximation above is clearly consistent since the random
variables $Y_i$ are appropriately bounded as long as the rewards are
bounded, because the importance weights are capped at the threshold $\tau$.

Next we turn to the bias term. For understanding bias, we look at the
MSE bound in Theorem~\ref{thm:MSEbound}, and observe that the last
term in that theorem is precisely the squared bias. Rather
than using a direct bias estimate, which would require knowledge of
the error in $\hat{r}$, we will upper bound this term. We
assume that the function $\rmax(x,a)$ is known. This is not limiting
since in most practical applications an
\emph{a priori} bound on the rewards is known. Then we
can upper bound the squared bias as
\begin{align*}
&\E_{\pi}\bigBracks{\epsilon\1(\rho>\tau)}^2
\le
 \E_{\pi}\bigBracks{\rmax\1(\rho>\tau)}^2.
\end{align*}
Replacing the expectation with an average, we obtain
\begin{align}
\notag \widehat{\text{Bias}}^2_\tau\coloneqq
\biggBracks{\frac{1}{n}\sum_{i=1}^n
  \E_\pi\bigBracks{R_{\max}\1(\rho>\tau) \bigGiven x_i} }^2 .
\end{align}

With these estimates, we pick the threshold $\widehat{\tau}$ by
optimizing the sum of estimated variance and the upper bound on bias,
\begin{equation}
  \widehat{\tau}\coloneqq\argmin_{\tau} \widehat{\Var}_\tau +
  \widehat{\text{Bias}}^2_\tau.
\label{eq:tau-tune}
\end{equation}
Our upper bound on the bias is rather conservative,
as it upper bounds the error of DM at the largest possible value for
every data point. This has the effect of favoring the use of the
unbiased part in \OA whenever possible, unless the variance would
overwhelm even an arbitrarily biased DM. This conservative choice, however,
immediately implies the minimax optimality of the \OA
estimator using $\widehat{\tau}$,
because the incurred bias is
no more than our upper bound, and it is incurred only when the minimax
optimal IPS estimator would be suffering an even larger variance.

Our automatic tuning is related to the MAGIC estimator
of~\citet{thomas2016data}. The key differences are that we pick only
one threshold $\tau$, while they combine the estimates with many
different $\tau$s using a weighting function. They pick this
weighting function by optimizing a bias-variance tradeoff, but with
significantly different bias and variance estimators. In our experiments,
the automatic tuning using \Eq{tau-tune} generally works better
than MAGIC.

\section{Experiments}\label{sec:exp}

\begin{figure*}[p]
	\centering
	\begin{subfigure}[t]{0.41\textwidth}
		\centering
		\includegraphics[width=\textwidth]{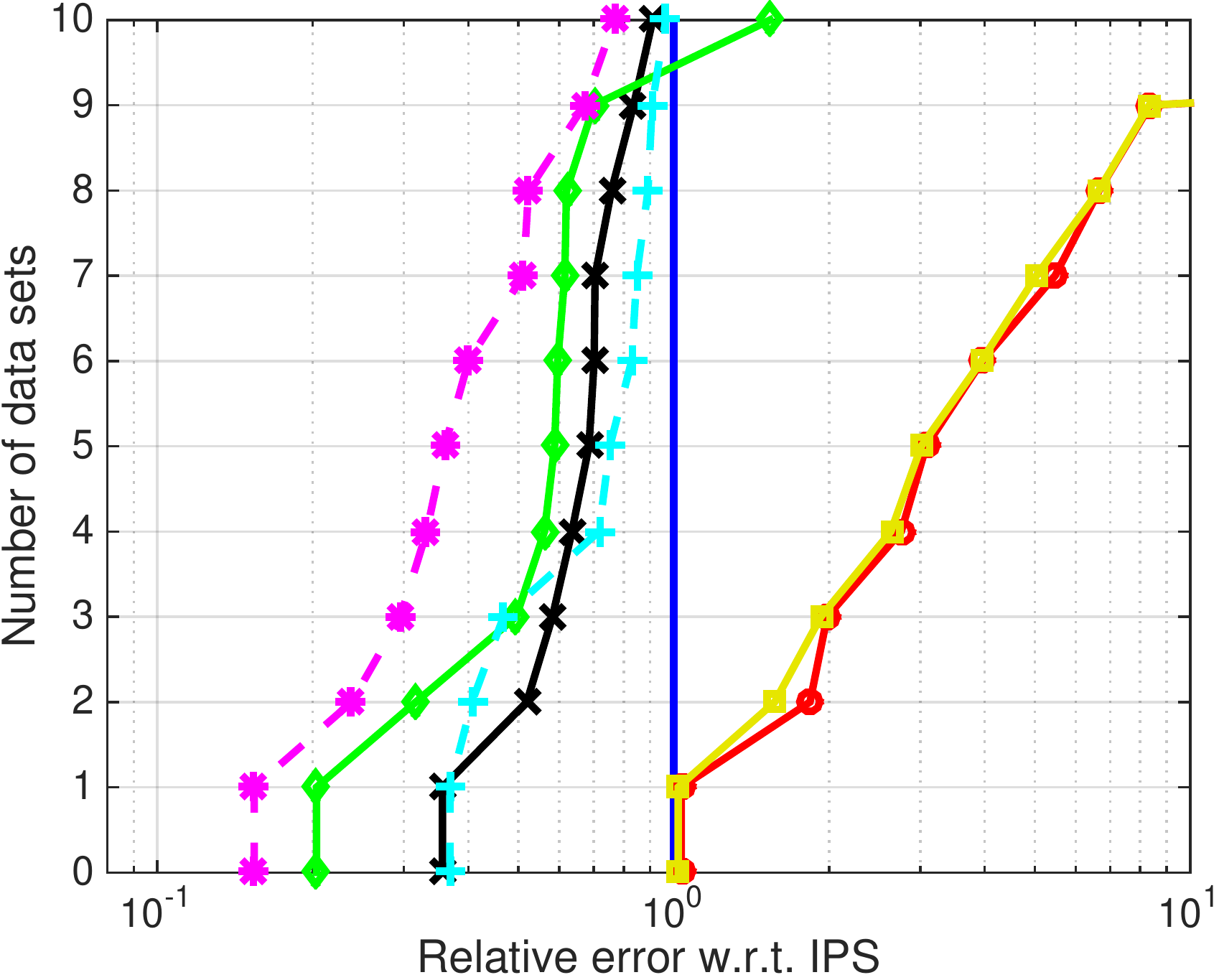}
		\caption{Deterministic reward}\label{fig:CDF-raw}		
	\end{subfigure}
	\quad\quad
	\begin{subfigure}[t]{0.4\textwidth}
		\centering
		\includegraphics[width=\textwidth]{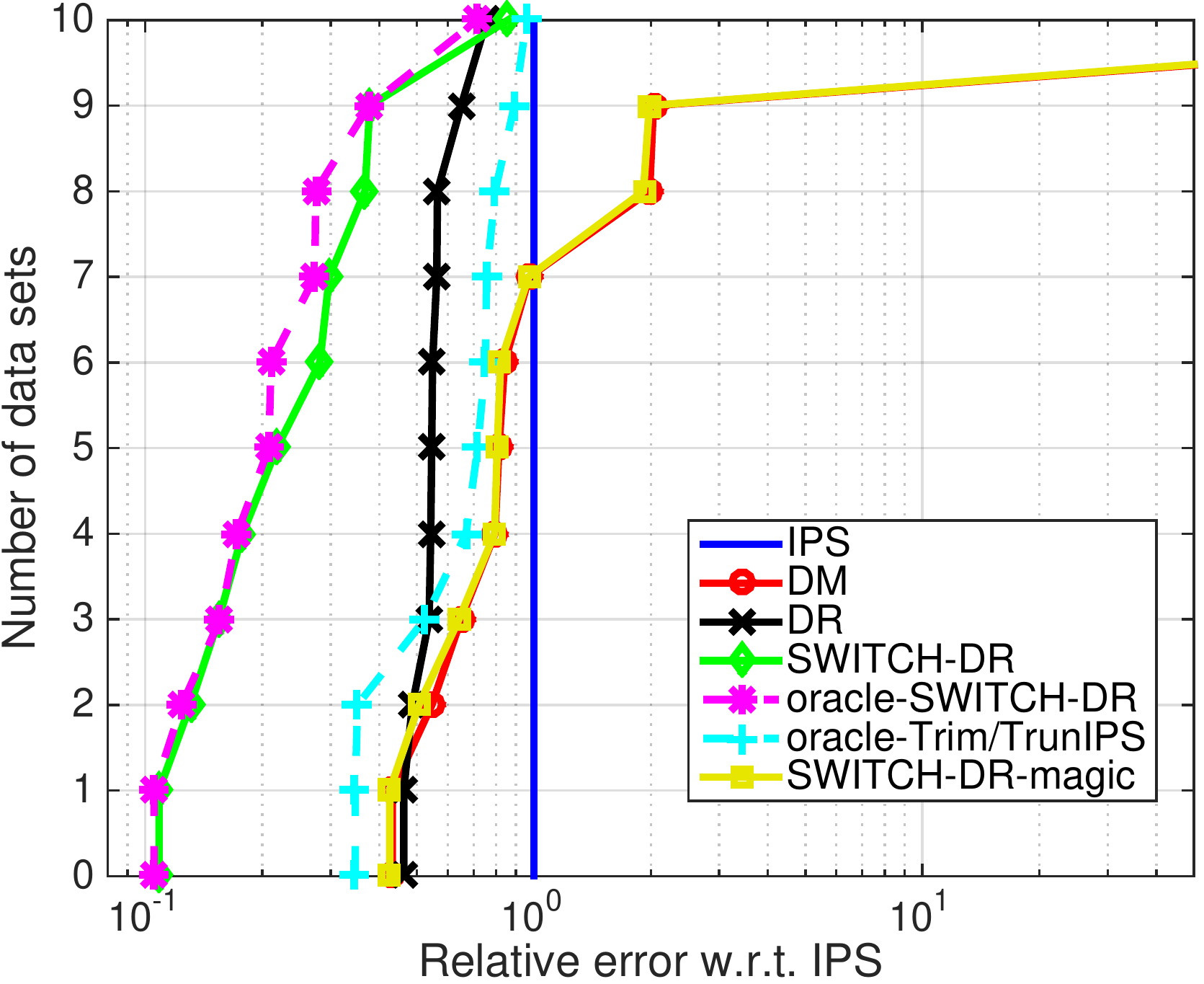}
		\caption{Noisy reward}\label{fig:CDF-noisy}
	\end{subfigure}\\
	\caption{The number of UCI data sets
          where each method
          achieves at least a given Rel.~MSE. On the left, the UCI
          labels are used as is; on the right, label noise is
          added. Curves towards top-left achieve smaller MSE in more
          cases. Methods in dashed lines are ``cheating'' by choosing
          the threshold $\tau$ to optimize test MSE. \OADR outperforms
          baselines and our tuning of $\tau$ is not too far from the
          best possible. Each data set uses an $n$ which is the size of
          the data set, drawn via bootstrap sampling and results are
          averaged over 500 trials.}\label{fig:CDF}
\bigskip
	\centering
  	\hspace{2em}\includegraphics[width=0.9\textwidth]{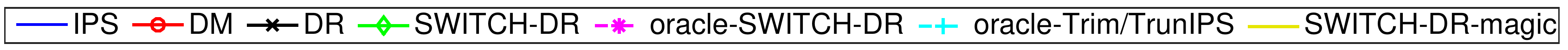}
  \\
  \medskip
  \begin{subfigure}[t]{0.4\textwidth}
    \centering
    \includegraphics[width=\textwidth]{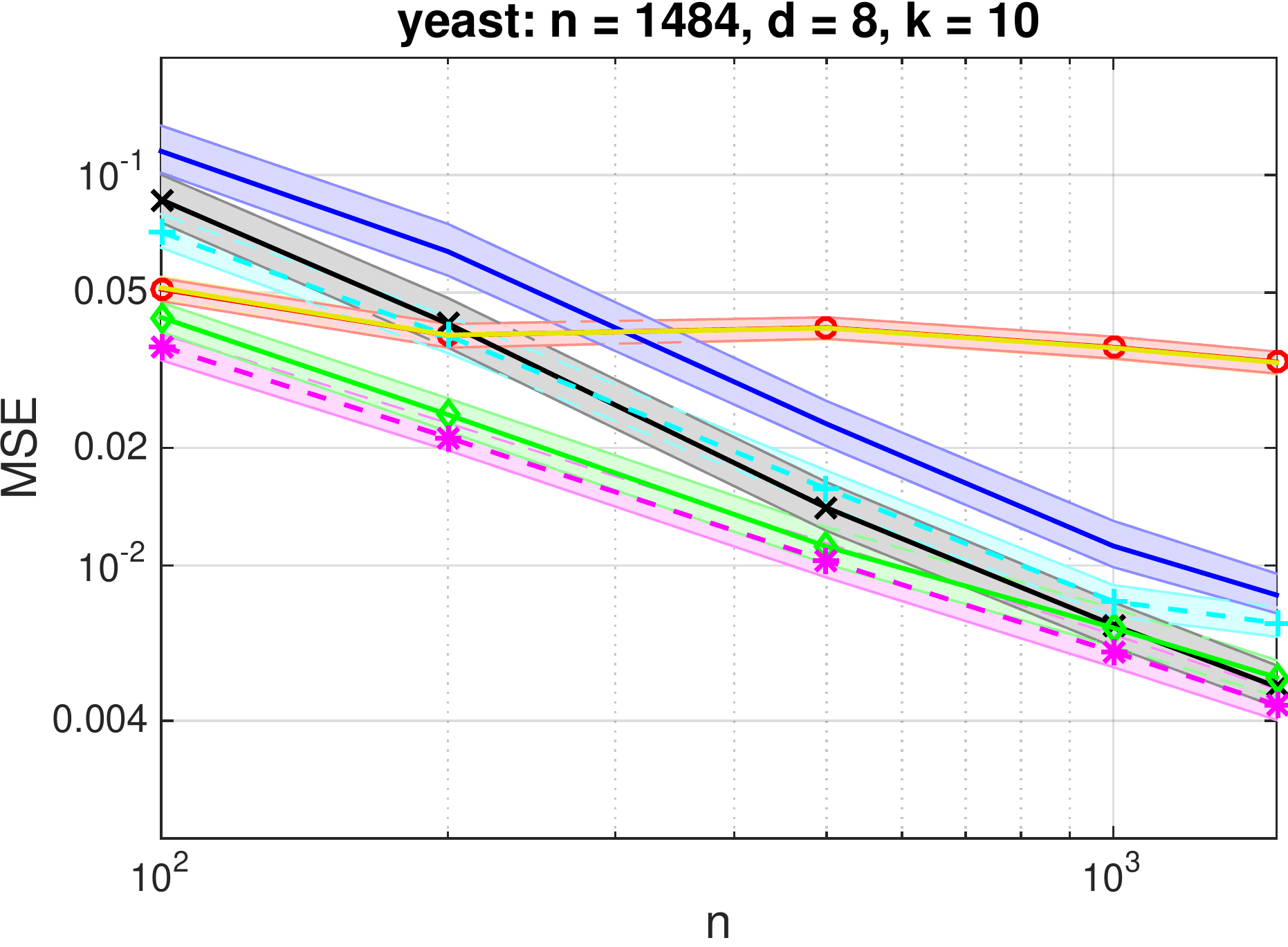}
    \figsqueezeB\caption{yeast / deterministic reward}\label{fig:raw-yeast}		
  \end{subfigure}
  \hspace{0.25in}
  \begin{subfigure}[t]{0.4\textwidth}
    \centering
    \includegraphics[width=\textwidth]{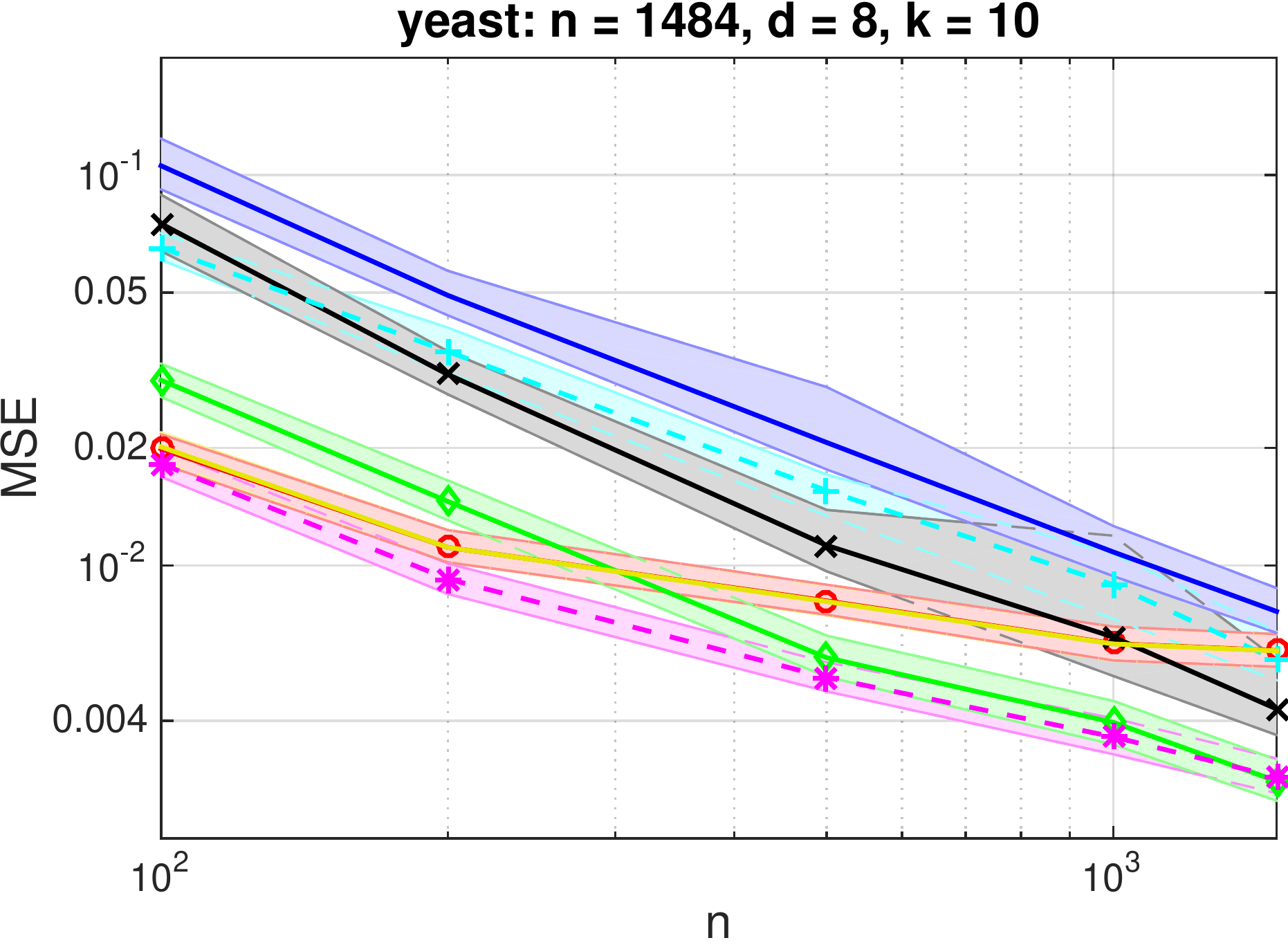}
    \figsqueezeB\caption{yeast / noisy reward}\label{fig:noisy-yeast}		
  \end{subfigure}
  \\
  \medskip
  \begin{subfigure}[t]{0.4\textwidth}
    \centering
    \includegraphics[width=\textwidth]{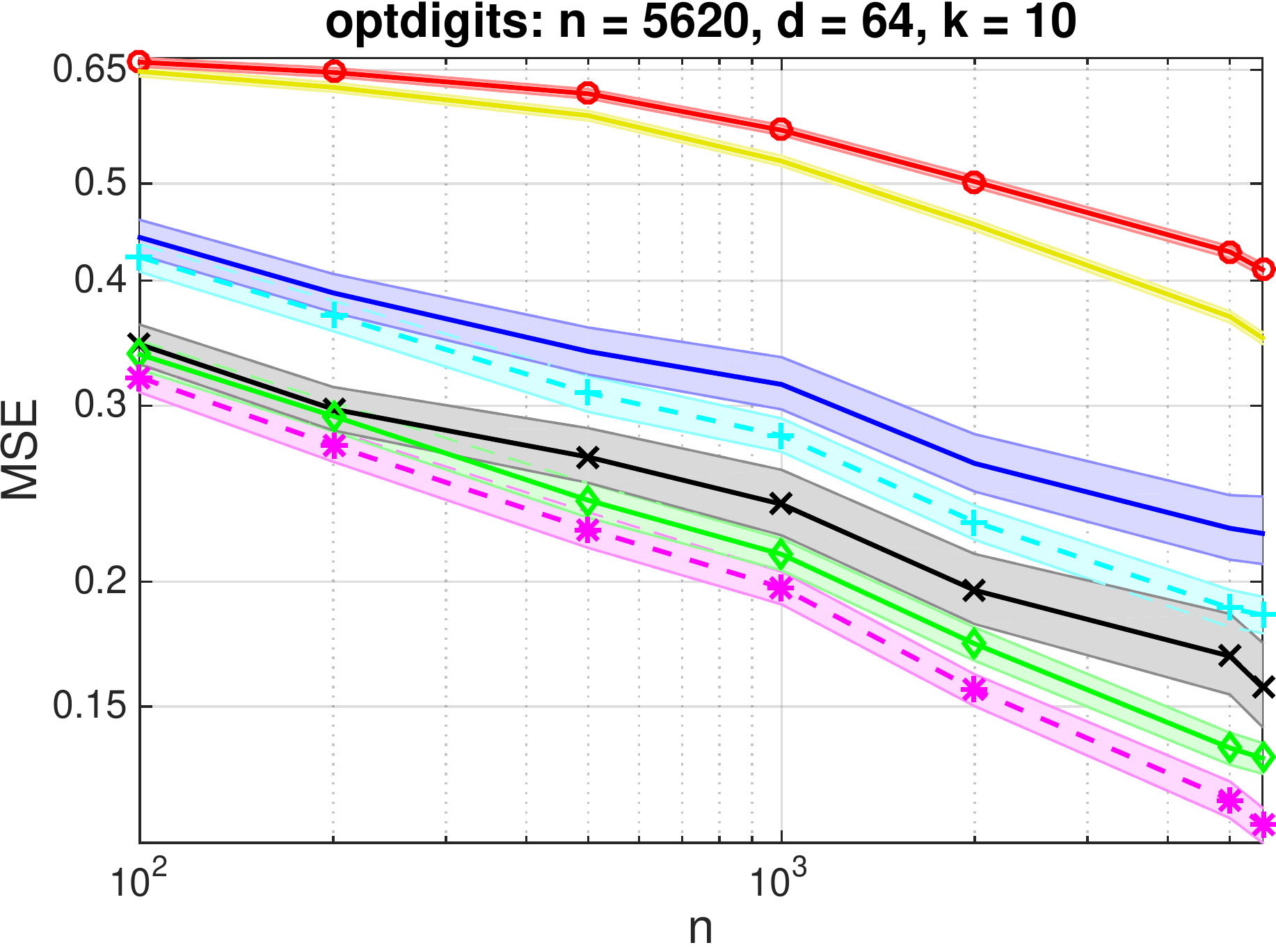}
    \figsqueezeB\caption{optdigits / deterministic reward}\label{fig:raw-optdigits}		
  \end{subfigure}
  \hspace{0.25in}
  \begin{subfigure}[t]{0.4\textwidth}
    \centering
    \includegraphics[width=\textwidth]{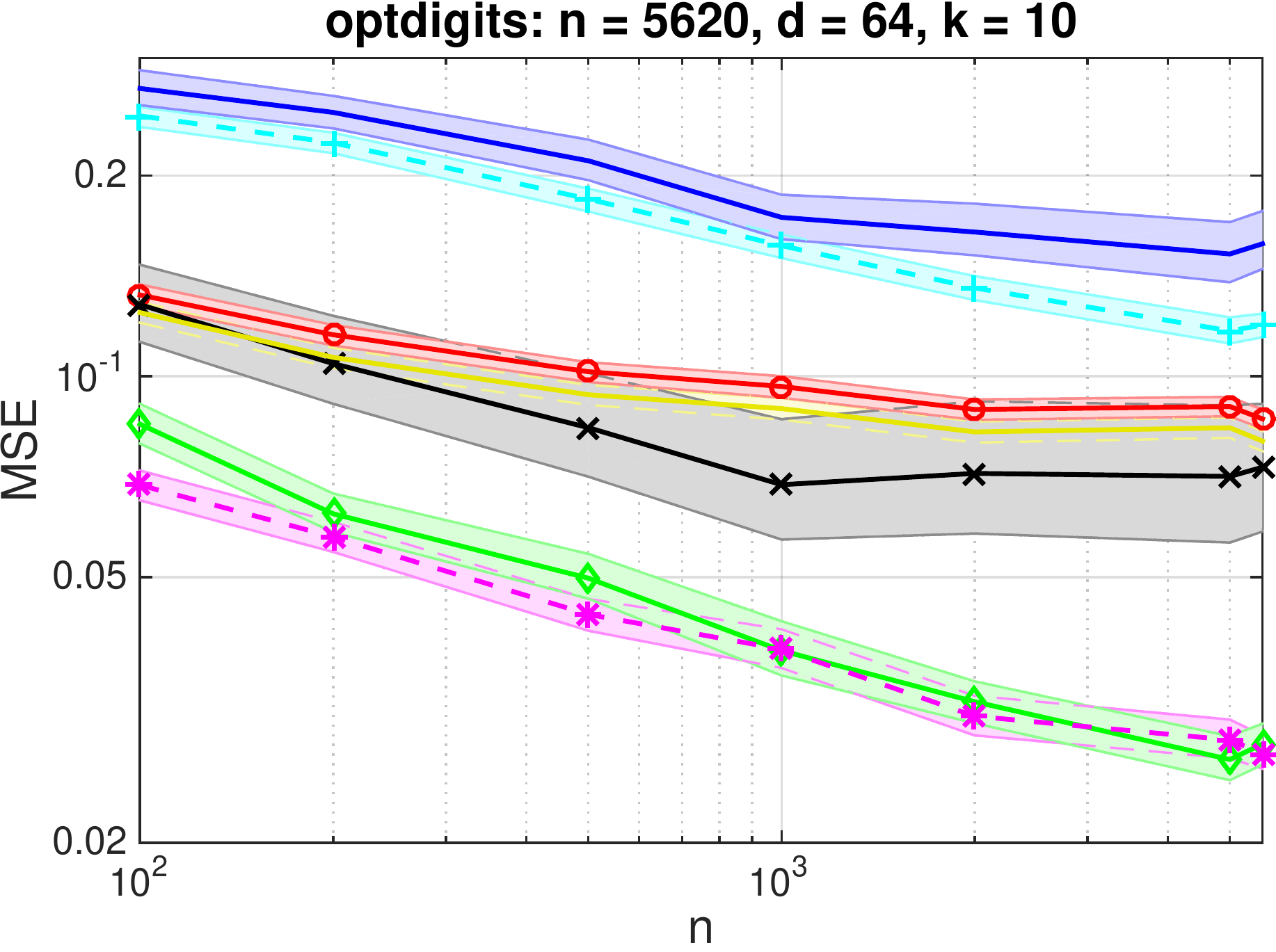}
    \figsqueezeB\caption{optdigits / noisy reward}\label{fig:noisy-optdigits}		
  \end{subfigure}
  \caption{MSE of different methods as a function of input data size. \emph{Top:} optdigits data set.
  \emph{Bottom:} yeast data set. }\label{fig:samplesize}
\end{figure*}

We next empirically evaluate the proposed
\OA estimators on the 10 UCI data sets previously used
for off-policy evaluation \citep{dudik2011doubly}. We convert the multi-class
classification problem to contextual bandits by treating the labels as
actions for a policy $\mu$, and recording the reward of $1$ if the
correct label is chosen, and $0$ otherwise.

In addition to this \emph{deterministic} reward model,
we also consider a \emph{noisy} reward model for each data set,
which reveals the correct reward with probability $0.5$ and outputs a
random coin toss otherwise. Theoretically, this
should lead to bigger $\sigma^2$ and larger variance in all
estimators. In both reward models, $\rmax \equiv 1$ is a valid bound.

The target policy $\pi$ is the deterministic decision of a
logistic regression classifier learned on the multi-class data, while the logging policy $\mu$
samples according to the probability estimates of a logistic model
learned on a covariate-shifted version of the data.
The covariate shift is obtained as in prior work
~\citep{dudik2011doubly,gretton2009covariate}.

In each data set with $n$ examples, we treat the uniform distribution over
the data set itself as a surrogate of the population distribution so
that we know the ground truth of the rewards. Then, in the simulator,
we randomly draw i.i.d.\ data sets of size
$100, 200, 500,1000,2000,5000, 10000,\dotsc$ until reaching $n$, with
$500$ different repetitions of each size.
We estimate MSE of each estimator by taking the empirical
average of the squared error over the $500$ replicates; note that we can calculate
the squared error exactly, because we know $v^\pi$. For some of the methods,
e.g., IPS and DR, the MSE can have a very large variance
due to the potentially large importance weights. This leads to very
large error bars if we estimate their MSE even with $500$
replicates. To circumvent this issue, we report a clipped version of
the MSE that truncates the squared error to $1$, namely $ \MSE =
\E[(\hat{v}-v^\pi)^2 \wedge 1].$ This allows us to get valid
confidence intervals for our empirical estimates of this
quantity. Note that this does not change the MSE estimate of our
approach at all, but is significantly more favorable towards IPS and
DR. In this section, whenever we refer to ``MSE'', we are referring to
this truncated version.

We compare \OA and \OADR against the following baselines:
1.~\emph{IPS};
2.~\emph{DM trained via logistic regression};
3.~\emph{DR};
4.~\emph{Truncated and Reweighted IPS (TrunIPS)};
and
5.~\emph{Trimmed IPS (TrimIPS)}.

In DM, we train $\hat{r}$ and then evaluate the policy on the same
contextual bandit data set. Following \citet{dudik2011doubly}, DR is constructed by randomly splitting the
contextual bandit data into two folds, estimating $\hat{r}$ on one fold, and then evaluating
$\pi$ on the other fold and vice versa, obtaining two estimates.
The final estimate is the average of the two.
%
TrunIPS is a variant of IPS, where importance weights are
capped at a threshold $\tau$ and then renormalized to sum to one~\citep[see,
  e.g.,][]{bembom2008data}.
TrimIPS is a special case of \OA due to~\citet{bottou2013counterfactual} described earlier,
where $\hat{r} \equiv 0$.

For \OA and \OADR as well as TrunIPS and TrimIPS we
select the parameter $\tau$ by our automatic tuning from
Section~\ref{sec:autotuning}. To evaluate our tuning approach,
we also include the
results for the $\tau$ tuned optimally in hindsight, which we
refer to as the \emph{oracle} setting, and also show results obtained
by the multi-threshold MAGIC approach. In all these approaches we
optimize among 21 possible thresholds, from an exponential grid between
the smallest and the largest importance weight observed in the data,
considering all actions in each observed context. 

In order to stay comparable across data sets and data sizes, our performance
measure is the relative MSE with
respect to the IPS. Thus, for each estimator $\hat{v}$, we calculate
$\mathrm{Rel.\ MSE}(\hat{v}) = \frac{\mathrm{MSE}(\hat{v})}{\mathrm{ MSE}(\hat{v}_{\mathrm{IPS}})}$.

The results are summarized in Figure~\ref{fig:CDF}, plotting the
number of data sets where each method achieves at least a given
relative MSE.%
\footnote{For
  clarity, we have excluded \OA, which
  significantly outperforms IPS, but is dominated by \OADR. Similarly,
  we only report the better of oracle-TrimIPS and oracle-TrunIPS.}
Thus, methods that achieve smaller MSE across more
data sets are towards the top-left corner of the plot, and a larger
area under the curve indicates better performance.
Some of the differences in MSE are several orders of magnitude
large since the relative MSE is shown on the logaritmic scale.
As we see, \OADR dominates all baselines and our empirical
tuning of $\tau$ is not too far from the best possible.
The automatic
tuning by MAGIC tends to revert to DM, because its bias estimate is too
optimistic and so DM is preferred whenever IPS or DR have some significant variance.
The gains of
\OADR are even greater in the noisy-reward setting, where we add label
noise to UCI data.

In Figure~\ref{fig:samplesize}, we illustrate the convergence of MSE
as $n$ increases.
We select two data sets and show how
\OADR performs against baselines in two typical cases: (i) when
the direct method works well initially but is outperformed by IPS and DR
as $n$ gets large, and (ii) when the direct method works
poorly.  In the first case, \OADR outperforms both DM and
IPS, while DR improves over IPS only moderately. In the second case,
\OADR performs about as well as IPS and DR despite a poor performance of DM.
In all cases, \OADR is robust to additional noise in the
reward, while IPS and DR suffer from higher variance.
Results for the remaining data sets are in
Appendix~\ref{sec:experiments}.

%

%
%

\section{Conclusion}
In this paper we have carried out minimax analysis of off-policy evaluation
in contextual bandits and showed that IPS and DR are minimax optimal
in the worst-case, when no consistent reward model is available.
This result complements existing asymptotic
theory with assumptions on reward models, and highlights the
differences between agnostic and consistent settings. Practically,
the result further motivates the importance of using side
information, possibly by modeling rewards directly,
especially when importance weights are too large. Given this
observation, we propose a new class of estimators called \OA that can
be used to combine any importance weighting estimators, including IPS
and DR, with DM. The estimators adaptively switch between DM
when the importance weights are large and either IPS or
DR when the importance weights are small. We show that the new
estimators have favorable theoretical properties and also work well on
real-world data. Many interesting directions remain open for future
work, including
high-probability upper bounds on the finite-sample MSE of \OA
estimators, as well as sharper finite-sample lower bounds under
realistic assumptions on the reward model.

\section*{Acknowledgments}

The work was partially completed during YW's internship at Microsoft
Research NYC from May 2016 to Aug 2016.  The authors would like to
thank Lihong Li and John Langford for helpful discussions, Edward
Kennedy for bringing our attention to related problems and recent
developments in causal inference, and an anonymous reviewer for
pointing out relevant econometric references and providing valuable
feedback that helped connect our work with research on average treatment effects.


\bibliographystyle{icml2017}
\bibliography{bandits}

\newpage

\appendix
\onecolumn

The supplementary materials is organized as follows. In
Appendices~\ref{sec:proofs_lower}, \ref{sec:otherproofs} and~\ref{sec:utilitylemmas},
we provide detailed proofs of the theoretical results in the
paper. In Appendix~\ref{sec:experiments}, we provide
additional figures for the experiments described in
Section~\ref{sec:exp}.

\section{Proof of Theorem~\ref{thm:minimax}}
\label{sec:proofs_lower}

In this appendix we prove the minimax bound of \Thm{minimax}. The result is obtained by combining the following
two lower bounds:
\begin{theorem}[Lower bound 1]\label{thm:lowerbound:sigma}
For each problem instance such that $\E_\mu[\rho^2\sigma^2] <\infty$, we have
\[
  R_n(\pi;\dx,\mu,\sigma,\rmax)
  \geq
  \frac{\E_\mu[\rho^2\sigma^2]}{32en}
  \left[1-
  \frac{\E_\mu\Bracks{\rho^2\sigma^2 \1\BigParens{\rho\sigma^2>\rmax\sqrt{n\E_\mu[\rho^2\sigma^2]/2}}}}{\E_\mu[\rho^2\sigma^2]}\right]^2.
\]
\end{theorem}
\begin{theorem}[Lower bound 2]\label{thm:lowerbound:rmax}
Assume that $\E_\mu[\rho^2 R_{\max}^2]<\infty$, and we are given $\lmax\in[0,1]$ and $\delta\in(0,1]$.
Write $\xi\coloneqq\xi_{\lmax}$ and $\lmax'\coloneqq\max\set{\lmax,\delta}$.
Then there exist functions $\hR(x,a)$ and $\hrho(x,a)$
such that
\[
   \hR^2(x,a)\le\rmax^2(x,a)\le (1+\delta)\hR^2(x,a)
\enspace,
\qquad
   \hrho^2(x,a)\le\rho^2(x,a)\le (1+\delta)\hrho^2(x,a)
\]
and the following lower bound holds:
\begin{multline*}
 R_n(\pi;\dx,\mu,
      \sigma,R_{\max})
\\
  \ge
  \frac{\E_\mu[\xi\hrho^2 \hR^2]}{32en}
  \Bracks{
      1 - \frac{\E_{\mu}
        \left[\xi \hrho^2 \hR^2
        \1\BigParens{\xi \hrho \hR > \sqrt{n\E_{\mu}[\xi \hrho^2 \hR^2]/16}}
        \right]}
      {\E_{\mu}[\xi \hrho^2 \hR^2]}
  }^2
  -
  \lmax'\log\bigParens{5/\lmax'}(1+\delta)\E_\mu[\xi\hrho^2 \hR^2]
\enspace.
\end{multline*}
\end{theorem}
The reason for introducing $\lmax'$ in
Theorem~\ref{thm:lowerbound:rmax} is to allow $\lmax=0$, which is an
important special case of the theorem. Otherwise, we could just assume
$0<\delta\le\lmax$. The first bound captures the intrinsic difficulty
due to the variance of reward, and is present even in a vanilla
multi-armed bandit problem without contexts.  The second result shows
the additional dependence on $R_{\max}^2$, even when $\sigma\equiv 0$,
whenever the distribution $\dx$ is not too degenerate, and captures
the additional difficulty of the contextual bandit problem. We next
show how these two lower bounds yield Theorem~\ref{thm:minimax} and
then return to their proofs.

\begin{proof}[Proof of \Thm{minimax}]
Throughout the theorem we write $\xi\coloneqq\xi_{\lmax}$.
We begin by simplifying the two lower bounds. Assume that
Assumption~\ref{ass:moment} holds with $\epsilon$. This also means
that $\E_\mu[\xi(\rho\rmax)^{2+\epsilon}]$ is finite as well as $\E_\mu[\xi(\rho\rmax)^2]$ is finite and
either both of them are zero or both of them are non-zero. Similarly, $\E_\mu[(\rho\sigma)^{2+\epsilon}]$ and $\E_\mu[(\rho\sigma)^2]$
are both finite and either both of them are zero or both of them are non-zero, so $C_{\lmax}$ is a finite constant.
Let $p=1+\epsilon/2$ and $q=1+2/\epsilon$, i.e., $1/p+1/q=1$. Further,
let $\hR$ and $\hrho$ be the functions from \Thm{lowerbound:rmax}. Then the
definition of $C_{\lmax}$ means that
\begin{align}
\notag C^{1/(\epsilon q)}_{\lmax}=C^{1/(2+\epsilon)}_{\lmax} &
= 2\cdot\max\Braces{
  \frac{\E_\mu\bigBracks{\xi(\rho^2\rmax^2)^{\frac{2+\epsilon}{2}}}^{\frac{2}{2+\epsilon}}}{\E_\mu\bigBracks{\xi\rho^2\rmax^2}},\,
  \frac{\E_\mu\bigBracks{(\rho^2\sigma^2)^{\frac{2+\epsilon}{2}}}^{\frac{2}{2+\epsilon}}}{\E_\mu\bigBracks{\rho^2\sigma^2}}
} \\
\notag
&
 = 2\cdot\max\Braces{
              \frac{\E_\mu\bigBracks{\xi(\rho^2\rmax^2)^p}^{1/p}}{\E_\mu\bigBracks{\xi\rho^2\rmax^2}},\,
              \frac{\E_\mu\bigBracks{(\rho^2\sigma^2)^p}^{1/p}}{\E_\mu\bigBracks{\rho^2\sigma^2}}
   }
\\
\label{eq:C:1/epsq}
&
\ge 2\cdot\max\Braces{
              \frac{\E_\mu\bigBracks{\xi(\hrho^2\hR^2)^p}^{1/p}}{\E_\mu\bigBracks{\xi\rho^2\rmax^2}},\,
              \frac{\E_\mu\bigBracks{(\rho^2\sigma^2)^p}^{1/p}}{\E_\mu\bigBracks{\rho^2\sigma^2}}
   }
\enspace,
\end{align}
and recall that we assume that
\begin{align}
\label{eq:n:exact}
n\ge\max\Braces{16 C^{1/\epsilon}_{\lmax},\,2 C^{2/\epsilon}_{\lmax}\E_\mu[\sigma^2/\rmax^2]}
\enspace.
\end{align}
First, we simplify the correction term in the lower bound of \Thm{lowerbound:sigma}. Using H\"older's inequality and \Eq{C:1/epsq}, we have
\begin{align}
\notag
&\E_\mu\Bracks{\rho^2\sigma^2 \1\left(\rho\sigma^2>R_{\max}
  \sqrt{n\E_\mu[\rho^2\sigma^2]/2}\right)}
\\
\notag
&\quad{}\le
\E_\mu\BigBracks{\bigParens{\rho^2\sigma^2}^p}^{1/p}
\cdot\P_\mu\BigBracks{\rho\sigma^2>R_{\max} \sqrt{n\E_\mu[\rho^2\sigma^2]/2}}^{1/q}
\\
\notag
&\quad{}\le
\frac12\E_\mu[\rho^2\sigma^2]\cdot C^{1/(\epsilon q)}_{\lmax}
\cdot\P_\mu\BigBracks{\rho\sigma^2/\rmax > \sqrt{n\E_\mu[\rho^2\sigma^2]/2}}^{1/q}.
\intertext{We further invoke Markov's inequality, Cauchy-Schwartz
  inequality, and \Eq{n:exact} in the following three steps to
  simplify this event as}
\notag
&\quad{}\le
\frac12\E_\mu[\rho^2\sigma^2]\cdot C^{1/(\epsilon q)}_{\lmax}
\cdot\Parens{\frac{\E_\mu\bigBracks{\rho\sigma\cdot(\sigma/\rmax)}}{\sqrt{n\E_\mu[\rho^2\sigma^2]/2}}}^{1/q}
\\
\notag
&\quad{}\le
\frac12\E_\mu[\rho^2\sigma^2]\cdot C^{1/(\epsilon q)}_{\lmax}
\cdot\Parens{\frac{\sqrt{\E_\mu[\rho^2\sigma^2]}\cdot\sqrt{\E_\mu[\sigma^2/\rmax^2]}}{\sqrt{n\E_\mu[\rho^2\sigma^2]/2}}}^{1/q}
\\
\label{eq:minimax:1}
&\quad{}=
\frac12\E_\mu[\rho^2\sigma^2]
\cdot\Parens{C^{2/\epsilon}_{\lmax}\cdot\frac{2\E_\mu[\sigma^2/\rmax^2]}{n}}^{1/2q}
\le
\frac12\E_\mu[\rho^2\sigma^2]
\enspace.
\intertext{%
For the correction term in \Thm{lowerbound:rmax}, we similarly have}
\notag
&\E_{\mu}\Bracks{\xi\hrho^2 \hR^2 \1\left(\xi\hrho \hR> \sqrt{n \E_\mu[\xi\hrho^2\hR^2]/16}\right)}
\\
\notag
&\quad{}\le
\E_\mu\BigBracks{\bigParens{\xi\hrho^2\hR^2}^p}^{1/p}
\cdot\P_\mu\BigBracks{\xi\hrho \hR> \sqrt{n \E_\mu[\xi\hrho^2\hR^2]/16}}^{1/q}
\\
\notag
&\quad{}\le
\frac12\E_\mu[\xi\rho^2\rmax^2]\cdot C_{\lmax}^{1/(\epsilon q)}
\cdot\P_\mu\BigBracks{\xi\hrho^2 \hR^2 > n \E_\mu[\xi\hrho^2\hR^2]/16}^{1/q},
\intertext{%
so that Markov's inequality and \Eq{n:exact} further yield}
\notag
&\quad{}\le
\frac12\E_\mu[\xi\rho^2\rmax^2]\cdot C_{\lmax}^{1/(\epsilon q)}
\cdot\Parens{\frac{\E_\mu[\xi\hrho^2 \hR^2]}{n \E_\mu[\xi\hrho^2\hR^2]/16}}^{1/q}
\\
\label{eq:minimax:2}
&\quad{}
=
\frac12\E_\mu[\xi\rho^2\rmax^2]
\cdot\Parens{C_{\lmax}^{1/\epsilon}\cdot\frac{16}{n}}^{1/q}
\le
\frac12\E_\mu[\xi\rho^2\rmax^2]
\le
\frac{(1+\delta)^2}{2}\E_\mu[\xi\hrho^2\hR^2]
\enspace.
\end{align}
Using \Eq{minimax:1}, the bound of \Thm{lowerbound:sigma} simplifies as
\begin{align}
\notag
&R_n(\pi;\dx,\mu,\sigma,R_{\max})
\\
\notag
&\quad{}
\ge
  \frac{\E_\mu[\rho^2\sigma^2]}{32en}
  \left[1-
  \frac{\E_\mu\Bracks{\rho^2\sigma^2 \1\left(\rho\sigma^2>R_{\max} \sqrt{n\E_\mu[\rho^2\sigma^2]/2}\right)}}{\E_\mu[\rho^2\sigma^2]}\right]^2
\\
\label{eq:lowerbound:sigma:simplified}
&\quad{}
\ge
  \frac{\E_\mu[\rho^2\sigma^2]}{32en}
  \Parens{1-\frac12}^2
=
  \frac{\E_\mu[\rho^2\sigma^2]}{128en}
\enspace.
\end{align}
Similarly, by \Eq{minimax:2}, \Thm{lowerbound:rmax} simplifies as
\begin{align*}
&R_n(\pi;\dx,\mu,\sigma,R_{\max})
\\
&\quad{}
\ge
  \frac{\E_\mu[\xi\hrho^2 \hR^2]}{32en}
  \Bracks{
      1 - \frac{\E_{\mu}
        \left[\xi \hrho^2 \hR^2
        \1\BigParens{\xi \hrho \hR > \sqrt{n\E_{\mu}[\xi \hrho^2 \hR^2]/16}}
        \right]}
      {\E_{\mu}[\xi \hrho^2 \hR^2]}
  }^2
  -
  \lmax'\log(5/\lmax')(1+\delta)\E_\mu[\xi\hrho^2 \hR^2]
\\
&\quad{}
\ge
  \frac{\E_\mu[\xi\hrho^2 \hR^2]}{32en}
  \Bracks{
      1 - \frac{(1+\delta)^2}{2}
  }^2
  -
  \lmax'\log(5/\lmax')(1+\delta)\E_\mu[\xi\hrho^2 \hR^2]
\\
&\quad{}
=
  \frac{\E_\mu[\xi\hrho^2 \hR^2]}{128en}
  \bigParens{1 - 2\delta-\delta^2}^2
  -
  \lmax'\log(5/\lmax')(1+\delta)\E_\mu[\xi\hrho^2 \hR^2]
\\
&\quad{}
\ge
  \frac{\E_\mu[\xi\rho^2 \rmax^2]}{128en}
  \frac{\bigParens{1 - 2\delta-\delta^2}^2}{(1+\delta)^2}
  -
  \lmax'\log(5/\lmax')(1+\delta)\E_\mu[\xi\rho^2 \rmax^2]
\enspace.
\end{align*}
Since this bound is valid for all $\delta>0$, taking $\delta\to 0$, we obtain
\begin{align*}
R_n(\pi;\dx,\mu,\sigma,R_{\max})
\ge
  \frac{\E_\mu[\xi\rho^2 \rmax^2]}{128en}
  -
  \lmax\log(5/\lmax)\E_\mu[\xi\rho^2 \rmax^2]
\enspace.
\end{align*}
Combining this bound with \Eq{lowerbound:sigma:simplified} yields
\begin{align*}
&R_n(\pi;\dx,\mu,\sigma,R_{\max})
\\
&\quad{}
\ge
  \frac12\cdot\frac{\E_\mu[\rho^2\sigma^2]}{128en}
  +
  \frac12\cdot\frac{\E_\mu[\xi\rho^2R_{\max}^2]}{128en}
  -
  \frac12\cdot\lmax\log(5/\lmax)\E_\mu[\xi\rho^2 \rmax^2]
\\[3pt]
&\quad{}
\ge
  \frac{\E_\mu[\rho^2\sigma^2]}{700n}
  +
  \frac{\E_\mu[\xi\rho^2R_{\max}^2]}{700n}
  -
  \frac12\cdot\lmax\log(5/\lmax)\E_\mu[\xi\rho^2 \rmax^2]
\\[3pt]
\tag*{\qedhere}
&\quad{}
=
  \frac{1}{700n}
  \BigBracks{
    \E_\mu[\rho^2\sigma^2]
    +
    \E_\mu[\xi\rho^2R_{\max}^2]
    \BigParens{
        1-350n\lmax\log(5/\lmax)
    }
  }
\enspace.
\end{align*}
\end{proof}

It remains to prove
Theorems~\ref{thm:lowerbound:sigma} and~\ref{thm:lowerbound:rmax}. They are both
proved by a reduction to hypothesis testing, and
invoke Le Cam's argument to lower-bound the error in this testing
problem. As in most arguments of this nature, the key contribution lies in
the construction of an appropriate testing problem that leads to the
desired lower bounds. Before proving the theorems, we recall the basic
result of Le Cam which underlies our proofs. We point the reader to the
excellent exposition of~\citet[Section 36.4]{wasserman2008minimax} on more details about Le Cam's
argument.

\begin{theorem}[Le Cam's method, {\citealp[Theorem 36.8]{wasserman2008minimax}}]
	Let $\cP$ be a set of distributions, let $X_1,\dotsc,X_n$ be an i.i.d.\ sample from some $P\in\cP$, let
    $\theta(P)$ be any function of $P\in\cP$, let $\hat{\theta}(X_1,\dotsc,X_n)$ be an estimator, and $d$ be a metric. For any pair $P_0,P_1\in \cP$,
	\begin{equation}\label{eq:lecam}
	\inf_{\hat{\theta}}\sup_{P\in \cP} \E_P[d(\hat{\theta},\theta(P))]
    \geq \frac{\Delta}{8} e^{-n D_{\mathrm{KL}}(P_0\|P_1)}
	\end{equation}
	where $\Delta=d(\theta(P_0),\theta(P_1))$, and $D_{\mathrm{KL}}(P_0\|P_1)=\int\log(dP_0/dP_1)dP_0$ is the KL-divergence.
\end{theorem}

While the proofs of the two theorems share a lot of similarities, they
have to use reductions to slightly different testing problems given
the different mean and variance constraints in the two results. We
begin with the proof of Theorem~\ref{thm:lowerbound:sigma}, which has a
simpler construction.

\subsection{Proof of Theorem~\ref{thm:lowerbound:sigma}}
\label{sec:proof_lower:sigma}

The basic idea of this proof is to reduce the problem of policy
evaluation to that of Gaussian mean estimation where there is a mean
associated with each $x,a$ pair. We now describe our construction.

\paragraph{Creating a family of problems} Since we aim to show a lower
bound on the hardness of policy evaluation in general, it suffices to
show a particular family of hard problem instances, such that every
estimator requires the stated number of samples on at least one of the
problems in this family. Recall that our minimax setup assumes that
$\pi$, $\mu$ and $\dx$ are fixed and the only aspect of the problem
which we can design is the conditional reward distribution
$D(r~|~x,a)$. For \Thm{lowerbound:sigma}, this choice is further
constrained to satisfy $\E[r~|~x,a] \leq \rmax(x,a)$ and
$\mbox{Var}(r~|~x,a) \leq \sigma^2(x,a)$. In order to describe our
construction, it will be convenient to define the shorthand
$\E[r~|~x,a] = \er(x,a)$. We will identify a problem in our family
with the function $\er(x,a)$ as that will be the only changing element
in our problems. For a chosen $\er$, the policy evaluation question
boils down to estimating $v_\er^\pi = \E[r(x,a)]$, where the
contexts $x$ are chosen according to $\dx$, actions are drawn from
$\pi(x,a)$ and the reward distribution $D_\er(r~|~x,a)$ is a normal
distribution with mean $\er(x,a)$ and variance $\sigma^2(x,a)$
\[
D_\er(r~|~x,a) = \cN(\er(x,a),\,\sigma^2(x,a)).
\]
Clearly this choice meets the variance constraint by construction, and
satisfies the upper bound so long as $\er(x,a) \leq \rmax(x,a)$ almost
surely. Since the evaluation policy $\pi$ is fixed throughout, we will
drop the superscript and use $v_\er$ to denote $v_\er^\pi$ in the
remainder of the proofs. With some abuse of notation, we also use
$\E_\er[\cdot]$ to denote expectations where contexts and actions are
drawn based on the fixed choices $\dx$ and $\mu$ corresponding to our
data generating distribution, and the rewards drawn from $\er$. We
further use $P_\er$ to denote this entire joint distribution over
$(x,a,r)$ triples.

Given this family of problem instances, it is easy to see that for any
pair of $\er_1, \er_2$ which are both pointwise upper bounded by
$\rmax$, we have the lower bound:
\[
R_n(\dx, \pi, \mu, \sigma^2, \rmax) \geq \inf_{\hat{v}}\max_{\er\in\er_1,\er_2}
\E_{\er} \Big[\underbrace{\mbox{$(\hat{v} -
      v_{\er})^2$}}_{\mbox{$\ell_{\er}(\hat{v})$}} \Big],
\]
where we have introduced the shorthand $\ell_{\er}(\hat{v})$ to
denote the squared error of $\hat{v}$ to $v_{\er}$. For a parameter
$\epsilon > 0$ to be chosen later, we can further lower bound this
risk for a fixed $\hat{v}$ as

\begin{align}
R_n(\hat{v})
&\geq \max_{\er\in\er_1,\er_2}
\E_{\er}[\ell_{\er}(\hat{v})] \geq \max_{\er\in\er_1,\er_2} \epsilon
\P_{\er}(\ell_{\er} \geq \epsilon) \nonumber\\
&\geq \frac{\epsilon}{2} \BigBracks{ \P_{\er_1}(\ell_{\er_1}(\hat{v}) \geq
\epsilon) + \P_{\er_2}(\ell_{\er_2}(\hat{v}) \geq \epsilon) },
\label{eq:test-upper}
\end{align}
where the last inequality lower bounds the maximum by the
average. So far we have been working with an estimation problem. We
next describe how to reduce this to a hypothesis testing problem.

\paragraph{Reduction to hypothesis testing}  For turning our
estimation problem into a testing problem, the idea is to identify a
pair $\er_1$, $\er_2$ such that they are far enough from each other so
that any estimator which gets a small estimation loss can essentially
identify whether the data generating distribution corresponds to
$P_{\eta_1}$ or $P_{\eta_2}$. In order to do this, we take any
estimator $\hat{v}$ and identify a corresponding test statistic which
maps $\hat{v}$ into one of $\er_1$, $\er_2$. The way to do this is
essentially identified in \Eq{test-upper}, and we
describe it next.

Note that since we are constructing a hypothesis test for a specific
pair of distributions $P_{\er_1}$ and $P_{\er_2}$, it is reasonable
to consider test statistics which have knowledge of $\er_1$ and
$\er_2$, and hence the corresponding distributions. Consequently,
these tests also know the true policy values $v_{\er_1}$ and
$v_{\er_2}$ and the only uncertainty is which of them gave rise to the
observed data samples. Therefore, for any estimator $\hat{v}$, we
can a associate a statistic $\phi(\hat{v}) =
\argmin_{\er}\left\{\ell_{\er_1}(\hat{v}),
\ell_{\er_2}(\hat{v})\right\}$.

Given this hypothesis test, we are interested in its error rate
$\P_\er(\phi(\hat{v}) \ne \er)$. We first relate the estimation error
of $\hat{v}$ to the error rate of the test. Suppose for now that
\begin{equation}
  \ell_{\er_1}(\hat{v}) + \ell_{\er_1}(\hat{v}) \geq 2\epsilon,
\label{eq:loss-sep}
\end{equation}
so that at least one of the losses is at least $\epsilon$.
Suppose that the data comes from $\er_1$. Then if
$\ell_{\er_1}(\hat{v}) < \epsilon$, we know that the test is correct,
because by \Eq{loss-sep} the other loss is greater than $\epsilon$,
and therefore $\phi(\hat{v})=\er_1$. This means that the error
under $\er_1$ can only occur if $\ell_{\er_1}(\hat{v})\ge\epsilon$.
Similarly, the error under $\er_2$ can only occur if $\ell_{\er_2}(\hat{v})\ge\epsilon$,
so the test error can be bounded as
\begin{align}
\nonumber
  \max_{\er \in \er_1, \er_2} \P_\er(\phi(\hat{v}) \ne \er)
  &\le
     \P_{\er_1}(\phi(\hat{v}) \ne \er_1)
  +  \P_{\er_2}(\phi(\hat{v}) \ne \er_2)
\\
\nonumber
  &\le
     \P_{\er_1}(\ell_{\er_1}(\hat{v}) \geq \epsilon)
  +  \P_{\er_2}(\ell_{\er_2}(\hat{v}) \geq \epsilon)
\\
\label{eq:test-est}
  &\leq \frac{2}{\epsilon}
  R_n(\hat{v}),
\end{align}
where the final inequality uses our earlier lower bound in
\Eq{test-upper}.

To finish connecting our the estimation problem to testing, it remains to establish
our earlier supposition \eqref{eq:loss-sep}. Assume for
now that $\er_1$ and $\er_2$ are chosen such that
\begin{equation}
  \left(v_{\er_1} - v_{\er_2}\right)^2 \geq 4\epsilon.
  \label{eq:param-sep}
\end{equation}
Then an application of the inequality $(a+b)^2 \leq 2a^2 +
2b^2$ yields
\[
4\epsilon
\leq \left(v_{\er_1} - v_{\er_2}\right)^2
\leq 2(\hat{v} - v_{\er_1})^2 + 2(\hat{v} - v_{\er_2})^2
=    2\ell_{\er_1}(\hat{v}) + 2\ell_{\er_2}(\hat{v}),
\]
which yields the posited bound~\eqref{eq:test-upper}.

\paragraph{Invoking Le Cam's argument} So far we have identified a
hypothesis testing problem and a test statistic whose error is upper
bounded in terms of the minimax risk of our problem. In order to
complete the proof, we now place a lower bound on the error of this
test statistic. Recall the result of Le Cam~\eqref{eq:lecam}, which places an
upper bound on the attainable error in any testing problem. In our
setting, this translates to

\[
\max_{\er \in \er_1, \er_2} \P_\er(\phi(\hat{v}) \ne \er) \geq
\frac{1}{8}e^{-n D_{\mathrm{KL}}(P_{\er_1}~||~P_{\er_2})}.
\]
Since the distribution of the rewards is a spherical Gaussian, the
KL-divergence is given by the squared distance between the means,
scaled by the variance, that is

\begin{align*}
  D_{\mathrm{KL}}(P_{\er_1}~||~P_{\er_2}) &= \E\left[\frac{(\er_1(x,a)
    - \er_2(x,a))^2}{2\sigma^2(x,a)}\right],
\end{align*}
where we recall that the contexts and actions are drawn from $\dx$ and
$\mu$ respectively. Since we would like the probability of error in
the test to be a constant, it suffices to choose $\er_1$ and $\er_2$
such that
\begin{equation}
\E\left[\frac{(\er_1(x,a) - \er_2(x,a))^2}{2\sigma^2(x,a)}\right] \leq
\frac{1}{n}.
\label{eq:param-close}
\end{equation}

\paragraph{Picking the parameters} So far, we have not made any
concrete choices for $\er_1$ and $\er_2$, apart from some constraints
which we have introduced along the way. Note that we have the
constraints \eqref{eq:param-sep} and~\eqref{eq:param-close} which
try to ensure that $\er_1$ and $\er_2$ are not too close that an
estimator does not have to identify the true parameter, or too far
that the testing problem becomes trivial. Additionally, we have the
upper and lower bounds of $0$ and $\rmax$ on $\er_1$ and $\er_2$. In
order to reason about these constraints, it is convenient to set
$\er_2 \equiv 0$, and pick $\er_1(x,a) = \er_1(x,a) - \er_2(x,a) =
\Delta(x,a)$. We now write all our constraints in terms of $\Delta$.

Note that $v_{\er_2}$ is now $0$, so that the first constraint
\eqref{eq:param-sep} is equivalent to
\[
v_{\er_1} = \E_{\er_1}[\rho(x,a)r(x,a)] = \E_{\Delta}[\rho(x,a)r(x,a)]
\geq 2\sqrt{\epsilon},
\]
where the importance weighting function $\rho$ is introduced since
$P_{\er_1}$ is based on choosing actions according to $\mu$ and we
seek to evaluate $\pi$. The second constraint~\eqref{eq:param-close}
is also straightforward
\[
 \E\left[ \frac{\Delta^2}{2\sigma^2} \right] \leq \frac{1}{n}.
\]
Finally, the bound $\rmax$ and non-negativity of $\er_1$ and $\er_2$ are enforced by requiring $0 \leq \Delta(x,a)
\leq \rmax(x,a)$ almost surely.

The minimax lower bound is then obtained by the largest $\epsilon$ in
the constraint~\eqref{eq:param-sep} such that the other two
constraints can be satisfied. This gives rise to the following
variational characterization of the minimax lower bound:
\begin{align}
\notag
  \max_{\Delta} \qquad & \epsilon\\
\label{eq:1:rho:r}
  \mbox{such that} \qquad & \E_{\Delta}[\rho(x,a)r(x,a)] \geq
  2\sqrt{\epsilon},  \\
\label{eq:1:Delta}
  & \E\left[ \frac{\Delta^2}{2\sigma^2} \right] \leq \frac{1}{n},
\\
\label{eq:1:rmax}
  &0 \leq \Delta(x,a) \leq \rmax(x,a).
\end{align}

Instead of finding the optimal solution, we just exhibit a feasible setting of
$\Delta$ here. We set
\begin{equation}
  \Delta = \min\left\{ \frac{\alpha \sigma^2
    \rho}{\E_\mu[\rho^2\sigma^2]},\rmax\right\}, \quad \mbox{where}
  \quad
  \alpha = \sqrt{\frac{2\E_\mu[\rho^2\sigma^2]}{n}}.
  \label{eq:delta2}
\end{equation}
This setting satisfies the bounds~\eqref{eq:1:rmax} by construction. A quick
substitution also verifies that the constraint~\eqref{eq:1:Delta} is
satisfied. Consequently, it suffices to set $\epsilon$ to the value attained
in the constraint~\eqref{eq:1:rho:r}. Substituting the value of $\Delta$ in the
constraint, we see that
\begin{align*}
  \E_{\Delta}[\rho(x,a)r(x,a)] &= \E_{x\sim \dx, a \sim \mu}
    [\rho(x,a)\Delta(x,a)]\\ &\geq \E_{x\sim \dx, a \sim \mu} \left[
      \rho \frac{\alpha \sigma^2
    \rho}{\E_\mu[\rho^2\sigma^2]}
      \1\BigParens{
      \rho\sigma^2\alpha
      \le
      \rmax\E_\mu[\rho^2\sigma^2]} \right]\\
    &= \alpha
       \left(1 -
               \frac{\E_\mu\Bracks{
                 \rho^2\sigma^2
                 \1\bigParens{\rho\sigma^2\alpha > \rmax\E_\mu[\rho^2\sigma^2]}
               }}
               {\E_\mu[\rho^2\sigma^2]}
       \right)
\\
    &\eqqcolon 2\sqrt{\epsilon}.
\end{align*}
Putting all the foregoing bounds together, we obtain that for all estimators $\hat{v}$
\begin{align*}
  R_n(\hat{v})
&\ge
  \frac{\epsilon}{2}
  \cdot
  \BigParens{
  \max_{\er \in \er_1, \er_2} \P_\er(\phi(\hat{v}) \ne \er)
  }
\\
&\ge
  \frac{\epsilon}{2}
  \cdot
  \frac{1}{8}e^{-n D_{\mathrm{KL}}(P_{\er_1}~||~P_{\er_2})}
\\
&
\ge
  \frac{\epsilon}{2}
  \cdot
  \frac{1}{8e}
= \frac{\epsilon}{16e}
\\
&=
  \frac{1}{16e}\cdot\frac{\alpha^2}{4}
       \left(1 -
               \frac{\E_\mu\Bracks{
                 \rho^2\sigma^2
                 \1\bigParens{\rho\sigma^2 > \rmax\E_\mu[\rho^2\sigma^2]/\alpha}
               }}
               {\E_\mu[\rho^2\sigma^2]}
       \right)^2
\\
&=
  \frac{\E_\mu[\rho^2\sigma^2]}{32en}
       \left(1 -
               \frac{\E_\mu\Bracks{
                 \rho^2\sigma^2
                 \1\bigParens{\rho\sigma^2 > \rmax\sqrt{n\E_\mu[\rho^2\sigma^2]/2}}
               }}
               {\E_\mu[\rho^2\sigma^2]}
       \right)^2.
\end{align*}

\subsection{Proof of Theorem~\ref{thm:lowerbound:rmax}}

We now give the proof of \Thm{lowerbound:rmax}. While it shares
a lot of reasoning with the proof of \Thm{lowerbound:sigma},
it has one crucial difference. In \Thm{lowerbound:sigma}, there
is a non-trivial noise in the reward function, unlike in
\Thm{lowerbound:rmax}. This allowed the proof to work with just
two candidate mean-reward functions, since any realization in the data
is corrupted with noise. However, in the absence of added noise, the
task of mean identification becomes rather trivial: an estimator can
just check whether $\er_1$ or $\er_2$ matches the observations
exactly.

To prevent such a strategy, we instead construct a richer family
of reward functions. Instead of merely two mean rewards, our construction will
involve a randomized design of the expected reward function from an
appropriate prior distribution. The draw of the mean reward from a
prior will essentially generate noise even though any given problem is
noiseless. The construction will also highlight the crucial sources of
difference between the contextual and multi-armed bandit problems,
since the arguments here rely on having access to
a rich context distribution, by which we mean distribution that
puts non-trivial probability on many contexts. In the absence of this property,
the bound of \Thm{lowerbound:rmax} becomes weaker.

\paragraph{Creating a family of problems}
Our family of problems will be parametrized by the two reals $\delta$ and $\lmax$
from the statement of the theorem. Our construction begins with a discretization step at the
resolution $\delta$, whose
goal is to create a countable partition of the set of pairs $\cX\times\cA$.
If sets $\cX$ and $\cA$ are countable or finite, this step is vacuous, but if the sets of contexts or actions
have continuous parts, this step is required.

First, let $\mu(x,a)$ denote the joint probability measure obtained
by first drawing $x\sim\dx$ and then $a\sim\mu(\cdot\given x)$.
In \Lem{partition}, we show that $\cX\times\cA$ can be split into countably
many disjoint sets $B_i$,
$\biguplus_{i\in\cI} B_i=\cX\times\cA$,
such that the following conditions are satisfied:
\begin{itemize}
\item Each $i\in\cI$ is associated with numbers $R_i\ge 0$, $\rho_i\ge 0$ and $\xi_i\in\set{0,1}$ such that
\[
    \rmax^2(x,a)\in[R^2_i,\,(1+\delta)R^2_i]
\enspace,
\quad
    \rho^2(x,a)\in[\rho^2_i,\,(1+\delta)\rho^2_i]
\enspace,
\quad
    \xi_{\lmax}(x,a)=\xi_i
\quad
\text{for all $(x,a)\in B_i$.}
\]
\item Each $B_i$ either satisfies $\mu(B_i)\le\delta$ or consists of a single pair $(x_i,a_i)$.
\end{itemize}
The numbers $R_i$ and $\rho_i$ will be exactly $\hR(x,a)$ and $\hrho(x,a)$
from the theorem statement.

As before, we parametrize the family of reward distributions in terms of the
mean reward function $\er(x,a)$. However, now $\er(x,a)$ is itself a
random variable, which is drawn from a prior distribution. The reward
function $\er(x,a)$ will be constant on each $B_i$, and its value
on $B_i$, written as $\er(i)$, will be drawn from a scaled Bernoulli, parametrized
by a prior function $\theta(i)$ as follows:
\begin{equation}
  \er(i) = \left\{\begin{array}{ccc} \xi_i R_i &
    \quad & \mbox{with probability $\theta(i)$,}\\ 0 & \quad &
    \mbox{with probability $1 - \theta(i)$.}\end{array}\right.
  \label{eq:dist-er}
\end{equation}

We now set $D_\er(r~|~x,a) = \er(i)$ whenever $(x,a)\in B_i$. This clearly satisfies the
constraints on the mean since $0 \leq \E[r~|~x,a] \leq R_i\le\rmax(x,a)$ from
the property of the partition, and also $\mbox{Var}(r~|~x,a) = 0$ as per the setting of
\Thm{lowerbound:rmax}. The goal of an estimator is to take $n$
samples generated by drawing $x \sim \dx$, $a~|~x \sim \mu$ and
$r~|~x,a \sim D_\er$, and output an estimate $\hat{v}$ such that
$\E_\er[(\hat{v} - v^\pi_\er)^2]$ is small. We recall our earlier
shorthand $v_\er$ to denote the value of $\pi$ under the reward
distribution generated by $\er$. For showing a lower bound on this
quantity, it is clearly sufficient to pick any prior distribution
governed by a parameter $\theta$, as in \Eq{dist-er}, and
lower bound the expectation
$\E_\theta\left[\E_\er[(\hat{v} - v_\er)^2~|~\er]\right]$. If this expectation is large for some
estimator $\hat{v}$, then there must be some realization $\er$,
which induces a large error least one function
$\er(x,a)$ which induces a large error $\E_\er[(\hat{v} - v_\er)^2~|~\er]$, as desired.
Consequently, we focus in the proof on lower
bounding the expectation $\E_\theta[\cdot]$. This expectation can be decomposed with the use
of the inequality
$a^2\ge(a+b)^2/2 - b^2$ as follows:
\begin{equation*}
  \E_\theta\BigBracks{\E_\er[(\hat{v} - v_\er)^2~|~\er]}
  \geq
  \frac{1}{2}
  \E_\theta\BigBracks{\E_\er[(\hat{v} - \E_{\theta}[v_\er])^2~|~\er]}
- \E_\theta\BigBracks{(v_\er - \E_\theta[v_\er])^2}.
\end{equation*}
Taking the worst case over all problems in the above inequality, we obtain
\begin{align}
\notag
  \sup_{\er}
  \E_\er[(\hat{v} - v_\er)^2]
&\geq
  \sup_{\theta}
  \E_\theta\BigBracks{
     \E_\er[(\hat{v} - v_\er)^2~|~\er]
  }
\\
\label{eq:lb-cauchy}
&\geq
  \underbrace{
        \sup_{\theta} \frac{1}{2}
           \E_\theta\BigBracks{\E_\er[(\hat{v}
                    - \E_{\theta}[v_\er])^2~|~\er]}}_{\term_1}
  -
  \underbrace{
        \sup_\theta
           \E_\theta\BigBracks{(v_\er - \E_\theta[v_\er])^2}}_{\term_2}.
\end{align}

This decomposition says that the expected MSE of an estimator in
estimating $v_\er$ can be related to the MSE of the same estimator in
estimating the quantity $\E_{\theta}[v_\er]$, as long as the variance
of the quantity $v_\er$ under the distribution generated by $\theta$
is not too large. This is a very important observation, since we can
now choose to instead study the MSE of an estimator in estimating
$\E_{\theta}[v_\er]$ as captured by $\term_1$. Unlike the distribution
$D_\er$ which is degenerate, this problem has a non-trivial noise
arising from the randomized draw of $\er$ according to $\theta$. Thus
we can use similar techniques as the proof of \Thm{lowerbound:sigma},
albeit where the reward distribution is a scaled Bernoulli instead of
Gaussian. For now, we focus on controlling $\term_1$, and $\term_2$
will be handled later.

In order to bound $\term_1$, we will consider two carefully designed
choices $\theta_1$ and $\theta_2$ to induce two different problem
instances and show that $\term_1$ is large for \emph{any estimator}
under one of the two parameters. In doing this, it will be convenient
to use the additional shorthand $\ell_\theta(\hat{v}) = (\hat{v} -
\E_\theta[v_\er])^2$. Proceeding as in the proof of
\Thm{lowerbound:sigma}, we have
\begin{align*}
  \term_1
  &=\frac12
      \sup_\theta
      \E_\theta\BigBracks{
         \E_\er[(\hat{v} - \E_{\theta}[v_\er])^2~|~\er]
      }
   =\frac12
      \sup_\theta
      \E_\theta\BigBracks{
         \E_\er[\ell_\theta(\hat{v})~|~\er]
      }
\\
  &\geq
  \frac{\epsilon}{2}
  \sup_\theta \P_\theta\left(\ell_\theta(\hat{v}) \geq
  \epsilon\right)
  \geq
  \frac{\epsilon}{2}
  \max_{\theta \in\theta_1,
    \theta_2} \P_\theta\left(\ell_\theta(\hat{v}) \geq
  \epsilon\right)
\\
  &\geq \frac{\epsilon}{4}
  \BigBracks{
  \P_{\theta_1}\left(\ell_{\theta_1}(\hat{v})\geq \epsilon\right)
+ \P_{\theta_2}\left(\ell_{\theta_2}(\hat{v})\geq \epsilon\right)
  }.
\end{align*}

\paragraph{Reduction to hypothesis testing} As in the proof of
\Thm{lowerbound:sigma}, we now reduce the estimation problem
into a hypothesis test for whether the data is generated according
to the parameter $\theta_1$ or $\theta_2$. The arguments here are
similar to the earlier proof, so we will be terser in this
presentation.

As before, our hypothesis test has entire knowledge of $D_\er$ as well
as $\theta_1$ and $\theta_2$. Consequently, we construct a test based
on picking $\theta_1$ whenever $\ell_{\theta_1}(\hat{v}) \leq
\ell_{\theta_2}(\hat{v})$. As before, we will ensure that
$|\E_{\theta_1}[v_\er] - \E_{\theta_2}[v_\er]| \geq 2\sqrt{\epsilon}$
so that for any estimator $\hat{v}$, we have
\[
\ell_{\theta_1}(\hat{v}) + \ell_{\theta_2}(\hat{v}) \geq 2\epsilon.
\]
Under this assumption, we can similarly conclude that
the error of our hypothesis test is at most
\[
  \P_{\theta_1}\left(\ell_{\theta_1}(\hat{v}) \geq
  \epsilon\right) + \P_{\theta_2}\left(\ell_{\theta_2}(\hat{v}) \geq
  \epsilon\right).
\]

\paragraph{Invoking Le Cam's argument} Once again, we can lower bound
the error rate of our test by invoking the result of Le Cam. This
requires an upper bound on the KL-divergence $D_{\mathrm{KL}}(P_{\theta_1} \|
P_{\theta_2})$. The only difference from our earlier argument is that
these distributions are now Bernoulli instead of Gaussian,
based on the
construction in \Eq{dist-er}. More formally, we have
\begin{align}
\notag
D_{\mathrm{KL}}(P_{\theta_1}\|P_{\theta_2})
&=
\sum_{i\in\cI}\sum_{r\in\set{0,xi_iR_i}}
    \log\Parens{\frac{p(r;\theta_1(i))}{p(r;\theta_2(i))}}
    p(r;\theta_1(i))\mu(B_i)
\\
\label{eq:2:KL}
&=
   \E_\mu \BigBracks{\,\xi_i D_{\mathrm{KL}}
              \bigParens{\,\text{Ber}(\theta_1(i))\bigm\lVert\text{Ber}(\theta_2(i))\,}
   \,},
\end{align}
where $i$ is treated as a random variable under $\mu$, and $\xi_i$ is included, because
the two distributions assign $r=0$ with probability one if $\xi_i=0$.

\paragraph{Picking the parameters} It remains to carefully choose $\theta_1$ and $\theta_2$.
We define
$\theta_2(i) \equiv 0.5$, and let $\theta_1(i) = \theta_2(i) +
\Delta_i$, where $\Delta_i$ will be chosen to satisfy certain
constraints as before. Then, by Lemma~\ref{lem:bernoulli_KL}, the KL divergence in \Eq{2:KL}
can be bounded as
\begin{align*}
D_{\mathrm{KL}}(P_{\theta_1}\|P_{\theta_2}) &\leq \frac{1}{4} \E_\mu
\left[\xi_i\Delta_i^2\right].
\end{align*}
It remains to choose $\Delta_i$. Following a similar logic as before, we seek to find a good feasible solution
of the maximization problem
\begin{align}
\notag
  \max_{\Delta} \qquad & \epsilon
\\
\label{eq:2:rho:r}
  \mbox{such that} \qquad & \E_{\mu}\bigBracks{\rho(x,a)\xi_i\Delta_i R_i} \geq
  2\sqrt{\epsilon},
\\
\label{eq:2:Delta}
  & \frac{1}{4}\E_\mu\left[\xi_i \Delta_i^2 \right] \leq \frac{1}{n},
\\
\label{eq:2:rmax}
  &0 \leq \Delta_i \leq 0.5.
\end{align}
For some $\alpha > 0$ to be determined shortly, we set
\[
  \Delta_i =
    \min\Braces{
      \frac{\xi_i\rho_i R_i\alpha}
           {\E_\mu[\xi_i\rho_i^2 R_i^2]},\,
      0.5
    }.
\]

The bound constraint~\eqref{eq:2:rmax} is satisfied by construction and we set $\alpha
= \sqrt{4\E_\mu[\xi_i\rho_i^2 R_i^2]/n}$ to satisfy the constraint~\eqref{eq:2:Delta}.
To obtain a feasible choice of $\epsilon$, we bound $\E_{\mu}[\rho(x,a)\xi_i\Delta_i R_i]$ as follows:
\begin{align*}
  \E_{\mu}\bigBracks{\rho(x,a)\xi_i\Delta_iR_i}
&\ge
  \E_{\mu}\bigBracks{\xi_i\rho_i\Delta_iR_i}
\\
&\ge
  \E_{\mu}\Bracks{
     \frac{\xi_i \rho_i^2 R_i^2\alpha
           \1\bigParens{
             \xi_i\rho_i R_i\leq \E_{\mu}[\xi_i \rho_i^2 R_i^2]/2\alpha
     }}
     {\E_{\mu}[\xi_i \rho_i^2 R_i^2]}
    }
\\
&=
  \alpha\Parens{
      1 - \frac{\E_{\mu}\left[\xi_i \rho_i^2 R_i^2 \1(\xi_i\rho_i R_i
      > \E_{\mu}[\xi_i \rho_i^2 R_i^2]/2\alpha)
      \right]}{\E_{\mu}[\xi_i \rho_i^2 R_i^2]}
  }
\\
&\eqqcolon 2\sqrt{\epsilon}.
\end{align*}

Collecting our arguments so far, we have established that
\begin{align*}
  \term_1
&\geq
  \frac{\epsilon}{4}
  \cdot
  \BigParens{
    \P_{\theta_1}\left(\ell_{\theta_1}(\hat{v}) \geq \epsilon\right)
    +
    \P_{\theta_2}\left(\ell_{\theta_2}(\hat{v}) \geq \epsilon\right)
  }
\\
&\geq
  \frac{\epsilon}{4}
  \cdot
  \frac{1}{8}e^{-n D_{\mathrm{KL}}(P_{\theta_1}\|P_{\theta_2})}
\\&
\ge
  \frac{\epsilon}{4}
  \cdot
  \frac{1}{8e}
=
  \frac{\epsilon}{32e}
\\
&=
  \frac{1}{32e}\cdot\frac{\alpha^2}{4}
  \Parens{
      1 - \frac{\E_{\mu}\left[\xi_i \rho_i^2 R_i^2 \1(\xi_i\rho_i R_i
      > \E_{\mu}[\xi_i \rho_i^2 R_i^2]/2\alpha)
      \right]}{\E_{\mu}[\xi_i \rho_i^2 R_i^2]}
  }^2
\\
&=
  \frac{\E_\mu[\xi_i\rho_i^2 R_i^2]}{32en}
  \Parens{
      1 - \frac{\E_{\mu}\left[\xi_i \rho_i^2 R_i^2 \1\Parens{\xi_i\rho_i R_i
      > \sqrt{n\E_{\mu}[\xi_i \rho_i^2 R_i^2]/16}}
      \right]}{\E_{\mu}[\xi_i \rho_i^2 R_i^2]}
  }^2
\enspace.
\end{align*}
In order to complete the proof, we need to further upper bound
$\term_2$ in the decomposition~\eqref{eq:lb-cauchy}.
%
%
\paragraph{Bounding $\term_2$}
We need to bound the supremum over all priors $\theta$. Consider an
arbitrary prior $\theta$ and assume that $\er$ is drawn according to \Eq{dist-er}.
To bound $\E_\theta\bigBracks{(v_\er - \E_\theta [v_\er])^2}$, we view
$(v_\er-\E_\theta[v_\er])^2$ as a random variable under $\theta$ and bound
it using Hoeffding's inequality.

We begin by bounding its range. From the definition of $\er$ and $v_\er$,
\begin{align*}
   0&\le v_\er\le \E_\pi[\xi_i R_i]=\E_\mu[\rho(x,a)\xi_i R_i]
     \le(1+\delta)^{1/2}\E_\mu[\xi_i\rho_i R_i]
\enspace,
\end{align*}
so also $0\le \E_\theta[v_\er]\le(1+\delta)^{1/2}\E_\mu[\xi_i\rho_i R_i]$.
Hence,
$\Abs{v_\er-\E_\theta[v_\er]} \le (1+\delta)^{1/2}\E_\mu[\xi_i\rho_i R_i]$, and we obtain the bound
\begin{equation}
\label{eq:finite:range:new}
(v_\er-\E_\theta[v_\er])^2
\le
(1+\delta)(\E_\mu[\xi_i\rho_i R_i])^2
\le
(1+\delta)\E_\mu[\xi_i\rho_i^2 R_i^2].
\end{equation}

The proof proceeds by applying Hoeffding's inequality to control the
probability that $(v_\er-\E_\theta[v_\er])^2\ge t^2$ for a suitable
$t$. Then we can, with high probability, use the bound
$(v_\er-\E_\theta[v_\er])^2\ge t^2$, and with the remaining small
probability apply the bound of \Eq{finite:range:new}.

To apply Hoeffding's inequality, we write $v_\er$
  explicitly as
\[
  v_\er = \sum_{i\in\cI} \mu(B_i)\rho'_i\er_i\eqqcolon\sum_{i\in\cI} Y_i
\]
where $\rho_i'\coloneqq\E_\mu[\rho(x,a)\given (x,a)\in B_i]$. Thus, $v_\er$
can be written as a sum of countably many independent variables, but we can
only apply Hoeffding's inequality to their finite subset. Note that
the variables $Y_i$ are non-negative and upper-bounded by a
summable series, namely $Y_i\le\mu(B_i)\rho'_i R_i$, where the summability follows
because $\E_\mu[\rho\rmax]\le 1+\E_\mu[\rho^2 \rmax^2]<\infty$. This means that for any $\delta_0>0$,
we can choose a finite set $\cI_0$ such that $\sum_{i\not\in\cI_0} Y_i\le\delta_0$.
We will determine the sufficiently small value of $\delta_0$ later; for now,
consider the corresponding set $\cI_0$ and define an auxiliary variable
\[
 v'_\er\coloneqq\sum_{i\in\cI_0} Y_i
\enspace,
\]
which by construction satisfies $v'_\er\le v_\er\le v'_\er+\delta_0$. Note that the
summands $Y_i$ can be bounded as
\[
 0\le Y_i\le\xi_i\rho'_i R_i\mu(B_i)
         \le\xi_i(1+\delta)^{1/2}\rho_i R_i\sqrt{\mu(B_i)}\sqrt{\lmax'}
\]
because $\rho'_i\le(1+\delta)^{1/2}\rho_i$ and $\xi_i\mu(B_i)\le\xi_i\sqrt{\mu(B_i)}\sqrt{\lmax'}$, because $\xi_i=0$ whenever $\mu(B_i)>\max\set{\lmax,\delta}=\lmax'$.
By Hoeffding's inequality, we thus have
\begin{align*}
\P(|v'_\er - \E_\theta v'_\er|\geq t)
&\leq
  2\exp\left\{-\frac{2t^2}
  {(1+\delta)
   \sum_{i\in\cI_0}
   \xi_i\rho_i^2 R_i^2\mu(B_i)\lmax'
  }
  \right\}
\\
&\le
  2\exp\left\{-\frac{2t^2}{(1+\delta)\lmax'\E_\mu [\xi_i\rho_i^2 R_i^2]}\right\}
.
\end{align*}
Now take $t = \sqrt{\lmax'\log(4/\lmax')(1+\delta)\E_\mu[\xi_i\rho_i^2 R_i^2]/2}$ in
the above bound, which yields
\[
\P\BigBracks{(v'_\er - \E_\theta v'_\er)^2\geq t^2}
=
\P\BigBracks{|v'_\er - \E_\theta v'_\er|\geq t}
\le
\frac{\lmax'}{2}
\enspace.
\]
Now, we can go back to analyzing $v_\er$. We set $\delta_0$ sufficiently small, so
$t+\delta_0\le\sqrt{\lmax'\log(5/\lmax')(1+\delta)\E_\mu [\xi_i\rho_i^2 R_i^2]/2}$.
Thus, using \Eq{finite:range:new}, we have
\begin{align*}
\E_\theta\Bracks{(v_\er - \E_\theta v_\er)^2}
&\le
(t+\delta_0)^2\cdot\P\BigBracks{(v_\er - \E_\theta v_\er)^2 <(t+\delta_0)^2}
\\
&\qquad\qquad{}
+
(1+\delta)\E_\mu[\xi_i\rho_i^2 R_i^2]\cdot
\P\BigBracks{(v_\er - \E_\theta v_\er)^2\geq (t+\delta_0)^2}
\\
&\le
(t+\delta_0)^2
+
(1+\delta)\E_\mu[\xi_i\rho_i^2 R_i^2]\cdot
\P\BigBracks{(v'_\er - \E_\theta v'_\er)^2\geq t^2}
\\
&=
\frac{\lmax'\log(5/\lmax')(1+\delta)\E_\mu[\xi_i\rho_i^2 R_i^2]}{2}
+
(1+\delta)\E_\mu[\xi_i\rho_i^2 R_i^2]\cdot\frac{\lmax'}{2}
\\
\tag*{\qedhere}
&
\le
\lmax'\log(5/\lmax')(1+\delta)\E_\mu[\xi_i\rho_i^2 R_i^2]
\enspace.
\end{align*}
Combining this bound with the bound on $\term_1$ yields the theorem.

\begin{lemma}
\label{lem:partition}
Let $\cZ\coloneqq\cX\times\cA$ be a subset of $\R^d$, let $\mu$ be a probability measure
on $\cZ$ and $\rmax$ and $\rho$ be non-negative measurable functions on $\cZ$.
Given $\lmax\in[0,1]$, define a random variable $\xi_{\lmax}(z)\coloneqq\1(\mu(z)\le\lmax)$.
Then for any $\delta\in(0,1]$, there exists
a countable index set $\cI$ and disjoint sets $B_i\subseteq\cZ$ alongside non-negative
reals $R_i$, $\rho_i$ and $\xi_i\in\set{0,1}$ such that the following conditions
hold:
\begin{itemize}[noitemsep]
\item Sets $B_i$ form a partition of $\cZ$, i.e., $\cZ=\uplus_{i\in\cI} B_i$.
\item Reals $R_i$ and $\rho_i$ approximate $\rmax$ and $\rho$, and $\xi_i$ equals $\xi_{\lmax}$ as follows:
\[
    \rmax^2(z)\in[R^2_i,\,(1+\delta)R^2_i]
\enspace,
\quad
    \rho^2(z)\in[\rho^2_i,\,(1+\delta)\rho^2_i]
\enspace,
\quad
    \xi_{\lmax}(z)=\xi_i
\quad
\text{for all $z\in B_i$.}
\]
\item Each set $B_i$ either satisfies $\mu(B_i)\le\delta$ or consists of a single $z\in\cZ$.
\end{itemize}
\end{lemma}
\begin{proof}
Let $\cZ\coloneqq\cX\times\cA$. We begin our construction by separating out atoms, i.e.,
the elements $z\in\cZ$ such that $\mu(z)>0$. Specifically, we
write $\cZ=\cZna\uplus\cZa$ where $\cZa$ consists of atoms and $\cZna$ of all non-atoms. The set $\cZa$ is either
finite or countably infinite, so $\cZna$ is measurable.

By a theorem of \citet{Sierpinski22}, since $\mu$ does not have any atoms on $\cZna$,
it must be continuous on $\cZna$ in the sense that
if $A$ is a measurable subset of $\cZna$ with $\mu(A)=a$ then for any $b\in[0,a]$, there exists a measurable set $B\subseteq A$
such that $\mu(B)=b$. This means that we can decompose $\cZna$ into $N\coloneqq\lceil 1/\delta\rceil$ sets $\cZna_1,\cZna_2,\dotsc,\cZna_N$
such that each has a measure at most $\delta$ and $\cZna=\biguplus_{j=1}^N\cZna_j$.

We next ensure the approximation properties for $\rmax$ and $\rho$. We begin by a countable decomposition of non-negative reals.
We consider the countable index set $\cJ\coloneqq\Z\cup\set{-\infty}$ and define the sequence $a_j\coloneqq(1+\delta)^{j/2}$, for $j\in\Z$.
Positive reals can then be decomposed into the following intervals indexed by $\cJ$:
\[
   I_{-\infty}\coloneqq\set{0}
\enspace,
\qquad
   I_j\coloneqq(a_j,a_{j+1}]
\quad
\text{for $j\in\Z$.}
\]
It will also be convenient to set $a_{-\infty}\coloneqq0$. Thus, the construction of $I_j$ guarantees that for all $j\in\cJ$ and all
$t\in I_j$ we have $a_j^2\le t^2\le(1+\delta)a_j^2$.

The desired partition, with the index set $\cI=\cZa\;\cup\;[N]\times\cJ^2$, is as follows:
\begin{align*}
 &\text{for $i=z\in\cZa$}:
&&
  B_i\coloneqq\set{z},\;
  R_i\coloneqq\rmax(z),\;
  \rho_i\coloneqq\rho(z),\;
  \xi_i\coloneqq \xi_{\lmax}(z);
\\
 &\text{for $i=(j,j_R,j_\rho)\in[N]\times\cJ^2$:}
&&
  B_i\coloneqq\cZna_j\cap\rmax^{-1}(I_{j_R})\cap\rho^{-1}(I_{j_\rho}),
\\
\tag*{\qedhere}
 &
&&
  R_i\coloneqq a_{j_{R}},\;
  \rho_i\coloneqq a_{j_{\rho}},\;
  \xi_i\coloneqq 1.
\end{align*}
\end{proof}

\section{Proof of Theorem~\ref{thm:MSEbound}}
\label{sec:otherproofs}

	

  Let $A_x\coloneqq\left\{a\in \cA :\:
  \rho(x,a) \leq\tau\right\}$. For brevity, we write
  $A_{i}\coloneqq A_{x_i}$. We decompose the mean squared error into
  the squared bias and variance and
  control each term separately,
\[
  \text{MSE}(\hat{v}_{\mathrm{\OA}}) =
  \bigAbs{\E[\hat{v}_{\mathrm{\OA}}]- v^\pi}^2 +
  \Var[\hat{v}_{\mathrm{\OA}}].
\]

  We first calculate the bias. Note that bias is incurred only in the
  terms that fall in $A^c_x$, so
  \begin{align*}
    \E[\hat{v}_{\mathrm{\OA}}]- v^\pi &= \E\left[\sum_{a\in A_x^c}
      \hat{r}(x,a) \pi(a|x) \right]  - \E\left[\sum_{a\in
        A_x^c} \E[r|x,a]\,\pi(a|x)\right]\\
    &=  \E_\pi \BigBracks{\bigParens{\hat{r}(x,a)-\E[r|x,a]}
                          ~\mathbf{1}(a\in A_x^c)
               }
\\
    &=  \E_\pi \bigBracks{\,\epsilon(x,a)
                          \,\mathbf{1}(\rho>\tau)
               \,}
  \end{align*}
  where we recall that $\epsilon(x,a) = \hat{r}(x,a) - \E[r|x,a]$.

  Next we upper bound the variance. Note that the variance
  contributions from the IPS part and the DM part are not independent,
  since the indicators $\rho(x_i,a) > \tau$ and $\rho(x_i,a) \leq
  \tau$ are mutually exclusive. To simplify the analysis, we use the
  following inequality that holds for any random variable $X$ and $Y$:
  $$
  \Var(X+Y) \leq  2\Var(X) + 2\Var(Y).
  $$
  This allows us to calculate the variance of each part separately.
  \begin{align*}
    \Var [\hat{v}_{\mathrm{\OA}}]
&\leq
    2\,
    \Var\!\Bracks{
        \frac{1}{n}\sum_{i=1}^n \left[r_i \rho_i \1(a_i\in
        A_i)\right]
    }
    +
    2\, \Var\!\Bracks{
    \frac{1}{n}\sum_{i=1}^n\sum_{a\in \cA} \hat{r}(x_i,a)\pi(a|x_i)
    \1(a \in A^c_i)
    }
\\
&=
    \frac{2}{n}\, \Var_\mu \bigBracks{r \rho \1(a\in A_x)}
    +
    \frac{2}{n}\, \Var\!\Bracks{
    \sum_{a\in A_x^c} \hat{r}(x,a)\pi(a|x)
    }
\\
&=
    \frac{2}{n}\,
    \E_\mu \Var\bigBracks{r \rho \1(a\in A_x)\bigGiven x,a}
    +
    \frac{2}{n}\, \Var_\mu \E\bigBracks{r\rho \1(a\in A_x)\bigGiven x,a}
    +
    \frac{2}{n}\, \Var\!\Bracks{
    \sum_{a\in A_x^c}\hat{r}(x,a)\pi(a|x)
    }
\\
&\leq
    \frac{2}{n}\,
    \E_\mu \Var\bigBracks{r \rho \1(a\in A_x)\bigGiven x,a}
    +
    \frac{2}{n}\, \E_\mu\BigBracks{\,\E[r\rho \1(a\in A_x)\given x,a]^2\,}
    +
    \frac{2}{n}\, \E\BiggBracks{\,
    \BigParens{\,\sum_{a\in A_x^c}\hat{r}(x,a)\pi(a|x)\,}^2
    \,}
\\
&\leq
    \frac{2}{n}\, \E_\mu\bigBracks{\sigma^2 \rho^2 \1(a\in A_x)}
    +
    \frac{2}{n}\, \E_\mu\bigBracks{R_{\max}^2\rho^2 \1(a\in A_x)}
    +
    \frac{2}{n}\, \E\BiggBracks{\,
    \BigParens{\,\sum_{a\in A_x^c}\hat{r}(x,a)\pi(a|x)\,}^2
    \,}.
  \end{align*}

  To complete the proof, note that the last term is further upper
  bounded using Jensen's inequality as
  \begin{align*}
    \E\BiggBracks{\,
    \BigParens{\,\sum_{a\in A_x^c}\hat{r}(x,a)\pi(a|x)\,}^2
    \,}
&=
    \E\BiggBracks{\,
    \BigParens{\,
      \sum_{a \in A^c_x} \pi(a|x)
    \,}^2
    \,
    \biggParens{\,
      \sum_{a\in A_x^c}
        \frac{\hat{r}(x,a)\pi(a|x)}{\sum_{a \in A^c_x}\pi(a|x)}
    \,}^2
    \,}
\\
&\leq
    \E\BiggBracks{\,
    \BigParens{\,
      \sum_{a \in A^c_x}\pi(a|x)
    \,}
    \;
    \biggParens{\,
      \sum_{a \in A^c_x}\hat{r}(x,a)^2\pi(a|x)
    \,}
    \,}
\\
&\leq
    \E_\pi\bigBracks{ \rmax^2 \1(\rho > \tau) },
\end{align*}
  where the final inequality uses $\sum_{a \in A^c_x} \pi(a|x) \leq
  1$ and $\hat{r}(x,a) \in [0, \rmax(x,a)]$ almost surely.

  Combining the bias and variance bounds, we get the stated MSE upper bound.
  \qed

\section{Utility Lemmas}\label{sec:utilitylemmas}
\begin{lemma}[{\citealp[Theorem~2]{hoeffding1963probability}}]\label{lem:hoeffding}
  Let $X_i \in [a_i,b_i]$ and $X_1,...,X_n$ are drawn independently. Then the empirical mean $\bar{X}= \frac{1}{n}(X_1 + ... + X_n)$ obeys
  $$
  \P(|\bar{X} - \E[\bar{X}]| \geq t) \leq 2 e^{-\frac{2n^2t^2}{\sum_{i=1}^n(b_i-a_i)^2}}.
$$
\end{lemma}

\begin{lemma}[Bernoulli KL-divergence] \label{lem:bernoulli_KL}
  For $0<p,q<1$, we have
  $$D_{\mathrm{KL}}(\mathrm{Ber}(p)\|\mathrm{Ber}(q)) \leq (p-q)^2(\frac{1}{q}+\frac{1}{1-q}).$$
\end{lemma}
\begin{proof}
  \begin{align*}
    D_{\mathrm{KL}}(\text{Ber}(p)\|\text{Ber}(q)) &=  p\log\Parens{\frac{p}{q}} + (1-p)\log\Parens{\frac{1-p}{1-q}} \\
    &\leq p  \frac{p-q}{q} + (1-p)\frac{q-p}{1-q} = \frac{(p-q)^2}{q} + (p-q)  + \frac{(p-q)^2}{1-q} + (q-p)
\\
\tag*{\qedhere}
    &= (p-q)^2 \Parens{\frac{1}{q}+\frac{1}{1-q}}.
  \end{align*}
\end{proof}

\section{Additional Figures from the Experiments}
\label{sec:experiments}
\begin{figure*}
	\centering
						\includegraphics[width=0.8\textwidth]{figs/legend}
						\medskip

	\begin{subfigure}[t]{0.39\textwidth}
		\centering
		\includegraphics[width=\textwidth]{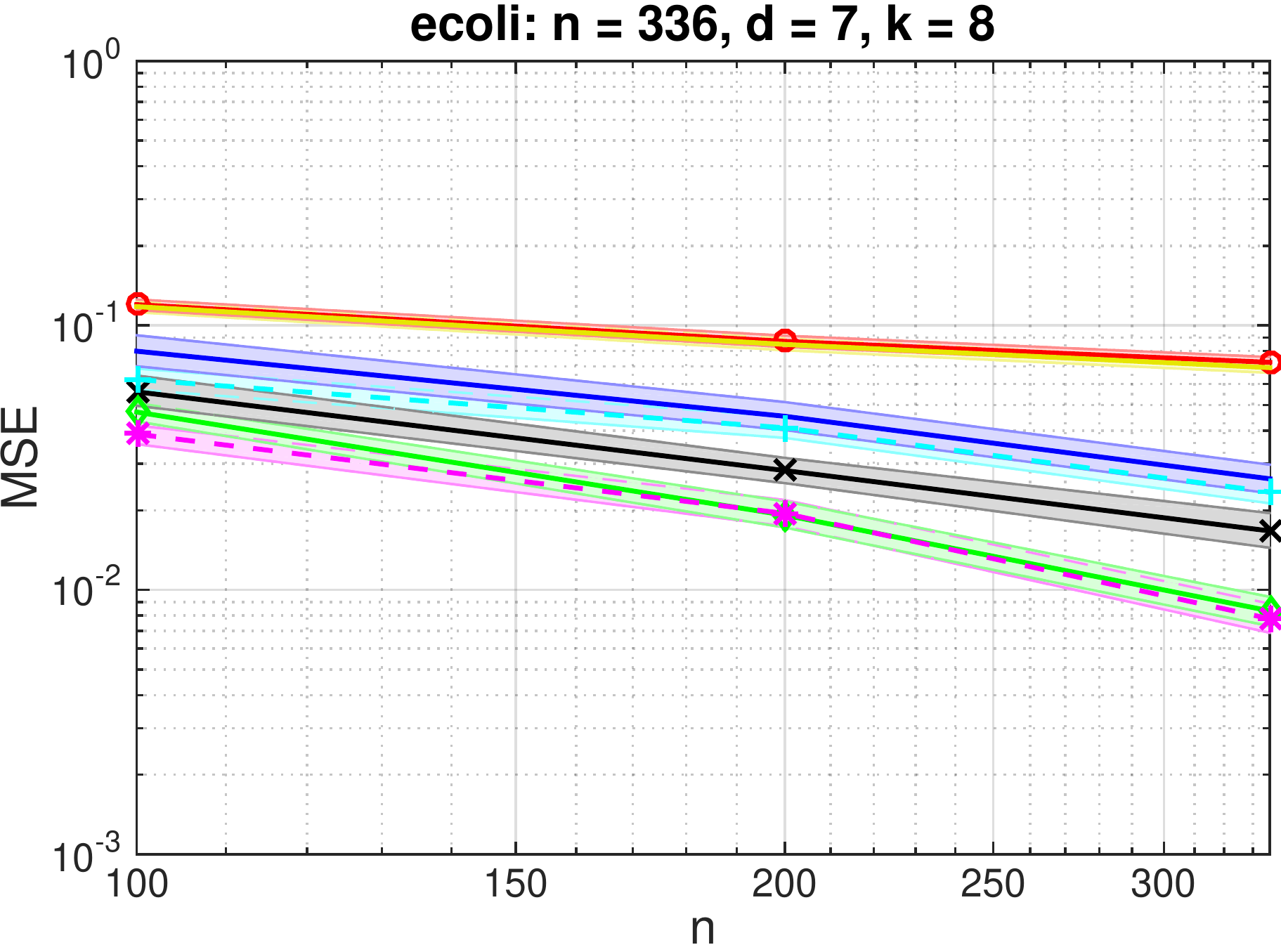}
		\figsqueeze\caption{ecoli / deterministic reward}\label{fig:raw-optecoli}		
	\end{subfigure}
	\hspace{0.25in}
			\begin{subfigure}[t]{0.39\textwidth}
				\centering
				\includegraphics[width=\textwidth]{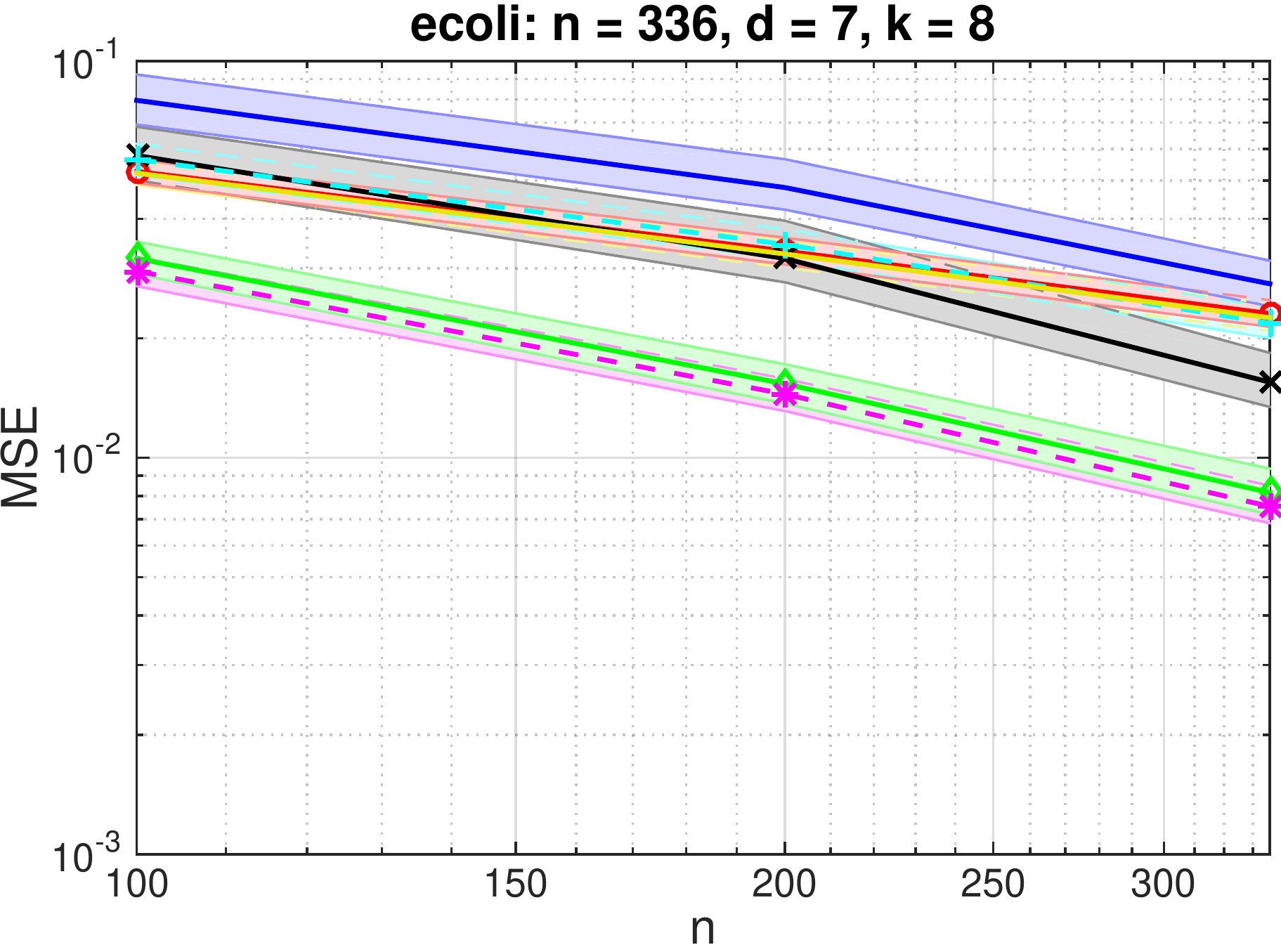}
				\figsqueeze\caption{ecoli / noisy reward}\label{fig:noisy-ecoli}		
			\end{subfigure}\\
	\medskip
	
	\begin{subfigure}[t]{0.39\textwidth}
		\centering
		\includegraphics[width=\textwidth]{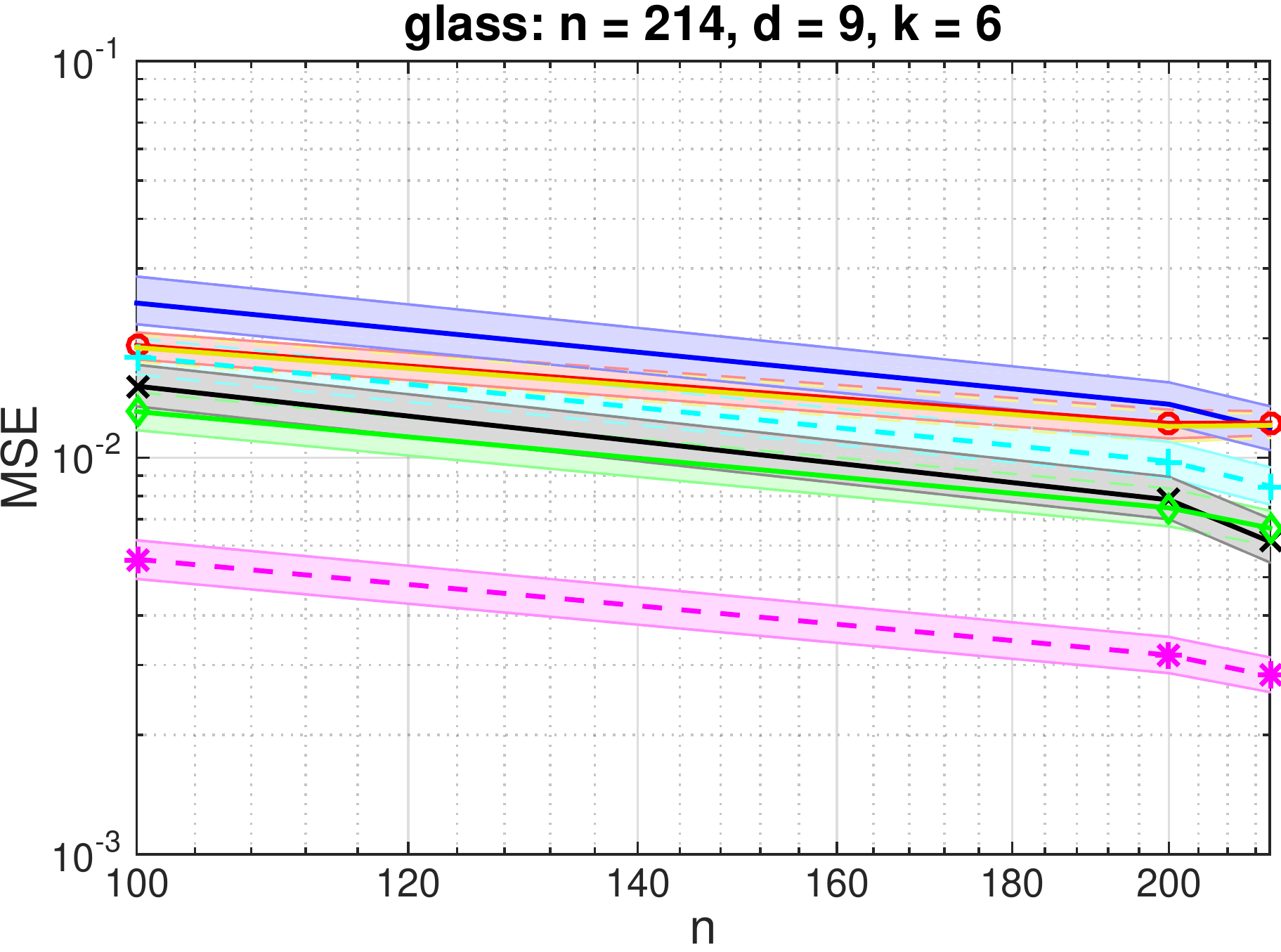}
		\figsqueeze\caption{glass / deterministic reward}\label{fig:raw-glass}		
	\end{subfigure}
		\hspace{0.25in}
		\begin{subfigure}[t]{0.39\textwidth}
			\centering
			\includegraphics[width=\textwidth]{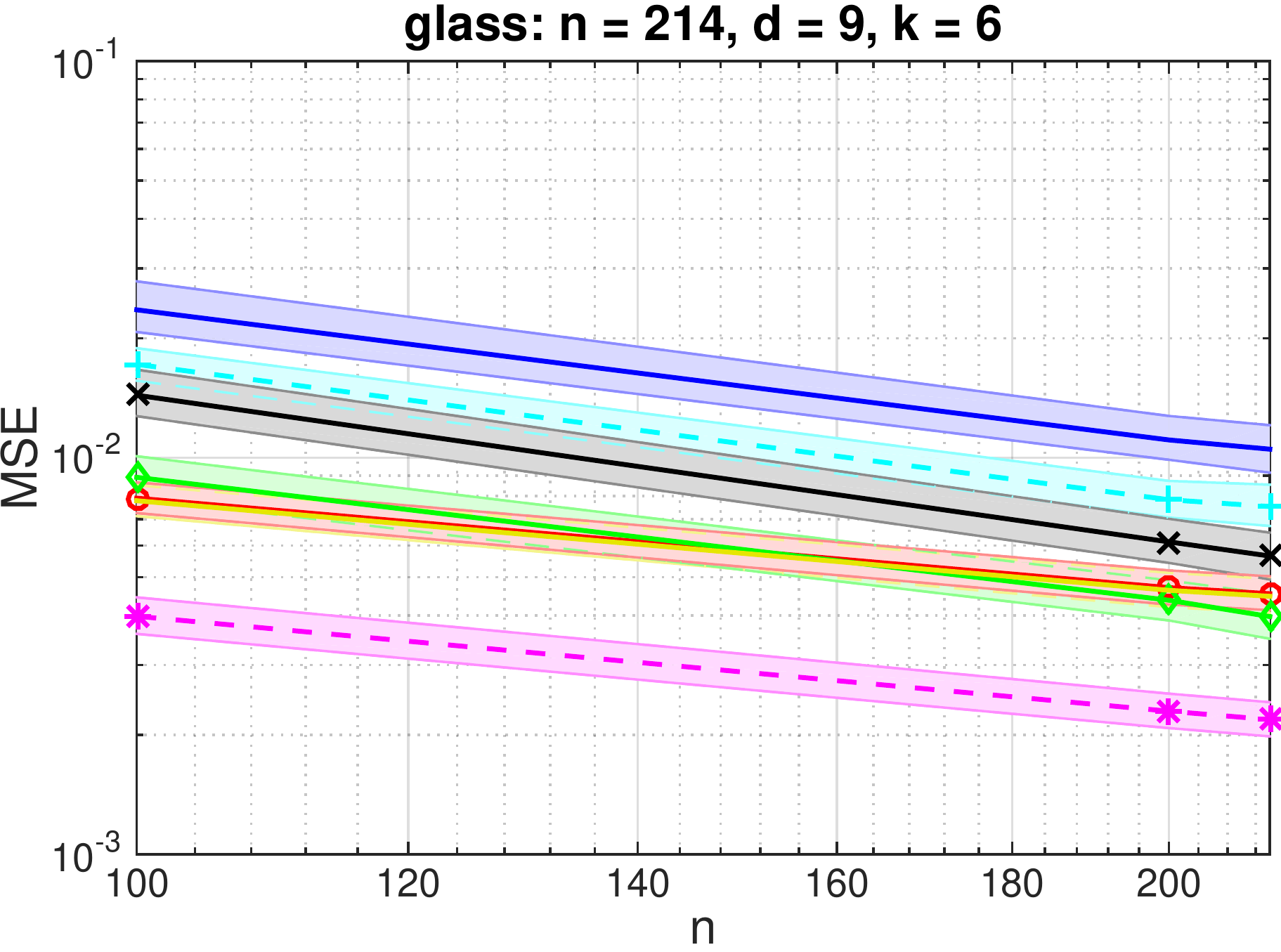}
			\figsqueeze\caption{glass / noisy reward}\label{fig:noisy-glass}		
		\end{subfigure}
		\\		\medskip
	
	\begin{subfigure}[t]{0.39\textwidth}
		\centering
		\includegraphics[width=\textwidth]{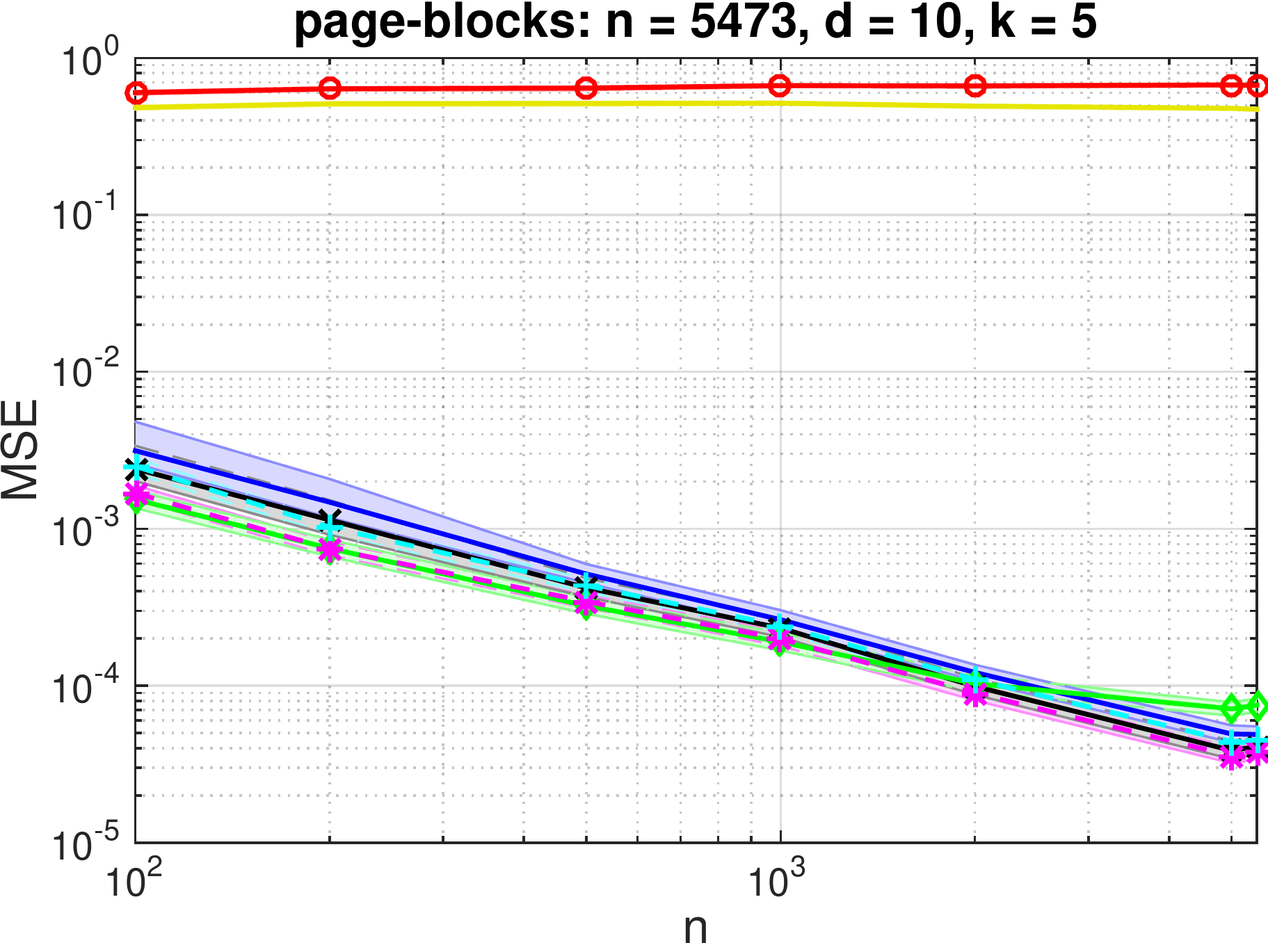}
		\figsqueeze\caption{page-blocks / deterministic reward}\label{fig:raw-page-blocks}		
	\end{subfigure}
	\hspace{0.25in}
		\begin{subfigure}[t]{0.39\textwidth}
			\centering
			\includegraphics[width=\textwidth]{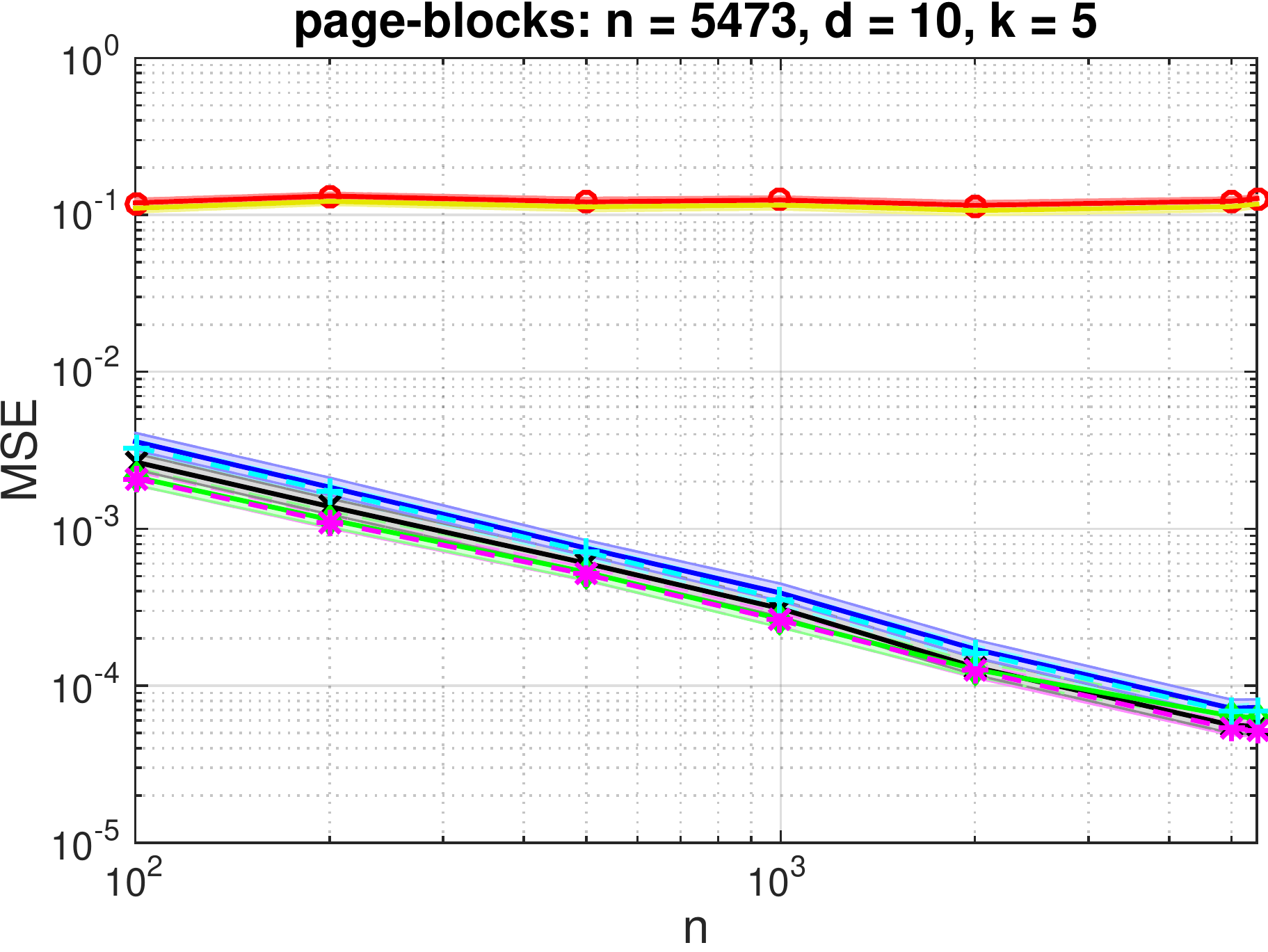}
			\figsqueeze\caption{page-blocks / noisy reward}\label{fig:noisy-page-blocks}		
		\end{subfigure}
	\\
	\medskip
				\begin{subfigure}[t]{0.39\textwidth}
					\centering
					\includegraphics[width=\textwidth]{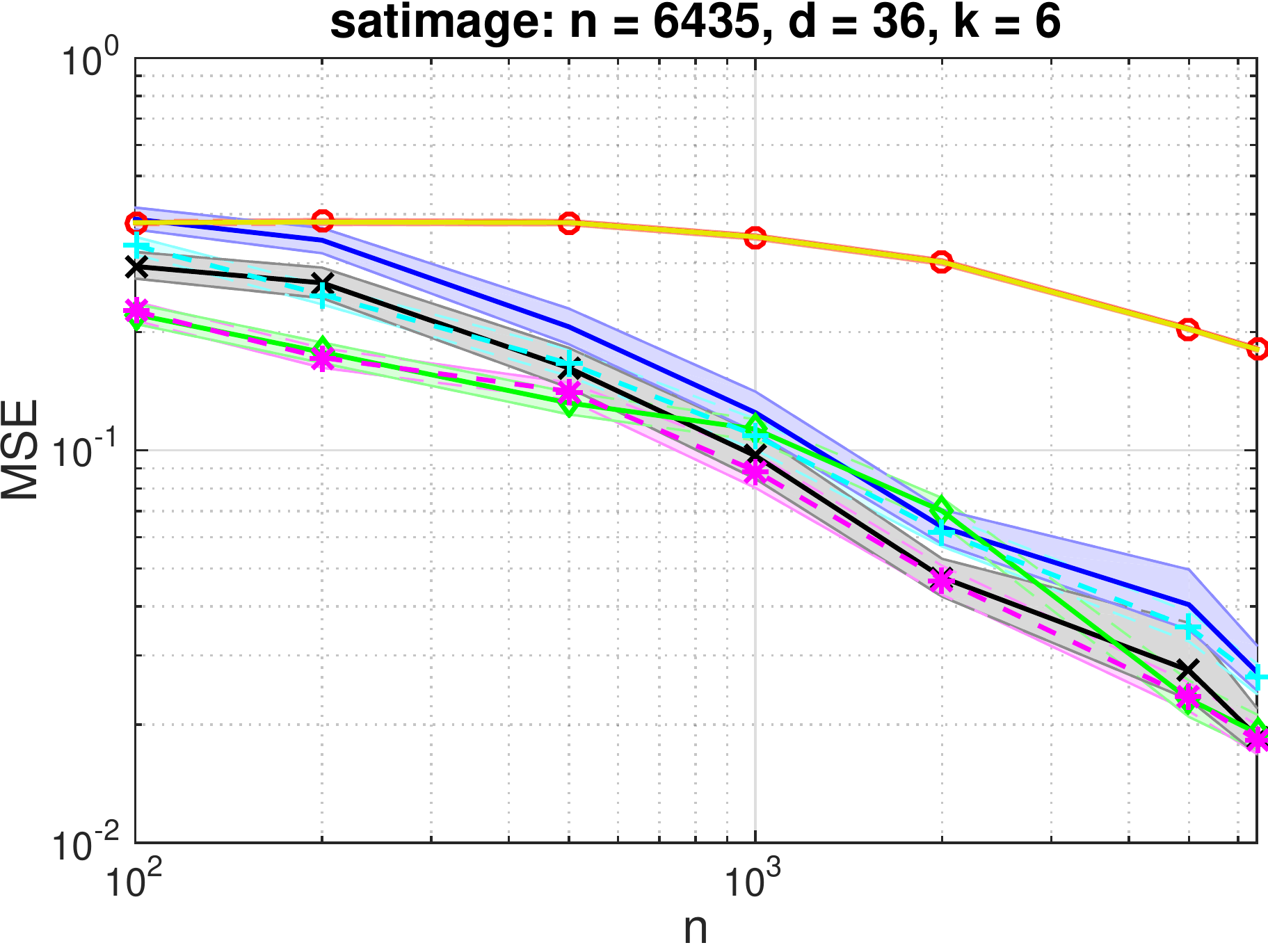}
					\figsqueeze\caption{satimage / deterministic reward}\label{fig:raw-optsatimage}		
				\end{subfigure}
					\hspace{0.25in}
			\begin{subfigure}[t]{0.39\textwidth}
				\centering
				\includegraphics[width=\textwidth]{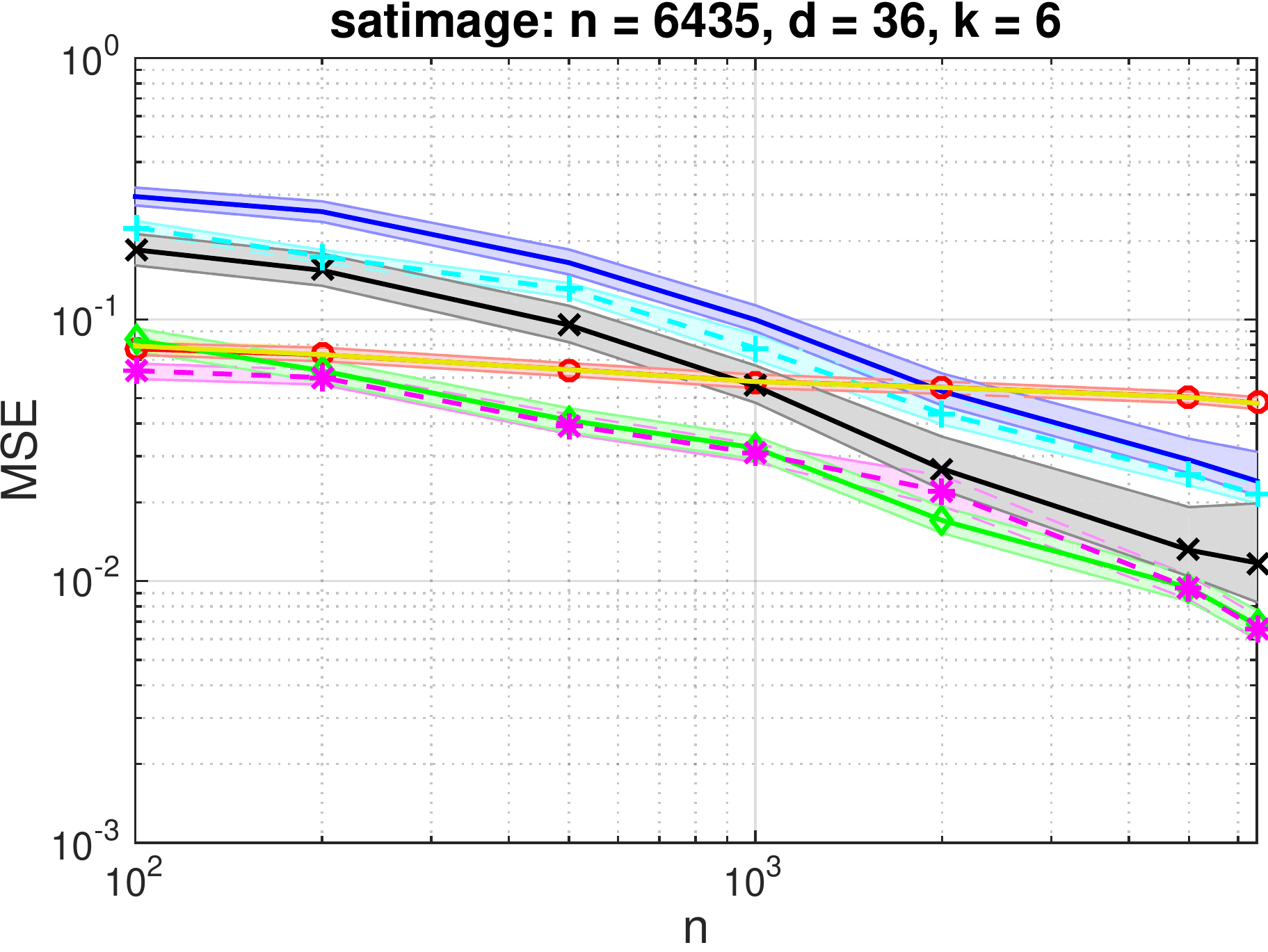}
				\figsqueeze\caption{satimage / noisy reward}\label{fig:noisy-satimage}		
			\end{subfigure}
	\end{figure*}
\begin{figure*}
  \centering
  		
  		\includegraphics[width=0.8\textwidth]{figs/legend}
  				\medskip
  		
	\begin{subfigure}[t]{0.39\textwidth}
	  \centering
	  \includegraphics[width=\textwidth]{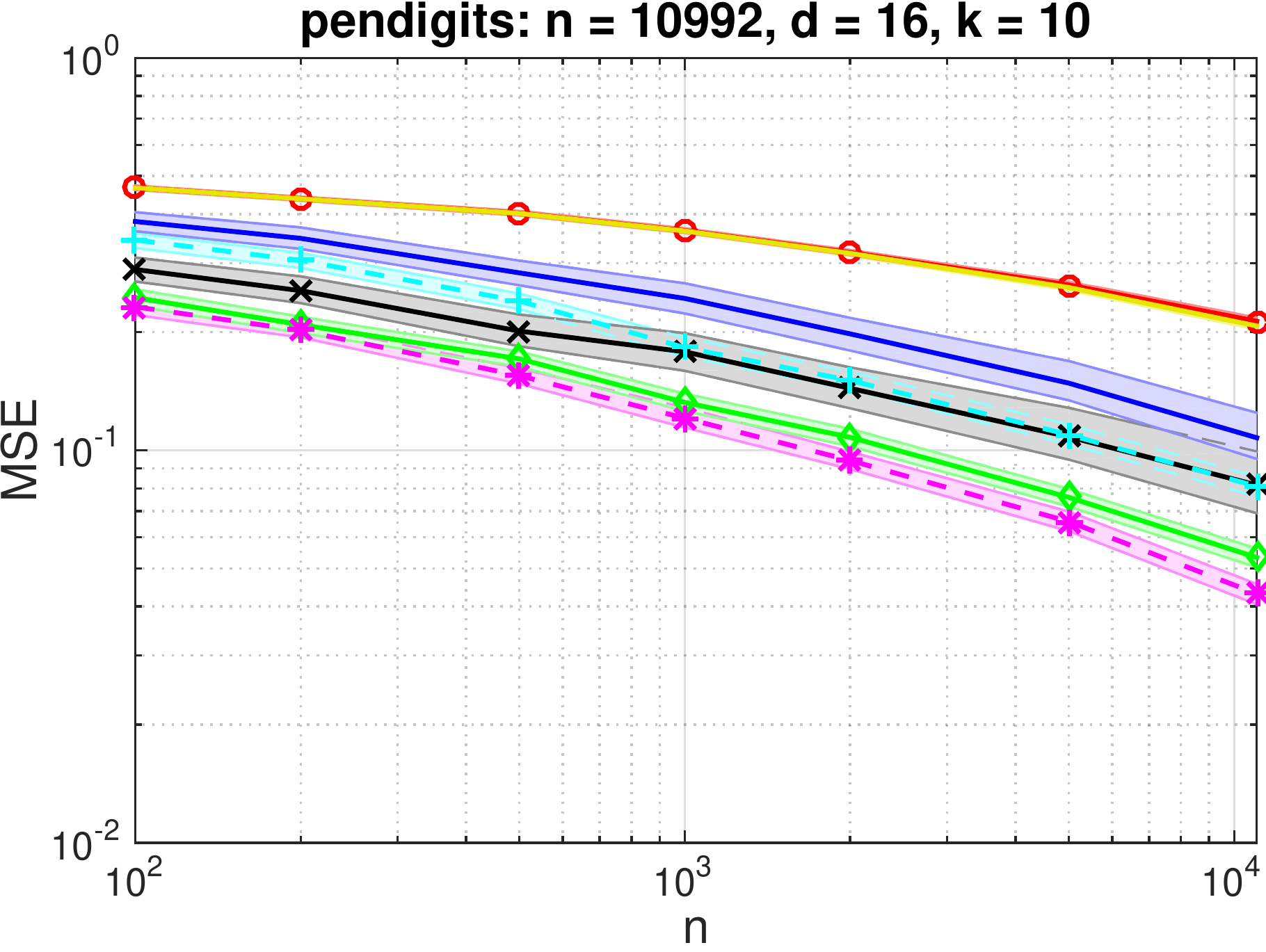}
	  \figsqueeze\caption{pendigits / deterministic reward}\label{fig:raw-pendigits}		
	\end{subfigure}
	\hspace{0.25in}
		\begin{subfigure}[t]{0.39\textwidth}
			\centering
			\includegraphics[width=\textwidth]{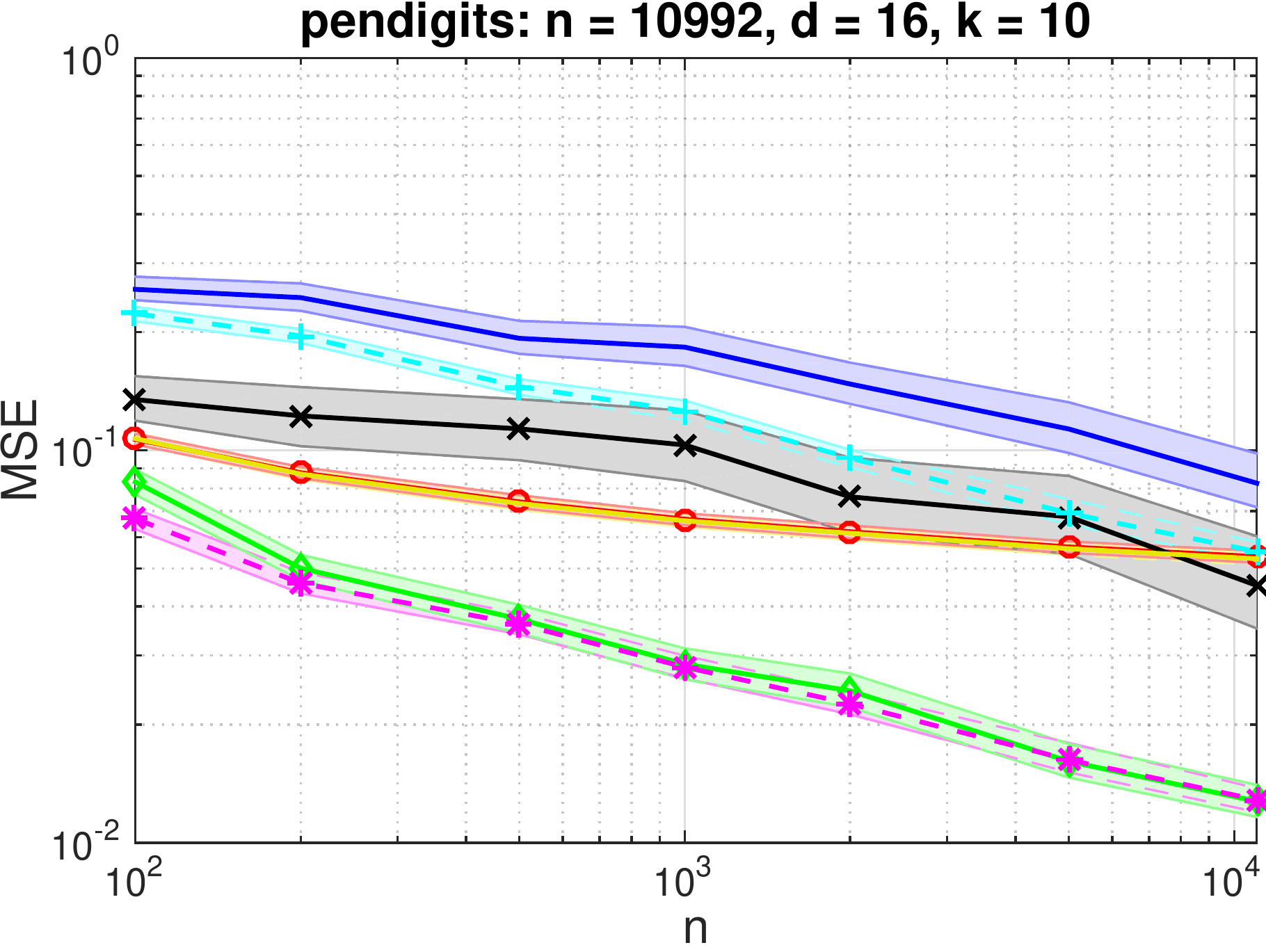}
			\figsqueeze\caption{pendigits / noisy reward}\label{fig:noisy-pendigits}		
		\end{subfigure}\\
		
		\medskip
        \begin{subfigure}[t]{0.39\textwidth}
	  \centering
	  \includegraphics[width=\textwidth]{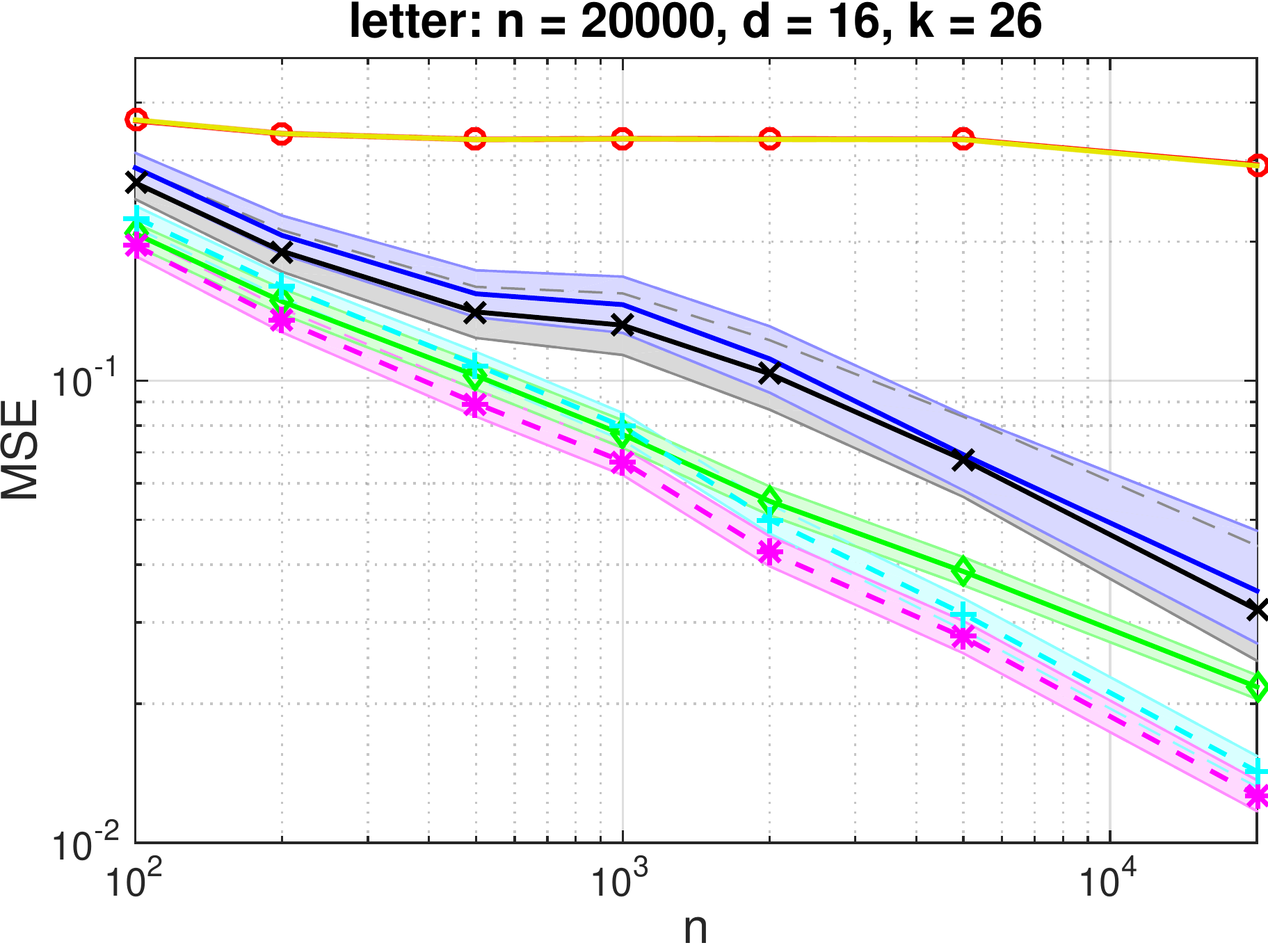}
	  \figsqueeze\caption{letter / deterministic reward}\label{fig:raw-letter}		
	\end{subfigure}
\hspace{0.25in}
	\begin{subfigure}[t]{0.39\textwidth}
	  \centering
	  \includegraphics[width=\textwidth]{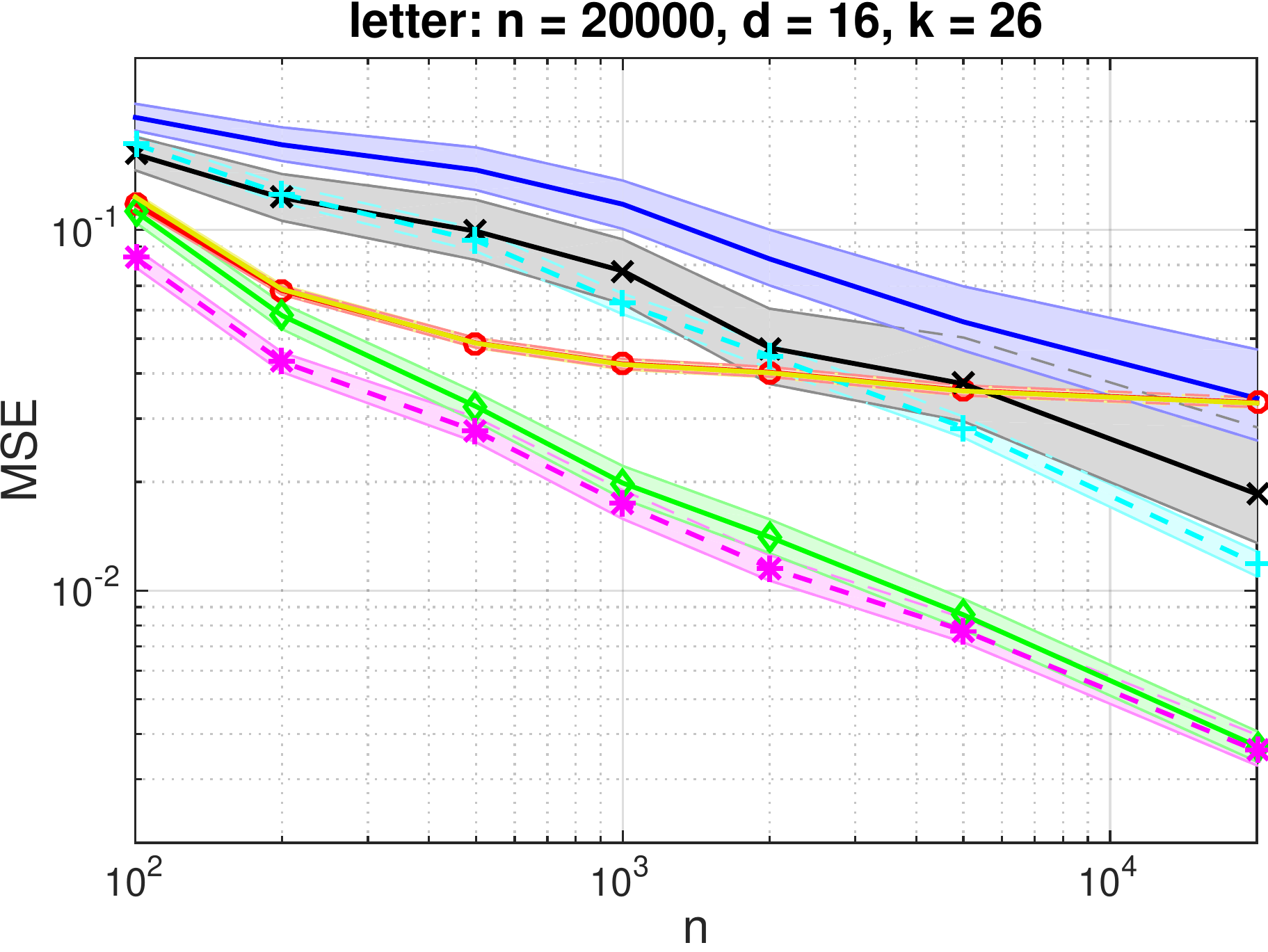}
	  \figsqueeze\caption{letter / noisy reward}\label{fig:noisy-letter}		
	\end{subfigure}\\
\medskip
		\begin{subfigure}[t]{0.39\textwidth}
			\centering
			\includegraphics[width=\textwidth]{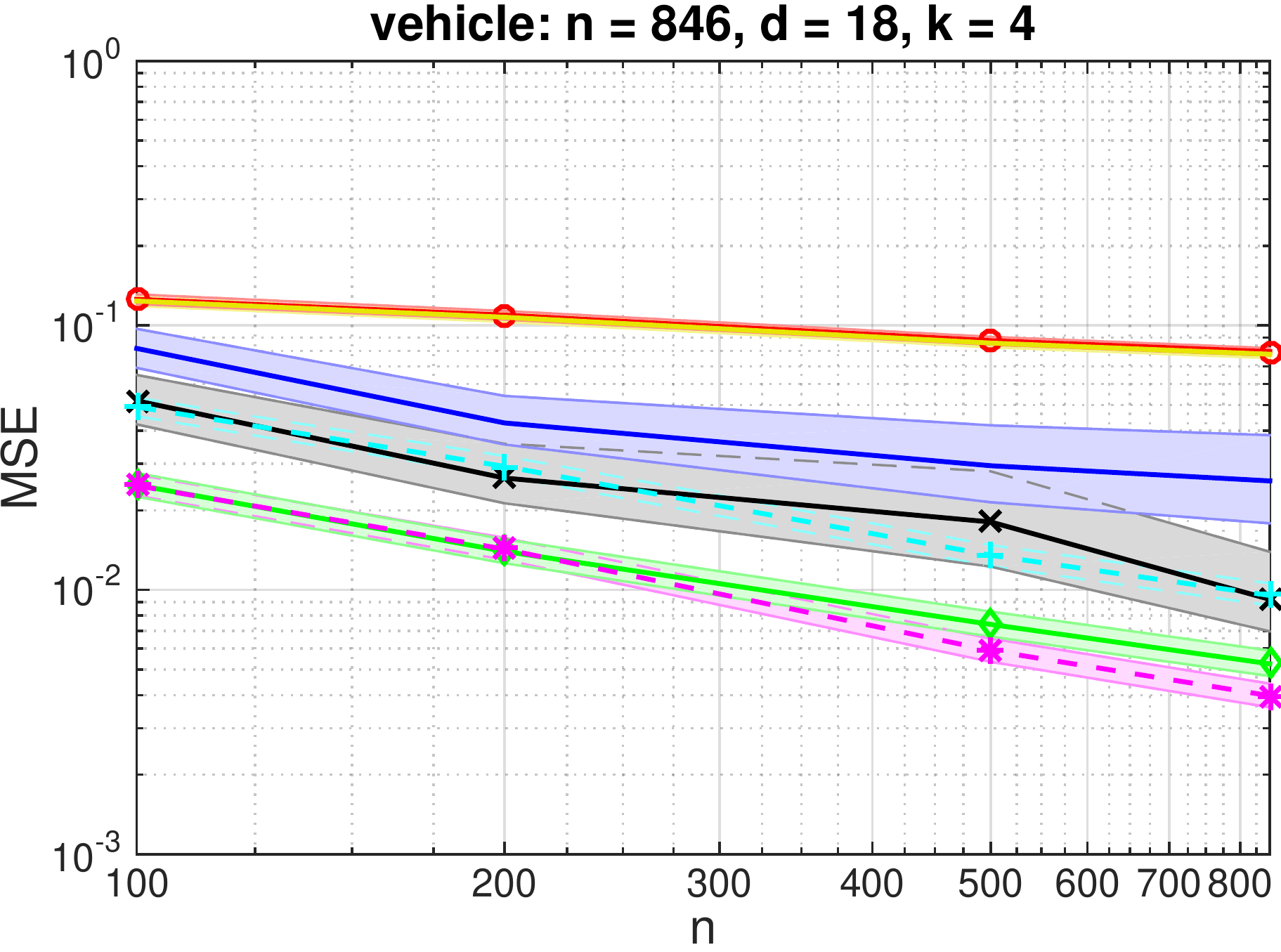}
			\figsqueeze\caption{vehicle / deterministic reward}\label{fig:raw-vehicle}		
		\end{subfigure}
		\hspace{0.25in}
				\begin{subfigure}[t]{0.39\textwidth}
					\centering
					\includegraphics[width=\textwidth]{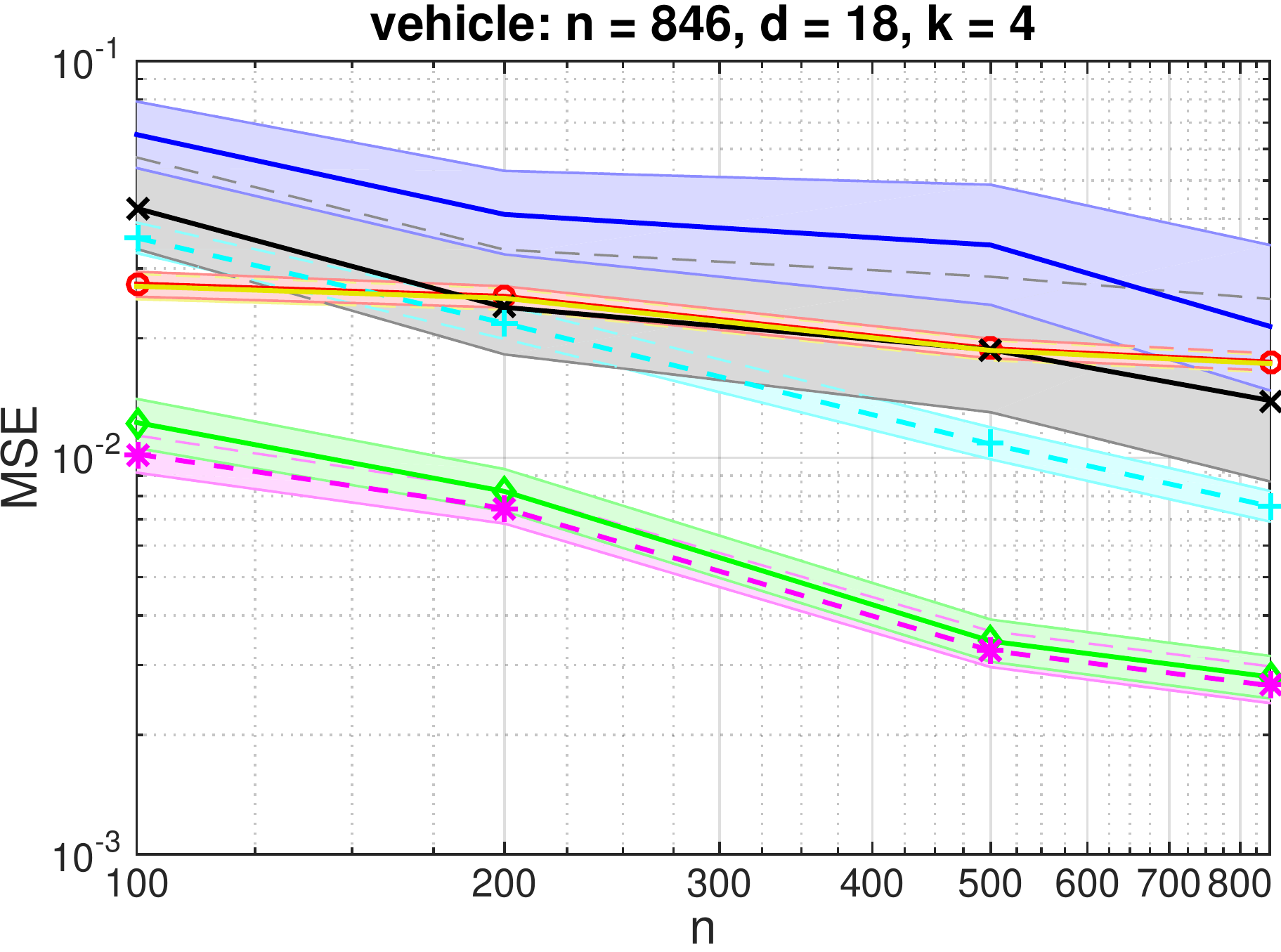}
					\figsqueeze\caption{vehicle / noisy reward}\label{fig:noisy-vehicle}		
				\end{subfigure}\\
				\medskip
		\begin{subfigure}[t]{0.39\textwidth}
			\centering
			\includegraphics[width=\textwidth]{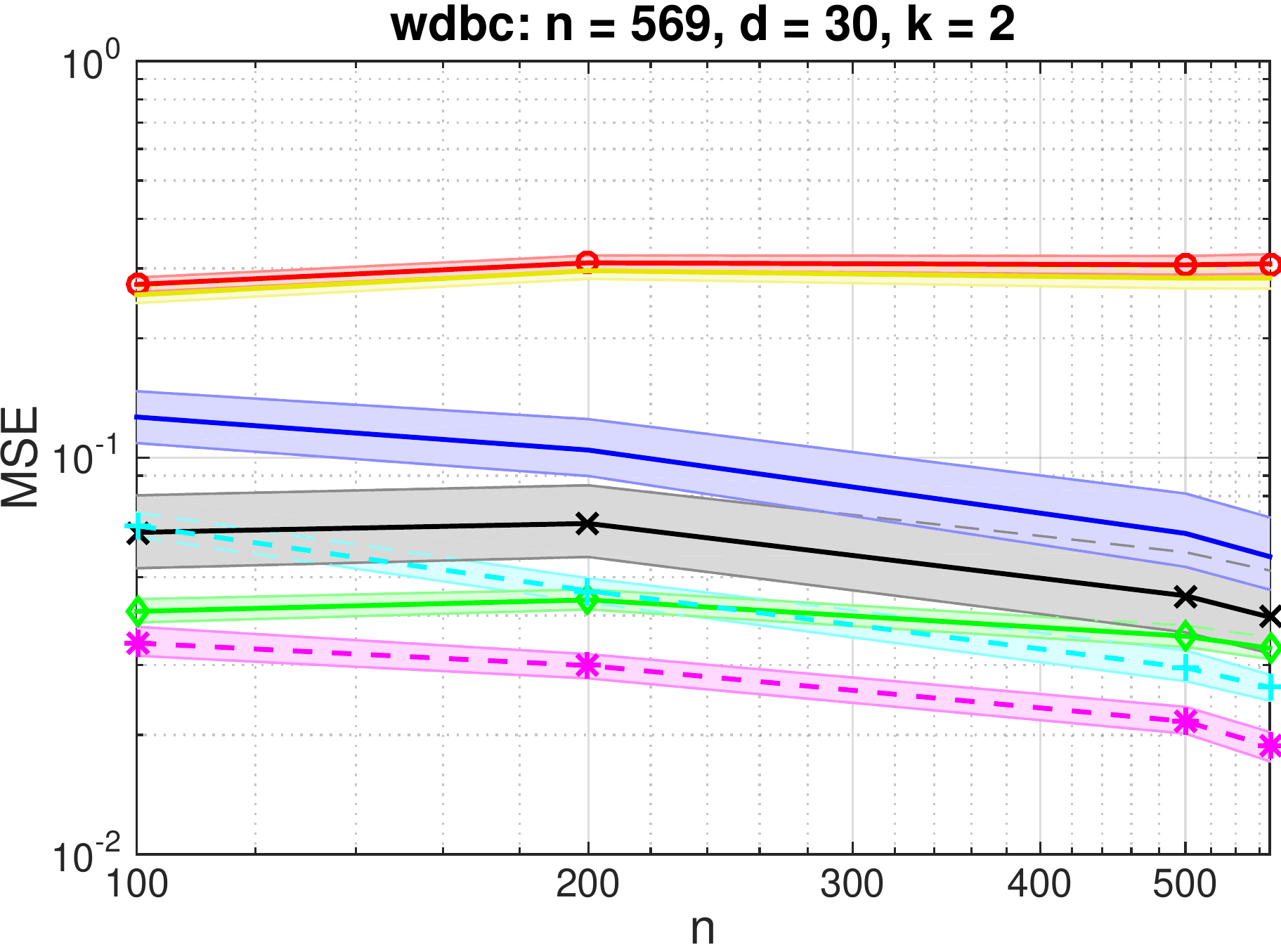}
			\figsqueeze\caption{wdbc / deterministic reward}\label{fig:raw-wdbc}		
		\end{subfigure}
\hspace{0.25in}
		\begin{subfigure}[t]{0.39\textwidth}
			\centering
			\includegraphics[width=\textwidth]{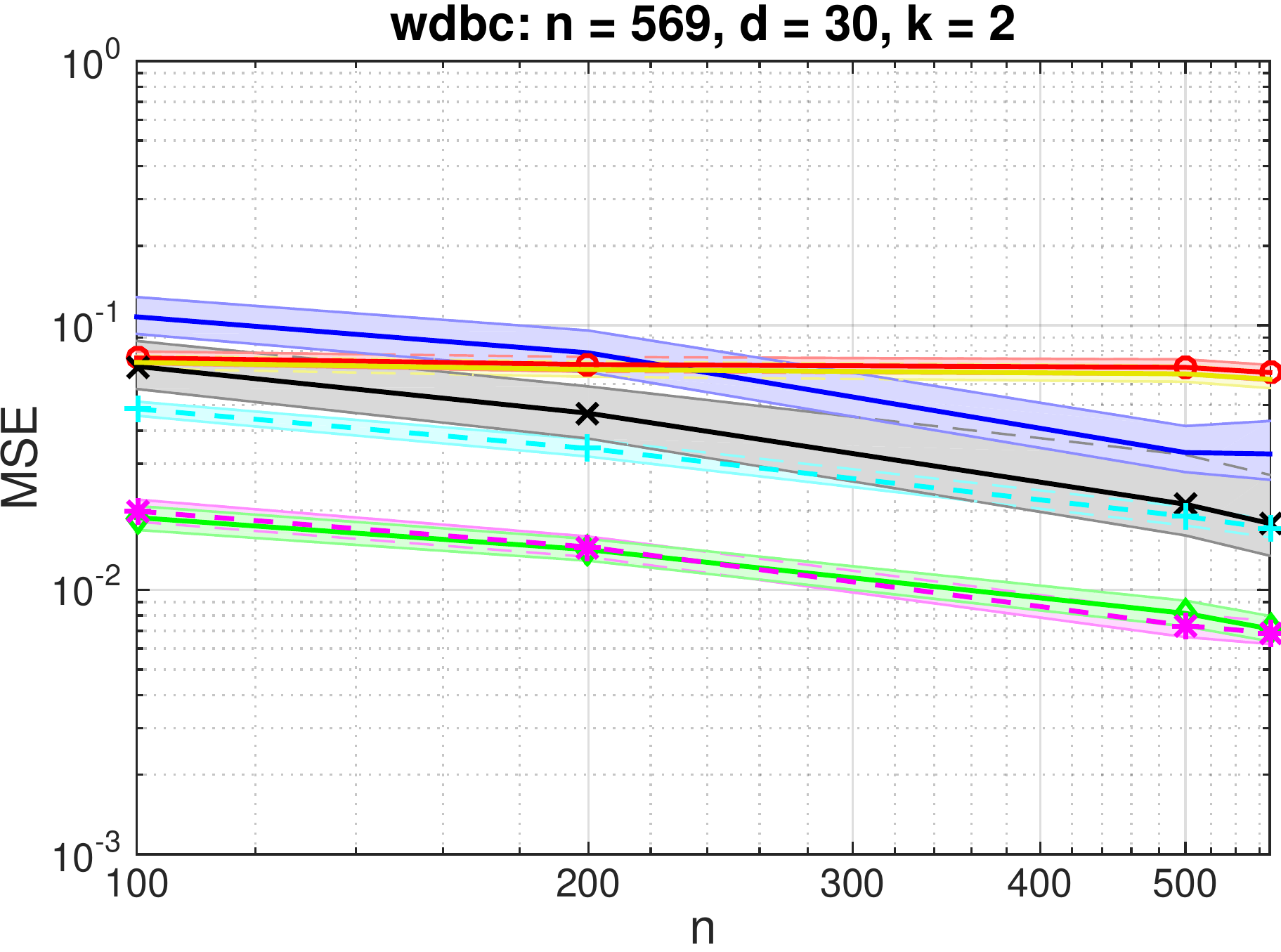}
			\figsqueeze\caption{wdbc / noisy reward}\label{fig:noisy-wdbc}		
		\end{subfigure}
	
\end{figure*}

\end{document}

{\color{OliveGreen}
\section{Formulating the minimax problem}\label{sec:formulation}

Off-policy evaluation is intrinsically a statistical estimation
problem, where the goal is to estimate $v^\pi$.  We study this problem
in a standard minimax framework: given $n$ i.i.d.\ samples according to a
policy $\mu$, what is the smallest mean square error (MSE) \emph{any}
estimator can achieve for evaluating a fixed policy $\pi$, in the
worst case over a particular class of data-generating
distributions. Formally, we seek to study
$$ \inf_{\hat{v}\in\text{all estimators}} \sup_{\text{A class of
    contextual bandit problems}} \E[(\hat{v} - v^\pi)^2].
$$
If an estimator matches the minimax risk (or matches up to constant), then we say it is minimax (or achieves minimax rate).
The tricky question that comes up in all minimax analysis is that what class of problems do we take the maximum over. If this class of problems is too large, then the minimax bound can be too conservative, but if it is too small, then the bound is not widely applicable.

To make matters worse, there are a lot of ways we can define a class of problems. Each contextual bandit problem is parameterized by $\pi,\mu, \cA, \cX, \cD_x$ and $\cD_{r|x,a}$. The question is which subset of these are we taking the maximum over in the minimax problem? \citet{li2015toward} defined a problem for each policy pair $(\pi,\mu)$ separately, for the case when there is one context, and in addition, assume the conditional (on an action) expectation and variance of the reward is bounded for any action.

A natural generalization for us is to consider one minimax problem for every $\cD_x$, $\mu$, $\pi$ separately so that we get problem-specific upper bounds and lower bounds.
To formulate our class of reward-generating
functions, assume we are given maps
$R_{\max}~:~\cX\times \cA \to \mathbb{R}_+$ and
$\sigma~:~\cX\times \cA \to \mathbb{R}_+$. The
class of conditional distributions $\cR(\sigma,R_{\max})$ is defined as
\begin{align*}
\cR(\sigma,R_{\max})
\coloneqq
\Bigl\{
\cD(r|x,a):\:
&
0\leq\E_D[r|x,a]\leq R_{\max}(x,a)
\text{ and }\\
&
\Var_D[r|x,a]\leq \sigma^2(x,a)
\text{ for all }x,a
\Bigr\}
.
\end{align*}


Formally, let an estimator be any function $\hat{v}:  (\cX,\cA,\R)^n \rightarrow \R$ that takes in the $n$ data points collected by policy $\mu$, the minimax risk is defined as
\begin{equation}
\label{eq:minimax_def}
R_n(\cD_x,\pi,\mu,\sigma,R_{\max})
\coloneqq
\adjustlimits
\inf_{\hat{v}}
\sup_{\cD(r|x,a)
	\in \cR(\sigma,R_{\max})}
\E\left[(\hat{v}-v^\pi)^2\right].
\end{equation}
A pair of matching upper and lower bounds of the minimax risk, which we will present in Section~\ref{sec:minimax}, should capture the difficulty of the model-free off-policy evaluation problem quite well.

We will also consider settings when there is auxiliary input to the estimator $\hat{v}$. In particular, it is often the case that there exists a model-based direct estimator of the expected reward (a value function) that's given to us (e.g., by training on a hold-out data set), or when we assume knowledge that the value function is within a class of functions (e.g., generalized linear, additive, smooth). In either cases, it is possible that we can do better than \eqref{eq:minimax_def}. We will revisit the point Section~\ref{sec:use_dm}, and then propose a class of new estimators.

\section{The limit of model-free off-policy evaluation}\label{sec:minimax}

 In this section, we investigate the information-theoretic limit of model-free off-policy evaluation as specified in \eqref{eq:minimax_def} and figure out the optimal estimators that nearly attains the minimax risk. The natural candidate is of course the IPS estimator \citep{horvitz1952generalization}:
\begin{equation}\label{eq:IPS}
\hat{v}^\pi_{\text{IPS}} =\sum_{i=1}^n \frac{\pi(a_i,x_i)}{\mu(a_i,x_i)} r_i.
\end{equation}
However, it is unclear in general whether IPS is the best we can do. In particular, could there be a more sophisticated estimator that favorably leverages the bias-variance tradeoff and achieves a better MSE?

\subsection{Multi-arm bandits}
We first look at the multi-arm bandits problem, where the both lower bound and upper bound are known and they matches asymptotically \citep{li2015toward}. The setting is as follows.

Let $k$ be the number of arms. Let $\mu$ be a policy distribution defined on $[k]$ and $\cD$ is some unknown distribution defined on $\R^k$. In each iteration we sample a reward vector $\bm{r}\sim \cD$ and an action $a \sim \mu$, then $\bm{r}_a$ is revealed. After $n$ rounds, we collect a data set
 $\{(a_1,r_1),...,(a_n,r_n)\}$. This can be thought of as a contextual bandits problem where $\cX$ is a singleton.

\citet{li2015toward} defined the minimax problem for multi-arm bandits and show that for sufficiently large $n$, the minimax rate
$$
R_n(\emptyset,\pi,\mu,\sigma^2,R_{\max}) =  \Omega(\frac{1}{n} \E_\mu \rho^2 \sigma^2).
$$
At the limit as $n\rightarrow \infty$, the lower bound is matched by the regression estimator
\begin{equation}\label{eq:reg}
\hat{v}^\pi_{\text{Reg}} = \sum_{a=1}^k \frac{\sum_{i=1}^n r_i \1_{\{a_i=a\}}}{\max\{n(a),1\}} \pi(a)
\end{equation}
which essentially estimates the empirical mean of the reward for each arm(and use $0$ when an arm is not seen at all in the data).
Interestingly, the importance sampling estimator \eqref{eq:IPS} is provably suboptimal, with a possibly very large asymptotic gap. Specifically, the variance of importance sampling estimator can be decomposed into two parts:
$$
\sup_{r|a}\Var{(\hat{v}_{\text{IPS}})} = \sup_{r|a} \left[\frac{1}{n}\E_\mu \Var(\rho r|a)  +  \frac{1}{n} \Var_\mu \E(\rho r|a)\right] = \frac{1}{n}\E_\mu \Var(\rho r|a)  +  \frac{1}{4n}\left[\E_\mu\rho^2\E(r|a)^2  - (v^\pi)^2\right].
$$
When we take supremum over $r|a$, the first term becomes $\frac{1}{n}\E_\mu \rho^2 \sigma^2$ which exactly matches the lower bound, but the second term is proportional to $\frac{1}{n}\E_\mu \rho^2 R_{\max}^2$ for almost every pair of $(\mu,\pi)$\footnote{One exception is that if $\pi$ always chooses $a$ and $\mu(a)>0$, then the second term is $0$ for any reward function. In this trivial case, IPS is in fact optimal.}, making the estimator arbitrarily suboptimal as $R_{\max}$ becomes large.

\subsection{Contextual bandits}

Contextual bandits is a more delicate problem as the number of state is infinite (or very large) so a crude reduction to a multi-arm bandits problem with number of arms equal to $|\cX||\cA|$ \citep{li2015toward} might not capture the hardness of the problem. In particular, the optimal regression estimator for multi-arm bandits now seems to be a really bad choice as we never see each context-action pair $(x,a)$ more than once. More formally, when $\cX$ is a continuous domain and $\cD_x$ is absolute continuous,
\eqref{eq:reg} now becomes
$$
    \hat{v}_{\text{Reg}} = \int_{x\in \cX} \sum_{a=1}^k\frac{\sum_{i=1}^n r_i \1_{\{a_i=a, x_i=x\}}}{\max\{n(x,a),1\} } \pi(a|x)d\cD_x
$$
and $\hat{v}_{\text{Reg}} \equiv 0$ for any $n<\infty$. The epic failure of the naive regression estimator is not surprising, but does this mean that we need to search for a better estimator or is the IPS estimator already the best we can do?

In this section, we establish stronger information-theoretic lower bounds and show that the importance sampling estimator is in fact optimal (up to a constant) for contextual bandits. We first refine \citet{li2015toward}'s lower bound idea to get a lower bound that does not require the importance weight to be uniformly bounded.
\begin{theorem}[Lower bound 1]\label{thm:lowerbound2}
Let $r|a,x$ be any random variables satisfying $\Var(r|a,x) \leq \sigma^2(a,x)$, and $0 \leq \E(r|a,x)\leq R_{\max}(a,x)$ for all $a,x$. Then for any $\pi,\mu,\sigma^2$ satisfying $\E_\mu \rho^2\sigma^2 <\infty$, we have:
  $$
  R_n(\cD_x,\pi,\mu,\sigma,R_{\max})
  \geq  \frac{\log 2 \E_\mu\rho^2\sigma^2}{8n} \left[1- \frac{\E_\pi \rho\sigma^2 \1_{\{\rho\sigma^2>R_{\max} \sqrt{n\E_\mu\rho^2\sigma^2/(2\log 2)}\}}}{\E_\pi\rho\sigma^2}\right]^2.
$$
\end{theorem}
The bound captures the intrinsic difficulty in the variance of reward.
The next result shows that the above bound is not tight and the dependence on $R_{\max}^2$ and $\E_\mu \rho^2$ is required for any estimators even if $\sigma^2\equiv 0$.

\begin{theorem}[Lower bound 2]\label{thm:lowerbound}
Let
 $\cD_x$ obeys that $\max_{x\in\cX}\P_{\cD_x}(x) \leq C/|\cX|$ for some constant $C$. Let $r|x,a$ be any deterministic map with range $[0,R_{\max}(x,a)]$ and denote $\rho(x,a):=\pi(x,a)/\mu(x,a)$.
Then for each policy $\pi$, $\mu$, $R_{\max}$ such that $\E_\mu \rho^2 R_{\max}^2 <\infty$, we have:
  \begin{align*}
  &R_n(\cD_x,\pi,\mu,0,R_{\max})\\
  \geq&  \frac{\log 2 \E_\mu\rho^2R_{\max}^2}{8n} \left[1-\frac{8}{\log 2}\frac{Cn\log|\cX|}{|\cX|} - \frac{\E_{\pi} \left(\rho R_{\max} 1_{\left\{\rho R_{\max}> \sqrt{n \E_\mu (\rho^2R_{\max}^2)/(16\log 2)}\right\}}\right)}{\E_\pi\rho R_{\max}}\right]^2.
  \end{align*}
\end{theorem}

Note that in both lower bounds, we have a multiplicative correction term that is smaller than $1$ but goes to $1$ as $n\rightarrow \infty$ (assuming $|\cX|=\infty$ and for the moment). These correction terms are unavoidable because, when $n$ is very small, and $\rho$ is large $\E_\mu \rho^2 (\sigma^2 + R_{\max}^2)/n$ can be significantly larger than the risk of a trivial constant estimator that always return $0$.

We claim that the correction terms is larger than a constant as soon as $n$ is larger than a constant that depends only on the tail probability decay of $\rho R_{\max}$ and $\rho \sigma^2$. The following remark formalizes the argument.
\begin{remark}\label{rmk:correction_conv}
Let the correction term in Theorem~\ref{thm:lowerbound2} and Theorem~\ref{thm:lowerbound} be $C_1(n)$ and $C_2(n)$ respectively.
Assume $R_{\max}$ is bounded away from $0$ for all $x$ and $a$,
$\rho R_{\max}$ and $\rho \sigma^2$ both have a polynomial tail probability $\P(\cdot > b) \leq O(1/b^{\alpha})$ with any parameter $\alpha>1$\footnote{Note that this is almost necessary, because when $\alpha  \leq 1$ the expectation does not exist.}, then
$$
\max\{C_1(n),C_2(n)\} = 1 - O(n^{-\alpha/2}).
$$
If we further assume $\rho R_{\max}$ and $\rho \sigma^2$ has an exponential tail $\P(\cdot > b)\leq O(b\exp(-b))$, then
$$
\max\{C_1(n),C_2(n)\}=  1-O(\sqrt{n}e^{-\sqrt{n}}).
$$
\end{remark}
\begin{remark}
Theorem~\ref{thm:lowerbound} contains an additional term that depends on $|\cX|$. In contextual bandit, $|\cX|$ is often very large, and our result suggests that as long as $|\cX|/\log|\cX| = \Omega(Cn)$, a lower bound with the same scaling will apply. The assumption on $\cD_x$ is to some extent unavoidable because if $\cD_x$ is very skewed towards a small number of contexts $x$, then the problem essentially becomes a multi-arm bandit and we know that the minimax risk at $n\rightarrow \infty$ is independent to $R_{\max}$ \citep{li2015toward}.
\end{remark}

Now we show that the IPS estimator is nearly optimal.
\begin{lemma}[IPS estimator]\label{lem:IS_upperbound}
Let $\sigma$, $r$ and $\rho$ be functions of $(x,a)$, representing the reward variance, reward expectation and importance weights associated with $(x,a)$. Then
$$
\E(\hat{v}_{\mathrm{IPS}}-v^\pi)^2 \leq  \frac{1}{n}\left(\E_{\mu} \rho^2 \sigma^2 + \frac{1}{4}\E_\mu \rho^2 R_{\max}^2\right).
$$
\end{lemma}
\begin{proof}
The variance (therefore MSE)  of the importance weighting estimator can be decomposed into two parts
\begin{align*}
\Var \left(\hat{v}_{IPS}\right) &= \E_{\mu}(\Var(\hat{v}_{IPS} | x,a) + \Var_\mu (\E(\hat{v}_{IPS}|x,a))
 = \frac{1}{n}\E_{\mu} \rho^2 \sigma^2  +  \frac{1}{n}\Var_\mu( \rho r|x,a).
\end{align*}
Now the second term,
\begin{align*}
\Var_\mu( \rho \E(r|x,a)) &=  \E \left[ (\rho \E(r|x,a) - \E_\mu \rho \E(r|x,a))^2  \right] = \E \E \left[ (\rho \E(r|x,a) - \E_\mu \rho \E(r|x,a))^2 \middle| \rho R_{\max} = \lambda \right] \\
&\leq \E\Big(\frac{\lambda^2}{4} \Big|  \rho R_{\max} = \lambda \Big) = \frac{1}{4}\E_\mu \rho^2 R_{\max}^2.
\end{align*}
\end{proof}
Lemma~\ref{lem:IS_upperbound} together with the mean of Theorem~\ref{thm:lowerbound2}~and~\ref{thm:lowerbound} characterizes the minimax rate up to a constant.
\begin{corollary}
For sufficiently large $|\cX|$ and $n$:
  $$\inf_{\hat{v}}\sup_{r|a,x \in \cR(\sigma^2,R_{\max})} \E_\mu(\hat{v} - v^{\pi})^2 = \Theta\left[\frac{1}{n} \left(\E_{\mu} \rho^2 \sigma^2 + \E_\mu \rho^2R_{\max}^2\right) \right]. $$
\end{corollary}

%
%
}

\section{Results that ends up not too useful.}

\begin{lemma}[Adapted from {\citet[Theorem~5]{li2015toward}}]
Let $\hat{v}$ be any estimator that takes in $\mu,\pi$ and the data $\{(x_i,a_i,r_i)\}_{i=1,...,n}$ collected by running policy $\mu$.
$$
R(\mu,\pi,\cD_{x}):=\inf_{\hat{v}}\sup_{\substack{r :\sup_{x,a}|\E r(x,a)|\leq 1,\\
 \sup_{x,a}\Var(r(x,a))\leq \sigma^2}}  \E(\hat{v} - v^{\pi})^2 = \Omega\left[\frac{\sigma^2}{n} \E_{\mu} \frac{\pi(a|x)^2}{\mu(a|x)^2}\right].
$$
\end{lemma}

This lower bound does not capture the variation due to the distribution of $x$ and the choice of $a$ by $\mu$, therefore is expected to be loose especially when $\sigma^2$ is small. This does not matter in the case when $|\cX|$ is finite, since one can (as $n$ gets large) hope to see all actions at every $x\in \cX$.

A more naive lower bound that captures that part can be obtained from the special case when $\mu = \pi$ and take $\E r(x,a)$ such that if we marginalize over $a\sim \mu(a|x)$, $\E r(x,a)\sim \text{Ber}(R_{\max},p)$ (e.g., for every new $x$ we can toss a coin and decide whether we receive $0$ or $R_{\max}$  expected reward, independent of $a$). Apparently $v^{\pi}=v^{\mu} = pR_{\max}$ now, but even if $\sigma^2 = 0$, the minimax risk of estimating Bernoulli mean:
\begin{lemma}
Let $\hat{p}$ be any function from $\R^n\rightarrow \R$ that takes in i.i.d.\ sample of size $n$ from $\mathrm{Bernoulli}(p$).
$$
\min_{\hat{p}}\max_{p\in [0,1]} \E(\hat{p}-p)^2 = \frac{n}{4(n+\sqrt{n})^2}.
$$
\end{lemma}
This bound applies to all $D_x$ that has a density (so that we do not see any $x$ more than once).

We can combine the two bounds into the following lemma
\begin{lemma}\label{lem:cont_lowerbound1}
Assume $\cD_{x}$ is a density, then we have
$$
R(\mu,\pi,\cD_{x}) = \Omega(\frac{\sigma^2}{n} \E_{\mu} \frac{\pi(a|x)^2}{\mu(a|x)^2}  + \frac{R_{\max}^2n}{4(n+\sqrt{n})^2}).
$$
\end{lemma}
This bound is also unsatisfactory because the importance scores should somehow enter the second term as well.
Recall that the variance of the importance sampling estimator is:
\begin{align*}
\frac{1}{n}\E\Var(\rho r|x,a)  + \frac{1}{n}\Var\E(\rho r|x,a) = \frac{\sigma^2}{n} \E_{\mu} \frac{\pi(a|x)^2}{\mu(a|x)^2}   +  \frac{1}{n}\Var(\rho \E(r|x,a))
\end{align*}
For problems where $\rho$ and $\E(r|x,a)$ are independent, we can write the second term into:
$$
\Var \rho \Var(\E(r|x,a))   +  \Var\rho  \E(r|x,a)^2  + \Var(\E(r|x,a)) \E(\rho)^2.
$$
It will be nice to have a lower bound that depends on some of these terms. $\E\rho = 1$ by definition of $\rho$.


{\color{blue}-----------------DISCUSSION BELOW ----------------------------------\\
I had been thinking along the line of reduction to Bernoulli minimax problem. If we can estimate $v^\pi$ with MSE $\epsilon_n$ using estimator $\hat{v}$,  then the idea is to construct a special case where we can use $\hat{v}$  to form an estimator of $v^\mu$, whose performance is limited by the minimax risk of estimating Bernoulli mean.

Example:
Suppose we fix $v^\pi  = p$.  At every context $x_i$,  choose actions uniformly at random until their cumulative probability in $\pi$ reach $p$. For the last action $a$, toss a coin with probability proportional to the amount of overshoot over $\mu(a)$ and if it is a tail include it in $A_i$, otherwise drop it from $A_i$. Then the assignment is assign $R_{\max}$ if $a_i \in A_i$ and  $0$ otherwise. The expectation of the choices can be thought of as a ``soft'' subset.

Note that conditioned on the context, once $p$ is known we can exactly calculate $\mu(A_i)$ or run simulation to get it using $p$.

Now suppose we have an estimator of $p$, namely $\hat{p}$.  Then we can use it to replace $p$ in such calculations and come up with an estimator of $\mu(A_i)$ for every $i$.  Note that  $\E \frac{1}{n} \sum_{i=1}^n  \mu(A_i)   = v^\mu$.  So when we replace $\mu(A_i)$ with its estimator using $\hat{p}$ and that gives us an estimator of $v^\mu$.

Note that even if $p$ is exact, $\mu(A_i)$ is still random due to the context distribution and how we choose $A_i$, but since it is not always giving us points at the boundary, it should have a much smaller variance than the $\bar{r}=1/n \sum_{i=1}^n r_i$. More concretely, let $\sigma^2=0$
$$
\Var{\bar{r}} = \frac{1}{n}  \left[\E_{x\sim \cD_x}\Var_{a\sim \mu(x)}(r|x,a)   +  \Var_{x\sim \cD_x}\E_{a \sim \mu(x)}(r|x,a) \right]  = \frac{1}{n}  \left[\E_{\cD_x}\Var_{a\sim \mu(x)}(r|x,a)  +  \Var_{\cD_x} \mu(A_i) \right].
$$
while the variance
$$\Var{\frac{1}{n} \sum_{i=1}^n  \mu(A_i)}  = \frac{1}{n}\Var_{\cD_x} \mu(A_i). $$

The risk lower bound of this estimator will translate to a risk lower bound for estimating $p  = v^\pi$.
It is however unclear at the moment, how to construct such set $A_i$ so that we get a meaningful bound (we need to be able to analytically work out the mapping from $p$  to $\mu(A_i)$.

-----------------END OF DISCUSSION---------------------------}\\

\begin{lemma}\label{lem:deterministic}
For deterministic policy $\pi$, we have that
$$R(\mu,\pi,\cD_x) = \Omega\left(\frac{R_{\max}^2}{n\E \mu(\pi(x_i))}\right).$$
\end{lemma}
\begin{proof}
Consider only deterministic policy $\pi$. Denote the number of observed actions such that $a_i = \pi(x_i)$ by $n^*$. First we reduce the problem to a Bernoulli estimation problem by considering the special $r|a,x$ such that $r = R_{\max}$ with probability $p$ whenever $a = \pi(x)$ and $0$ otherwise. Note that for simplicity, we assume this is a deterministic reward. Clearly, the value of policy $\pi$ is equal to $p$. Since we observe each $x$ only once, this problem reduces to estimating the Bernoulli mean using the subset of the given data where $a_i = \pi(x_i)$. Formally, we have
%

\begin{align*}
  &\min_{\hat{v}}\max_{r|a,x} \E(\hat{v}(\{x_i,a_i,r_i\}_{i=1}^n) - v^\pi)^2 \geq R_{\max}^2 \min_{\hat{p}}\max_{p\in [0,1]} \E(\hat{p}(\{x_i,a_i,r_i\}_{i=1}^n) - p)^2  \\
   =& R_{\max}^2 \min_{\hat{p}}\max_{p\in [0,1]} \E(\hat{p}(\{x_i,r_i\}_{i: a_i=\pi(x_i)}) - p)^2 \\
   =& R_{\max}^2\min_{\hat{p}}\max_{p\in [0,1]} \E\E \left[(\hat{p}(\{x_i,r_i\}_{i: a_i=\pi(x_i)}) - p)^2 \middle| |\{i: a_i=\pi(x_i)\}|\leq n\E\mu(\pi(x_i))\right]\\
   \geq& R_{\max}^2\min_{\hat{p}}\max_{p\in [0,1]}  0.5 \E \left[(\hat{p}(\{x_i,r_i\}_{i: a_i=\pi(x_i)}) - p)^2 \middle| |\{i: a_i=\pi(x_i)\}|\leq n\E\mu(\pi(x_i))\right]\\
   \geq& R_{\max}^2\min_{\hat{p}}\max_{p\in [0,1]} 0.5 \E \left[(\hat{p}(\{x_i,r_i\}_{i: a_i=\pi(x_i)}) - p)^2 \middle| |\{i: a_i=\pi(x_i)\}|= n\E\mu(\pi(x_i))\right]\\
   =& \Omega\left(\frac{R_{\max}^2}{n\E \mu(\pi(x_i))}\right).
\end{align*}
In Line 2, we make use of that the sufficient statistic for estimating $p$ is those that $a_i$ agrees with $\pi(x_i)$.
Note that in Line 4, we used the fact that the median of the binomial distribution very close to its mean.
\end{proof}

\begin{remark}[Importance sampling estimator is nearly optimal]
  When we have deterministic policy $\pi$ to evaluate, deterministic reward given $(x,a)$ specified by $r(x,\pi(x)) \sim \text{Ber}(p,R_{\max})$, the maximum variance of the importance sampling estimator is
  $$
  \max_{p\in[0,1]} \frac{R_{\max}^2}{n} \left(p \E \frac{1}{\mu(\pi(x))} - p^2\right) = \begin{cases}
  \frac{R_{\max}^2}{n}\E \frac{1}{\mu(\pi(x))} & \text{ if }\E \frac{1}{\mu(\pi(x))} \geq 2\\
\frac{R_{\max}^2}{n} \frac{(\E \frac{1}{\mu(\pi(x))})^2}{4} & \text{otherwise.}
  \end{cases}
  $$
  In the region where $\pi(x)$ and $\mu(x)$ are close, the estimator converges quadratically to the lower bound $\frac{R_{\max}^2}{4n}$, while when the discrepancy is large, it depends linearly in $\E \frac{1}{\mu(\pi(x))}$. This almost matches the lower bound of minimax rate, except that the expectation is taken at a wrong place. When $\mu$ is a uniform exploration policy the expectation can be dropped and the upper bound and lower bound matches.
\end{remark}

Now we will extend the lemma above to any general $\pi$.
\begin{lemma}
Let $K$ be the total number of actions. Any policy $\pi$ obeys
$$\pi = \sum_{i=1}^{\infty} w_i\pi_{i}$$
where $\pi_{i}$ are deterministic policies, $\sum_{i=1}^\infty w_i=1$, for any $i$, $w_i\geq 0$ and
$$ \frac{\prod_{j=1}^{i-1}{1-w_j}}{K} \leq  w_i\leq (1-1/K)^{i-1}$$.
\end{lemma}
\begin{proof}
  It suffices to construct such a decomposition. For each $x$ take the $\pi_1(a|x) = \argmax_a \pi(a|x)$,  assign $w_1=\min_{x}\max_{a}\pi(a|x)$.
  By pigeon-hole principle, $w_1\geq \frac{1}{K}$. So the remaining probability mass is $1-w_1$ which is at most $1-1/K$. Recursively use the same construction, we get the desired property.
\end{proof}

\begin{lemma}
Any policy can be decomposed into a sum of $K$ weighted deterministic policies that are mutually ``orthogonal''.
\end{lemma}
\begin{proof}
  Consider the following greedy construction:
  For each $x$, recursively choose $\argmax_a \pi(a|x)$ and assign weight $\pi(a|x)$.
  A more naive reconstruction is to simply decompose $\pi$ over $a$.
\end{proof}

\begin{lemma}\label{lem:weighted_deter}
  Let $\pi$ be a weighted deterministic policy with weight $w(x)$. The value function $v^{\pi} = \E w(x) r(\pi(x),x)$. Then
  $$
  \min_{\hat{v}}\max_{r|a,x} \E(\hat{v} - v^{\pi})^2 = \Omega\left( \frac{(\E w(x))^2 R_{\max}^2}{n\E\mu(\pi(x))}\right)
  $$
\end{lemma}
\begin{proof}
Let $\hat{v}'$ be the class of privileged estimator that knows $\E w(x)$ besides the given data available to $\hat{v}$
\begin{align*}
&\min_{\hat{v}}\max \E(\hat{v} - v^\pi)^2 \geq \min_{\hat{v}}\max_{p\in[0,1]} \E (\hat{v} - p\E w(x))^2\\
\geq& \min_{\hat{v}'}\max_{p\in[0,1]} \E (\hat{v}' - p\E w(x))^2 =  \min_{\hat{v}'}\max_{p\in[0,1]} (\E w(x))^2\E (\frac{\hat{v}'}{\E w(x)} - p)^2 \\
=& (\E w(x))^2 \min_{\hat{p}}\max_{p\in[0,1]} \E (\hat{p}-p)^2
\end{align*}
where $\hat{p}$ only sees the examples when $\mu(x) = \pi(x)$. By the same argument in Lemma~\ref{lem:deterministic}, we get a lower bound in the statement.
\end{proof}

\subsection{Complex version of the upper bound proof}

The idea is to use IPS to estimate the conditional expectation for $a_i\in A_i$ and use the oracle to estimate the other half for $a_i\notin A_i$. For short, we denote the importance weight at instance $i$, $\frac{\pi(a_i|x_i)}{\mu(a_i|x_i)} =: \rho_i$. Abuse the notation and use $A_i$ to denote the event over $\cD_x \times \mu$ such that $A_i$ is true if $a_i\in A_i$ and false otherwise. Moreover, denote the conditional likelihood ratio $ \frac{\pi(a_i|x_i,A_i)}{\mu(a_i|x_i,A_i)} :=\rho_i'$. \red{[CORRECTION] This should be $ \rho_i' :=\frac{p_{\pi}(a_i,x_i | A_i)}{p_{\mu}(a_i,x_i|A_i)} $ for the identity below to be correct.}

We first look at the population quantity and then the estimator will be a natural plug-in estimator. By law of total expectation and then apply the IPS estimator to the conditional distribution, we get
\begin{align*}
v^\pi =& \E_{\pi}r =  \E_\pi [ \E(r|x,a) | a\in A_x] \P_\pi(a\in A_x) + \E_\pi [\E(r|x,a) | a\in A_x^c] \P_\pi(a\in A_x^c)\\
=&\E_{\mu} [\E(r|x,a) \rho'_{x,a}| a\in A_x] \P_\pi(a\in A_x) + \E_\pi [\E(r|x,a) | a\in A_x^c] \P_\pi(a\in A_x^c).
\end{align*}
The first term cannot be estimated directly because\red{we do not know $\rho'_{x,a}$, despite the complete knowledge of $\pi$ and $\mu$.} However, we can It turns out that
$$
\E_{\mu} [r_a \rho'_{x,a}| a\in A_x] \P_\pi(a\in A_x)=\E_{\mu}[r_a \rho'_{x,a} 1_{a\in A_x}]\frac{\P_\pi(a\in A_x)}{\P_\mu(a\in A_x)}=\red{\E_{\mu}[r_a\rho_{x,a}1_{a\in A_x}]}.
$$
We can thus estimated it by the finite sum over the observed data.

For the second term, we can estimate the conditional expectation directly using the oracle $\hat{r}$. This involves calculating $\sum_{a\in A_i^c} \hat{r}(x_i,a)\pi(a|x_i,A_i^c)$ for each observed $x_i$ and take the average. On the other hand, we can estimate $\P_\pi(a\in A_x^c)$ using the same data too. Putting these together, we have the following oracle-assisted estimator
$$
\hat{v}  = \frac{1}{n}\sum_{i=1}^n \left[r_i \rho_i 1_{\{a_i\in A_i\}}\right]  +
\left[\frac{1}{n}\sum_{i=1}^n\sum_{a\in A_i^c} \hat{r}(x_i,a)\pi(a|x_i,A_i^c)\right]
 \left[\frac{1}{n}\sum_{i=1}^n \P_\pi(a\in A_i^c|x_i)\right].
$$

The estimator is a little unsatisfactory because the second term involves a product of two things that are not independent to each other. To simplify the analysis, we define a different estimator based on leave-one-out estimate of the probability.
$$
\hat{v}_{OA}  = \frac{1}{n}\sum_{i=1}^n \left[r_i \rho_i 1_{\{a_i\in A_i\}}\right]  +
\left[\frac{1}{n}\sum_{i=1}^n\sum_{a\in A_i^c} \hat{r}(x_i,a)\pi(a|x_i,A_i^c) \frac{1}{n-1}\sum_{j\neq i} \P_\pi(a\in A_j^c|x_j)\right].
$$
Now let us analyze this estimator.
\begin{theorem}\label{thm:MSEbound}
Let $A_x$ be such that for any $a\in A_x$, $\rho = \pi(a|x)/\mu(a|x) \leq \tau$. Denote $\pi' := \pi(a|x,a\in A_x^c)$, and $\epsilon(a,x):= \hat{r}(a,x)-\E(r|a,x)$. Furthermore, assume $R_{\max}$ does not depend on $x,a$. Then for every $n =\Omega(\tau\P_{\pi}(a\in A_x^c)R_{\max}^2\log^2(n))$, we have
  $$
  \mathrm{MSE}(\hat{v}_{\mathrm{\OA}}) = O\left(\frac{R_{\max}^2\log^2(n)}{n} + \frac{\E_\mu \rho^2 \sigma^2 1_{\{a\in A_x\}} +  R_{\max}^2\E_\mu\rho^21_{\{a\in A_x\}}}{n}+  \left|\E_{\pi'} \epsilon(a,x) \right|^2 \P_\pi(a\in A_x^c)^2 \right).
  $$
\end{theorem}

\begin{proof}[Proof of Theorem~\ref{thm:MSEbound}]
  We first decompose the MSE to bias and variance.

  It is not hard to show that the bias is only on the second term and it is
\begin{equation}\label{eq:bias_bound_oa}
\left|\E\hat{v}' - v^\pi\right|  = \left|\E  \sum_{a=1}^K  (\hat{r}(x,a)-\E(r|x,a))\pi(a|x,a\notin A_x^c)\right| \P_\pi(a\in A_x^c) =  \left|\E_{\pi'} \epsilon(a,x_i) \right| \P_\pi(a\in A_x^c),
\end{equation}
where . The expectation is taken over $x\sim \cD_x$. This is potentially important because $\hat{r}$ might not behave uniformly well on all $x$.

By the law of total variance,
\begin{align*}
&\Var(\hat{v}_{\mathrm{\OA}}) = \E_\mu \Var(\hat{v}_{\mathrm{\OA}} | x_{1:n},a_{1:n}) + \Var_\mu \E(\hat{v}_{\mathrm{\OA}} | x_{1:n},a_{1:n})
= \underbrace{\E_\mu \Var\left(\frac{1}{n}\sum_{i=1}^n \left[r_i \rho_i 1_{\{a_i\in A_i\}}\right] \middle| x_{1:n},a_{1:n}\right)}_\text{Part (a)} \\
&+ \Var_\mu \left\{ \underbrace{\E\left(\frac{1}{n}\sum_{i=1}^n \left[r_i \rho_i 1_{\{a_i\in A_i\}}\right] \middle| x_{1:n},a_{1:n}\right)}_\text{Part (b)}  +  \underbrace{\left[\frac{1}{n}\sum_{i=1}^n\sum_{a\in A_i^c} \hat{r}(x_i,a)\pi(a|x_i,A_i^c) \frac{1}{n-1}\sum_{j\neq i} \P_\pi(a\in A_j^c|x_j)\right]}_\text{Part (c)}\right\}.
\end{align*}
Note that the two expressions in the second line are not conditional independent, but by Cauchy-Schwartz inequality, we have for any two random variables $X,Y$
$$\Var(X+Y) \leq \Var(X) + \Var(Y) + 2\sqrt{\Var(X)\Var(Y)} \leq 2 \Var(X) + 2\Var(Y).$$

We now bound the Part (a) and variance of the Part (b) by direct calculation, and then control the variance of Part (c) using concentration inequalities and union bound (since the expectations are bounded random variables).
\begin{equation}\label{eq:Evar_bound_oa}
\E_\mu \Var\left(\frac{1}{n}\sum_{i=1}^n \left[r_i \rho_i 1_{\{a_i\in A_i\}}\right] \middle| x_{1:n},a_{1:n}\right) \leq \frac{1}{n^2}\sum_{i=1}^n\E_{\mu}\rho_i^2\sigma_i^2 1_{\{a_i\in A_i\}} = \frac{1}{n}\E_{\mu}\rho^2\sigma^2 1_{\{a\in A_x\}}.
\end{equation}
\begin{equation}\label{eq:varE_bound_partb}
\Var_\mu \E\left(\frac{1}{n}\sum_{i=1}^n \left[r_i \rho_i 1_{\{a_i\in A_i\}}\right] \middle| x_{1:n},a_{1:n}\right) \leq \frac{1}{n} \E_\mu \rho^2(\E r|x,a)^2 1_{\{a\in A_x\}} \leq \frac{R_{\max}^2}{n} \E_\mu \rho^2 1_{\{a\in A_x\}}.
\end{equation}

It remains to deal with the variance of Part (c), which are not sum of independent random variables due to the leave-one-out construction. Part (c) however, contains only $n+1$ pre-specified empirical means which allows us to take the easy route through Hoeffding's inequality and union bound. First we verify the boundedness. For any $x_i$, we have $\hat{r}(x_i,a)\pi(a|x_i,A_i^c) \leq R_{\max}$. Moreover, $\P_{\pi}(a\in A_j^c | x_j) \leq 1$.

Now let $Y_1,...,Y_n$ and $Z_1,...,Z_n$ be random variables, then we have the following inequality
\begin{align*}
\frac{1}{n}\sum_{i=1}^n Y_i Z_i - \frac{1}{n}\sum_{i=1}^n \E Y_i \E Z_i &= \frac{1}{n}\sum_{i=1}^n Y_i (Z_i-\E Z_i) + \E Z_1 (\bar{Y}- \E \bar{Y}) \\
&\leq  \max_j(|Z_j-\E Z_j|) (\bar{Y}-\E\bar{Y}) +  \max_j(|Z_j-\E Z_j|) \E\bar{Y} +  |\bar{Y}- \E \bar{Y}|\E Z_1.
\end{align*}
Take $Z_i = \frac{1}{n-1}\sum_{j\neq i} \P_{\pi}(a\in A_j^c | x_j)$ and $Y_i = \sum_{a\in A_i^c} \hat{r}(x_i,a)\pi(a|x_i,A_i^c)$. $\E Z_i Y_i = \E_{\pi'}\hat{r}(x,a) \P_{\pi}(a\in A_x^c)$. Substitute into the above inequality, we get
\begin{align*}
&\left| \frac{1}{n}\sum_{i=1}^n\sum_{a\in A_i^c} \hat{r}(x_i,a)\pi(a|x_i,A_i^c) \frac{1}{n-1}\sum_{j\neq i} \P_\pi(a\in A_j^c|x_j) - \E_{\pi'}\hat{r}(x,a) \P_{\pi}(a\in A_x^c)\right| \\
\leq&  \left|\frac{1}{n}\sum_{i=1}^n \sum_{a\in A_i^c} \hat{r}(x_i,a)\pi(a|x_i,A_i^c)   -  \E_{\pi'}\hat{r}(x,a)\right|\max_j \left| \frac{1}{n-1}\sum_{i\neq j}\P_{\pi}(a\in A_i^c|x_i) - \P_{\pi}(a\in A_x^c)\right| \\
&+
\E_{\pi'}\hat{r}(x,a) \max_j  \left| \frac{1}{n-1}\sum_{i\neq j}\P_{\pi}(a\in A_i^c|x_i) - \P_{\pi}(a\in A_x^c)\right|\\
&+ \P_{\pi}(a\in A_x^c|x) \left|\frac{1}{n}\sum_{i=1}^n \sum_{a\in A_i^c} \hat{r}(x_i,a)\pi(a|x_i,A_i^c)   -  \E_{\pi'}\hat{r}(x,a)\right|.
\end{align*}
We will control the third and fourth line using Hoeffding's inequality, which will be both be $O(1/\sqrt{n})$ term. The second line is a $O(1/n)$ term and under the condition on $n$ being sufficiently large in the statement, it is strictly smaller than either the third or the fourth line. Applying Hoeffding's inequality (Lemma~\ref{lem:hoeffding}) to each of the predefined $n+1$ summations, we get that with probability $1-\delta$,
\begin{align*}
 |\text{Part (c)}| \leq \frac{3\max\{\E_{\pi'} \hat{r}(x,a),\P_\pi(a\in A_x^c)R_{\max}\}\log\frac{2(n+1)}{\delta}}{\sqrt{2n}}\leq \frac{3R_{\max}\log\frac{2(n+1)}{\delta}}{\sqrt{2n}}.
\end{align*}


Take $\delta = 1/n$ and use the fact that $\E(\hat{v}_{\mathrm{\OA}}|x_{1:n},a_{1:n})$ is bounded by $R_{\max}$, we get
\begin{align}
\Var_\mu \E(\hat{v}_{\mathrm{\OA}}|x_{1:n},a_{1:n}) &\leq (1-\delta)\frac{9\R_{\max}^2 \log^2(2n(n+1))}{2n}  +  \delta \R_{\max}^2 = O\left(\frac{\R_{\max}^2\log^2(n)}{n}\right).\label{eq:varE_bound_partc}
\end{align}
Combine \eqref{eq:bias_bound_oa}\eqref{eq:Evar_bound_oa}\eqref{eq:varE_bound_partb} and \eqref{eq:varE_bound_partc} we get the stated MSE upper bound.
\end{proof}

\blue{[Discussion] The following lemma is in fact how we arrived at the near optimal feasible solution. Although it does not seem like we need it any more.}
\begin{lemma}\label{lem:solve_QP}
	If $\mathrm{ess\,sup}_{a,x}\rho \leq \frac{\E_\mu \rho^2}{4\sqrt{\epsilon}}$, then the optimal value of
	\begin{equation}\label{eq:optim_KL}
	\begin{aligned}
	\minimize_{\Delta} \quad& \E_\mu (\Delta(x,a)^2)\\
	\mathrm{subject\;to}\quad & \E_{0.5+\Delta}v^\pi - \E_{0.5}v^\pi = 2\sqrt{\epsilon}\\
	& 0 \leq \Delta\leq 0.5
	\end{aligned}
	\end{equation}
	is
	$
	4\epsilon/(R_{\max}^2\E_\mu\rho^2).
	$
	When this condition is not true, there exists a set $\Omega \subset \cX\otimes \cA$ that converges to $\emptyset$ as $\epsilon$ gets smaller such that we can write the optimal value of $\eqref{eq:optim_KL}$ as
	$$
	\frac{(2\sqrt{\epsilon} - 0.5R_{\max}\P_\pi((x,a)\in\Omega))^2}{R_{\max}^2\E_\mu (\rho^2 1_{(x,a)\notin \Omega})}  + 0.25 \P_{\mu}((x,a)\in\Omega).
	$$
\end{lemma}
\begin{proof}
	Let $a$, $b$ and $c$ are functions with compatible domain as $f$ (infinite dimensional vectors) and $K$ be a scalar.
	\begin{align*}
	\minimize_{f}& \langle a,f^2\rangle\\
	\text{subject to: }&  \langle b,f\rangle = K\\
	&0\leq f \leq c
	\end{align*}
	The semantics are that functions $a\geq 0$ $b\geq 0, c\geq 0$, $K$ are positive constants.
	Note that if $a(x)=0$ or $b(x)=0$ for any $x$, we can choose $f(x)=0$. If $c(x)=0$ for some $x$ we much choose $f(x)=0$.
	For the part of domain where $a,b,c>0$ (for light notation, we do not explicitly write this restriction in the derivations that follows).
	
	The Lagrangian
	$$\cL(f,u_1,u_2,v) = \langle a,f^2\rangle + \langle u_1,-f\rangle + \langle u_2, f-c\rangle  + v (K-\langle b,f\rangle).$$
	Any optimal solution $f^*$ much obeys KKT conditions including:
	$$
	f^* = \frac{u_1-u_2 + v b  }{2a} = \frac{vb}{2a} + \tilde{u}_1 - \tilde{u}_2
	$$
	as well as primal, dual feasibility and complementary slackness. Note that the reformulation of the Lagrange multipliers $\tilde{u}_1 = u_1/(2a)$ and $\tilde{u}_2 = u_2/(2a)$ is without loss of generality.
	
	$v$ will specify the normalization such that the first constraint $\langle b,f \rangle = K$ is true, and $u_1$ and $u_2$ specifies the adjustment due to the box constraint.
	
	Specifically, in cases when $0<f^*<c$, $v^*=\frac{2K}{\int b^2/a dx}$ and we have
	$$
	f^* = \frac{b}{2a}  \frac{2K}{\int b^2/a dx} = \frac{bK}{a \int (b/a)^2 a dx}.
	$$
	The corresponding optimal objective value is
	$$K^2\frac{\int (b/a)^2 a dx}{(\int (b/a)^2 a dx)^2} = \frac{K^2}{\int (b/a)^2 a dx}.$$
	
	When the above does not obey $0<f^*<c$ if becomes slightly more complicated and we do not have a closed-form solution.
	
	The solution involves iteratively projecting $f$ to $f'$ inside the box constraint (which will result in $\langle b,f'\rangle<K$) and renormalizing the new $f'$ into $f^{''}$ by increasing $v$ such that $\langle b,f^{''}\rangle=K$ (this may result in $f^{''}(x)>c$ for some $x$). Since $v$ monotonically increase and the active set of $x$ that is bounded at $f(x)=c$ only gets larger, the limit of this iterative procedure will converge at a fixed $v^*>0$ and a fixed subset $\Omega$ of the domain that obeys
	$$
	v \int_{\Omega^c}\frac{b^2}{2a} dx + \int_{\Omega} cb dx = K.
	$$
	Therefore, the optimal solutions are
	\begin{align*}
	v^* &= \frac{K - \int_{\Omega} cb dx}{\int_{\Omega^c}\frac{b^2}{2a} dx}\\
	f^*(x) &=  \begin{cases}
	\frac{b(x)}{a(x)} \frac{K - \int_{\Omega} cb dx}{\int_{\Omega^c}(b/a)^2 a dx} & \text{ when }x\in \Omega\\
	c(x) & \text{ otherwise.}
	\end{cases}
	\end{align*}
	And the optimal objective value is
	$$
	\frac{(K-\int_{\Omega}cb dx)^2}{\int_{\Omega^c}(b/a)^2 a dx} + \int_\Omega c^2 a dx.
	$$
	
	It remains to substitute the semantics of our specific problem into this. In our case: $a= \cD_x\otimes \mu$, $b = R_{\max}\cD_x\otimes \pi$, $c\equiv 0.5$ and $K= 2\sqrt{\epsilon}$. So suppose $\mathrm{ess\,sup}_{a,x}\frac{2\rho\sqrt{\epsilon}}{R_{\max}^2\E_\mu \rho^2} < 0.5$ (which will be the case eventually as $n$ gets large, since we will choose $\epsilon = O(1/n)$), then the optimal objective value is $\frac{4\epsilon}{R_{\max}^2\E_\mu\rho^2}$. Otherwise, for any $\epsilon$, there is a set $\Omega\subset \cX\otimes \cA$ such that the optimal value is
	$$
	\frac{(2\sqrt{\epsilon} - 0.5R_{\max}\P_\pi((x,a)\in\Omega))^2}{R_{\max}^2\E_\mu (\rho^2 1_{(x,a)\notin \Omega})}  + 0.25 \P_{\mu}((x,a)\in\Omega).
	$$
	This completes the proof.
\end{proof}

We first prove \Thm{lowerbound:rmax} because
\Thm{lowerbound:sigma} uses a very similar argument.

\begin{proof}[Proof of \Thm{lowerbound:rmax}]
We prove by reducing to a hypothesis testing (or binary classification) problem.
	
Let $P$ be a joint distribution over $(x,a,r)$ within the class of distribution $\cP = \{ \cD_x\otimes\mu_{a|x}\otimes r(a,x)\}$ where $r(a,x)$ is a deterministic map from $\cX\times \cA$ into $[0,R_{\max}]$. Since our problem is specified for each fixed $\cD_x$ and $\mu$, every $r(a,x)$ uniquely defines such a distribution $P$. In addition, we denote a class prior distributions over $\cP$, which can be parameterized by a parameter vector $\theta \in \Theta$. In particular, we will choose this prior distributions to be Bernoulli for each pair $(x,a)$ so $\theta: \cX\times \cA \rightarrow [0,1]$.

Let $\hat{v}$ be any function $(\cX,\cA,\R)^n \rightarrow \R$ and denote $\ell_P(\hat{v})=(\hat{v} - v^\pi)^2$. Note that $v^\pi$ is a function of $P$ that is why $\ell_P$ is parameterized by $P$.
We also define $\ell_\theta(\hat{v}) =(\hat{v} - \E_{\theta}v^\pi)^2$.

{\color{red}
	By the ``triangular inequality'' of square distance, we have
	\begin{equation}\label{eq:tranform_loss_finite}
	\E_\theta \E(\ell_P(\hat{v})|P) \geq \frac{1}{2} \E_\theta\E(\ell_\theta(\hat{v})|P) - \E_\theta(v^\pi - \E_\theta v^\pi)^2.
	\end{equation}
When specified assumptions in are satisfied, namely, $|\cX| = N$, the joint context and action distribution $\P_{\cD_x}(x) \leq C/N$ for any $x\in \cX$.
Substitute our result in Lemma~\ref{lem:transform_loss_finite} into \eqref{eq:tranform_loss_finite}, we get
\begin{equation}\label{eq:transform_loss_finite2}
\E_\theta \E(\ell_P(\hat{v})|P) \geq \frac{1}{2} \E_\theta\E(\ell_\theta(\hat{v})|P) - \frac{C\log(N)\E_{\mu}R_{\max}^2\rho^2}{N}.
\end{equation}
}


By the fact that maximum is bigger than expectation and then \eqref{eq:transform_loss_finite2}, we get:
\begin{equation}\label{eq:bayes-minimax}
\sup_{P\in\cP} \E(\ell_P(\hat{v})|P) \geq \sup_{\theta\in\Theta}\E_{\theta} \E(\ell_P(\hat{v})|P) \geq \frac{1}{2}\sup_{\theta\in\Theta}\E_\theta\E(\ell_\theta(\hat{v})|P) - \frac{C\log(N)\E_{\mu}R_{\max}^2\rho^2}{N}.
\end{equation}
Now the problem transforms into finding a lower bound for $\sup_{\theta\in\Theta}\E_\theta\E(\ell_\theta(\hat{v})|P)$. By Markov's inequality on each $\theta\in\Theta$, we get:
\begin{align}\label{eq:markov}
\sup_{\theta\in\Theta}\E_{\theta} \E(\ell_\theta(\hat{v})|P) \geq \epsilon \sup_{\theta \in\Theta} \P_{\theta}(\ell_\theta(\hat{v}) \geq \epsilon).
\end{align}
It follows that for any two $\theta_1,\theta_2\in\Theta$
\begin{equation}\label{eq:class_error}
\sup_{\theta \in\Theta} \P_{\theta}(\ell_\theta(\hat{v}) \geq \epsilon) \geq \max_{i=1,2}\P_{\theta_i}(\ell_{\theta_i}(\hat{v})\geq \epsilon) \geq \frac{1}{2}\left(\P_{\theta_1}(\ell_{\theta_1}(\hat{v})\geq \epsilon)+\P_{\theta_2}(\ell_{\theta_2}(\hat{v})\geq \epsilon)\right).
\end{equation}

Now define a hypothetical binary classifier $\phi(\ell(\hat{v}))$ which outputs $1$ if $ \ell_{\theta_1}(\hat{v}) \leq  \ell_{\theta_2}(\hat{v})$ and $2$ otherwise. Suppose we choose $\theta_1,\theta_2$ such that $\E_{\theta_1}v^\pi - \E_{\theta_2}v^\pi \geq 2\sqrt{\epsilon}$. By the ``triangular inequality'' of square error,
$$
4\epsilon \leq (\E_{\theta_1}v^\pi - \E_{\theta_2}v^\pi)^2\leq 2 (\hat{v} - \E_{\theta_1}v^\pi)^2 + 2(\hat{v}-\E_{\theta_2}v^\pi)^2 = 2\ell_{\theta_1}(\hat{v}) + 2 \ell_{\theta_2}(\hat{v}).
$$
We argue that if
$ \ell_{\theta_1}(\hat{v})< \epsilon$ then $\ell_{\theta_2}(\hat{v})\geq \epsilon$ and the classifier $\phi$ returns $1$; similarly, if $\ell_{\theta_2}(\hat{v})<\epsilon$, the classifier returns $2$. As a result, $\P_{\theta_1}(\ell_{\theta_1}(\hat{v})\geq \epsilon)$ and $\P_{\theta_2}(\ell_{\theta_2}(\hat{v})\geq \epsilon)$ are upper bounds of the misclassification errors when the correct answer is $1$ and $2$ respectively. This classifier however must perform worse than the Bayes error rate and this gives us the a lower bound.

In particular, using Le Cam's argument, the Bayes error rate obeys
\begin{equation}\label{eq:leCam}
\int dP_{\theta_1}^n \wedge d P_{\theta_2}^n \geq \frac{1}{2}e^{-n D_{KL}(P_{\theta_1}\|P_{\theta_2})}.
\end{equation}
In our problem, the KL-divergence obeys
\begin{align*}
D_{KL}(P_{\theta_1}\|P_{\theta_2}) &= \int p(r;\theta_1(x,a)) p(x)\mu(a|x) \log \frac{p(r;\theta_1(x,a))}{p(r;\theta_2(x,a))} dr da dx \\
&= \E D_{KL}(\text{Ber}(\theta_1(x,a)\|\text{Ber}(\theta_2(x,a)))).
\end{align*}
Let $\theta_2(x,a)\equiv 0.5$ and $\theta_1(x,a)=\theta_2(x,a)+\Delta(x,a)$, by Lemma~\ref{lem:bernoulli_KL} we obtains an upper bound
$$
D_{KL}(P_{\theta_1}\|P_{\theta_2})  \leq \frac{1}{4} \E_\mu \Delta(x,a)^2.
$$
We want to find the largest possible $\epsilon$ such that $\frac{1}{4} \E_\mu \Delta(x,a)^2 \leq \frac{\log 2}{n}$.
Denote $\varepsilon:=2\sqrt{\epsilon}$, this problem can be formulated as
\begin{align}
\maximize_{\Delta}\quad &  \varepsilon \nonumber\\
\text{subject to:}\quad& \E_{\theta_2 + \Delta} v^\pi - \E_{\theta_2} v^\pi \geq \varepsilon\nonumber\\
&\frac{1}{4} \E_\mu \Delta(x,a)^2 \leq \frac{\log 2}{n}\label{eq:upperbound_KL}\\
& 0 \leq \Delta \leq 0.5\nonumber
\end{align}
Every feasible solution $(\varepsilon,\Delta)$ will yield a lower bound of interest. The idea is to construct one that is nearly optimal so that we have a strong lower bound.

Let $K$ be a non-negative scalar and take
$$\Delta = \min\left\{\frac{\rho R_{\max}  K}{\E_\mu\rho^2R_{\max}^2}, 0.5\right\}.$$
Due to the non-negativity of $\Delta$ and that $\min\{\cdot,0.5\} \leq 0.5$, the third constraint holds automatically for any $K$. When we take $K = \sqrt{4\log 2 \E_\mu\rho^2R_{\max}^2/n}$, we can check that $\Delta' = \frac{\rho R_{\max}  K}{\E_\mu\rho^2R_{\max}^2}$ satisfies the second constraint. Since $\Delta \leq \min\{\Delta',0.5\}$, $\Delta$ also satisfies the second constraint.


A feasible $\varepsilon$ is therefore
\begin{align}
\varepsilon &= \E_{\theta_2 + \Delta} v^\pi - \E_{\theta_2} v^\pi = K - K\frac{\E_{\pi} \left[\rho R_{\max} 1_{\{\rho R_{\max} > \E_\mu(\rho^2R_{\max}^2)/(2K)\}}\right]}{\E_\pi\rho R_{\max}} \nonumber\\
&=\sqrt{\frac{4\log 2 \E_\mu\rho^2R_{\max}^2}{n}} \left[1-\frac{\E_{\pi} \left(\rho R_{\max} 1_{\left\{\rho R_{\max}> \sqrt{n \E_\mu (\rho^2\R_{\max}^2)/(16\log 2)}\right\}}\right)}{\E_\pi\rho R_{\max}}\right].\label{eq:varepsilon}
\end{align}

Collecting the results in \eqref{eq:markov}\eqref{eq:class_error}\eqref{eq:leCam}\eqref{eq:upperbound_KL} and \eqref{eq:varepsilon}, we have
\begin{align*}
\sup_{\theta\in\Theta}\E_{\theta} \E(\ell_\theta(\hat{v})|P) &\geq \frac{\epsilon}{2}e^{-n D_{KL}(P_{\theta_1}\|P_{\theta_2})} \geq \frac{\varepsilon^2}{8} e^{-\log 2} \\
&= \frac{\log 2 \E_\mu\rho^2R_{\max}^2}{4n} \left[1-\frac{\E_{\pi} \left(\rho R_{\max} 1_{\left\{\rho R_{\max}> \sqrt{n \E_\mu (\rho^2\R_{\max}^2)/(16\log 2)}\right\}}\right)}{\E_\pi\rho R_{\max}}\right]^2.
\end{align*}
The proof of the first case is complete by substituting into \eqref{eq:bayes-minimax} and noting that this bound holds for any estimator $\hat{v}$.
\end{proof}

\begin{proof}[Proof of \Thm{lowerbound:sigma}]
	The proof uses nearly the same argument as in \Thm{lowerbound:rmax}, we can reduce the problem to a binary classification problem which tells the difference of two Gaussian means. The argument yields a slightly sharper bound (constant improvement) over \citet{li2015toward}.
	
	Let $P$ be the class of distributions specified by $\mu$, $\cD_{x}$, $\sigma^2$ and $R_{\max}$ as before.
	Define define the class of distributions that are normals with mean vector $\theta$ and variance $\sigma^2$.
	Denote $P_\theta = \cD_x \otimes \mu_{a|x}\otimes \cN(\theta(x,a),\sigma^2) \in \cP_\Theta$. Clearly, $\cP_{\Theta}\subset \cP$. Using the same chain of arguments (it's actually simpler as we do not need to deal with the differences of $\ell_P$ and $\ell_\theta$.)
	\begin{align*}
	\sup_{P\in \cP}\ell_P(\hat{v}) &\geq \sup_{\theta\in\Theta}\ell_{P_\theta}(\hat{v}) \geq \epsilon \sup_{\theta\in\Theta} \P_{\theta}(\ell_{P_\theta}\geq \epsilon)\\
	&\geq \epsilon \max_{i=1,2} \P_{\theta_i}(\ell_{P_{\theta_i}}\geq \epsilon)\\
	&\geq \frac{\epsilon}{2} \left[\P_{\theta_1}(\ell_{P_{\theta_i}}\geq \epsilon)+\P_{\theta_2}(\ell_{P_{\theta_i}}\geq \epsilon)\right]\\
	&\geq \epsilon \int dP_{\theta_1}^n\wedge dP_{\theta_2}^n \geq \frac{\epsilon}{2}e^{-n D_{KL}(P_{\theta_1}\|P_{\theta_2})}
	\end{align*}
	The inequality in the last line holds if $\Delta = \theta_2-\theta_1$ are chosen such that $v^\pi_{\theta_2} - v^\pi_{\theta_1} = \E_{\pi}\Delta \geq 2\sqrt{\epsilon}.$
	
	By the formula for KL-divergence of two Gaussians
	$$
	D_{KL}(P_{\theta_1}\|P_{\theta_2}) = \E_\mu \frac{\Delta^2}{2\sigma^2}.
	$$
	
	Denote $\varepsilon:=2\sqrt{\epsilon}$, it remains to find a (nearly optimal) feasible solution to:
	\begin{align*}
	\maximize_{\Delta}\quad &  \varepsilon \nonumber\\
	\text{subject to:}\quad& \E_{\pi} \Delta(x,a) \geq  \varepsilon\nonumber\\
	&\E_\mu \frac{\Delta(x,a)^2}{2\sigma^2} \leq \frac{\log 2}{n}\\
	& 0 \leq \Delta \leq R_{\max}\nonumber
	\end{align*}
	Take $\Delta = \min\left\{ \frac{\rho \sigma^2 K}{\E_\mu (\rho^2\sigma^2)}, R_{\max} \right\}$. Check that it satisfies the third inequality automatically and satisfies the second when we choose $K = \sqrt{2\log 2 \E_\mu \rho^2\sigma^2 / n}$.
	
	Let $\mathrm{Id}(\cdot)$ be the identity operator.
	Denote $\Delta = \frac{\rho \sigma^2 K}{\E_\mu (\rho^2\sigma^2)}  - \left[\mathrm{Id}(\cdot) - \min\{\cdot, R_{\max}\}\right] (\frac{\rho \sigma^2 K}{\E_\mu (\rho^2\sigma^2)})$ and for the first inequality to hold, we can choose
	$$
	\varepsilon = K  -\E_{\pi}\left[ \frac{\rho \sigma^2 K}{\E_\pi\rho\sigma^2} 1_{\{\frac{\rho \sigma^2 K}{\E_\pi (\rho\sigma^2)} > R_{\max}\}}\right]= \sqrt{\frac{2\log 2 \E_\pi\rho\sigma^2}{n}} \left[1- \frac{\E_\pi \rho\sigma^2 1_{\{\rho\sigma^2>R_{\max}\sqrt{n\E_\pi\rho\sigma^2/(2\log 2)}\}}}{\E_\pi\rho\sigma^2}\right].
	$$
	In summary, we get
	$$
	\sup_{P\in \cP}\ell_P(\hat{v}) \geq \frac{\log 2 \E_\pi\rho\sigma^2}{8n} \left[1- \frac{\E_\pi \rho\sigma^2 1_{\{\rho\sigma^2>R_{\max} \sqrt{n\E_\pi\rho\sigma^2/(2\log 2)}\}}}{\E_\pi\rho\sigma^2}\right]^2
	$$
	as required.
\end{proof}

\begin{proof}[Proof of Remark~\ref{rmk:correction_conv}]
	By completing the square we know $(1-a)^2 \geq 1-2a$ so it remains to upper bound $a$.
	$a$ in both $C_1(n)$ and $C_2(n)$ can be upper bounded by the following form with a random variable $X$ and a constant $b$.
	$$
	\frac{\E(X 1_{\{X>\sqrt{n}b\}})}{\E(X)}.
	$$
	In $C_1(n)$, $X=\rho \sigma^2$, $b=c\sqrt{\E_{\mu}(\rho^2\sigma^2)/(2\log 2)}$ and in $C_2(n)$, $X=\rho R_{\max}$, $b = \sqrt{\E_{\mu}(\rho^2R_{\max}^2)/(16\log 2)}$.
	
	When we have $1/n^\alpha$ tail probability
	$$
	\E(X 1_{\{X>\sqrt{n}b\}})  =  \int_{\sqrt{n}b}^\infty X  dP(X) \leq \int_{\sqrt{n}b}^\infty  Y dP(Y) = \int_{\sqrt{n}b}^\infty y^{-1-\alpha} dy = O((\sqrt{n}b)^{-\alpha}).
	$$
	where $Y$ is a scaled Pareto random variable. Similarly, when we have a subexponential distribution
	$$
	\E(X 1_{\{X>\sqrt{n}b\}})  = O(\sqrt{n}b e^{-\sqrt{n}b}).
	$$
\end{proof}

This is precisely because it is model-free --- $\E (r|x,a)$ can be
completely unstructured (the class of functions has a constant
Rademacher complexity). Arguably, all interesting problems in practice
that we care about should have meaningful structures. For instance,
$\E(r|x,a)$ and $\E (r|x',a)$ should be similar whenever $x$ and $x'$
are close to each other. If we can somehow estimate $\E( r|x,a)$
directly using $f(x,a) \in \cF$ for some function space $\cF$, then we
can hope to construct a direct estimator for the value of any test
policy $\pi$ Of course, the immediate drawback of this approach is
that if $f$ is a poor proxy of $\E(r|x,a)$, the estimator is going to
be highly biased. To guard against that in practice, it makes sense to
somehow use it together with the importance sampling estimator. One
popular approach of doing this is through the doubly robust estimator
\citep{dudik2011doubly,dudik2014doubly,jiang2016doubly}
\begin{equation}\label{eq:DR}
\hat{v}^\pi_{\text{DR}} = \frac{1}{n} \sum_{i=1}^n \left[ \frac{(r_i-\hat{v}^{a_i}_{\text{DM}}) \pi(a_i)}{\mu(a_i)}  + \hat{v}^\pi_{\text{DM}}\right].
\end{equation}
The estimator has the feature that it remains unbiased and it often has lower variance than that of $\hat{v}_{\text{IS}}$.

\subsection{Why is doubly robust estimator unsatisfactory?} \label{sec:doublyrobust}
However, in many cases, the doubly robust estimator does not make effective use of a direct method, even we have an optimal or even perfect direct estimator.

\paragraph{DR in Multi-arm bandits.} Let us first take a step back and consider the simpler multi-arm bandits setting.
As we discussed earlier, the regression estimator \eqref{eq:reg} is asymptotically optimal but has poor finite sample dependence. On the other hand, the IPS sampling is suboptimal, but it has smaller MSE when the number of samples is small. The doubly robust estimator that takes $\hat{v}_{\text{DM}} = \hat{v}_{\text{REG}}$
is supposedly able to combine the benefits of the two. Here it is also assumed that $\hat{v}_{\text{Reg}}$ is independent of the data set (e.g., random splitting the data into two, one of them for estimating $\hat{v}_{\text{Reg}}$, the other for calculating the doubly robust estimator).

The doubly robust estimator is unbiased because it essentially uses the IPS estimator to estimate and then offset the bias in expectation. Its variance, according to \citet[Lemma~3.3(i)]{dudik2014doubly} (when the context is a singleton, and $\mu$ is known exactly), is
\begin{equation}\label{eq:varianceDR}
\text{MSE}(\hat{v}_{\text{DR}}) = \frac{\sigma^2}{n}\E_\pi \frac{\pi(a)}{\mu(a)}  +  \E_{\pi}  \frac{\pi(a)}{\mu(a)} (\hat{v}^{a}_{\text{Reg}}-\E_{\pi} r(a))^2  - (\E_{\pi}\hat{v}^{a}_{\text{Reg}}-\E_\pi r(a))^2.
\end{equation}
If we further take expectation over $\hat{v}^{a}_{\text{Reg}}$, it reduces to calculating the MSE of estimating $\E r(a)$ conditioned on each chosen $a$. Denote the regression estimator of reward for action $a$ for short by $\hat{r}$ and the true expected reward by $r$. The bias is
$$
b_n = r (1-\mu(a))^n.
$$
The bias converges to $0$ exponentially as $n\rightarrow \infty$.
The variance can be calculated by the law of total variance
\begin{align*}
v_n =&  \Var \E(\hat{r}|n_a) + \E \Var(\hat{r}|n_a)\\
=&(1-\mu(a))^n(1-(1-\mu(a))^n)r^2 + (1-\mu(a))^n\times 0 + (1-(1-\mu(a))^n)\sum_{i=1}^n\frac{\sigma^2}{i} \P(n_a=i)\\
=&(1-\mu(a))^n(1-(1-\mu(a))^n)r^2 + (1-(1-\mu(a))^n)\sum_{i=1}^n\frac{\sigma^2}{i} {n\choose i} \mu(a)^i(1-\mu(a))^{n-i}.
\end{align*}

Assume $\mu(a)>0$, and $r <\infty$. As $n\rightarrow \infty$, the first term in $v_n \rightarrow 0$ exponentially, by the continuous mapping theorem, the second term converges to $\frac{\sigma^2}{\E n_a } = \frac{\sigma^2}{\mu(a)n}$. As a result, asymptotically,
$$n \E(\hat{r}-r)^2  \rightarrow  \frac{\sigma^2}{\mu(a)}.$$
Substitute into \eqref{eq:varianceDR} and take limit, we get
$$
\lim_{n\rightarrow \infty} n \text{MSE}(\hat{v}_{\text{DR}}) = \sigma^2  \left[\sum_{a} \frac{\pi(a)^2}{\mu(a)} + \frac{\pi(a)^2}{\mu(a)^2}\right].
$$
Therefore, unlike the plain regression estimator, the doubly robust estimator has an asymptotic gap from the Cramer-Rao lower bound, even if we allow the regression estimator to be obtained for free with an additional $n$ fresh data points.

Also, comparing to IPS, there is a poorer dependence on $\mu$, which makes $\hat{v}_{DR}$ worse than either $\hat{v}_{IPS}$ or $\hat{v}_{Reg}$ in many interesting regimes. This echoes and in some sense formalizes the empirical observation in \citet{kang2007demystifying}, even in cases when the regression model is correctly specified. It is however worth noting that the dependence on $R_{\max}$ in $\hat{v}_{IPS}$ is completely removed by DR at least asymptotically, making it an attractive choice when $\sigma^2\ll R_{\max}^2$ and when the importance weights are relatively small.

\paragraph{DR with perfect oracle.}
More generally, we can show that even if we assume that we have a perfect direct estimator of the expected reward, the doubly robust estimator is still going to have a large variance.
\begin{proposition}
Let define the direct method estimator \eqref{eq:DM} using and oracle $f(x,a) = \E( r|x,a)$ and construct doubly robust estimator \eqref{eq:DM} using the this perfect direct estimator. Then
\begin{align*}
&\mathrm{MSE}(\hat{v}_{\mathrm{DM}}) = \frac{1}{n} \Var \left[\sum_{a\in\cA} \pi(a|x)\E(r|x,a)\right] \leq \frac{1}{n}\E[\E_\pi(R_{\max}^2|x)],\\
&\mathrm{MSE}(\hat{v}_{\mathrm{DR}})  = \frac{1}{n}\E_{\mu}\rho^2\Var(r|x,a)^2 + \frac{1}{n} \Var \left[\sum_{a\in\cA} \pi(a|x)\E(r|x,a)\right] \leq \frac{1}{n}(\E_\mu \rho^2\sigma^2 +\E[\E_\pi(R_{\max}|x)^2]).
\end{align*}
\end{proposition}
\begin{proof}
The variance of direct method can be calculated from \eqref{eq:DM} itself. For that of the doubly robust estimator, substituting $f(x,a)=\E r|x,a$ into Lemma~3.1(i) of \citet{dudik2014doubly} produces the stated expression.
\end{proof}
Note that the above upper bounds are tight, since
	\begin{equation}\label{eq:oracle_lowerbound}
	\sup_{r|x,a \in \cR(\sigma^2,R_{\max})} \frac{1}{n} \Var \left[\sum_{a\in\cA} \pi(a|x)\E(r|x,a)\right] = \Omega\left(\frac{\E[ \E_\pi(R_{\max}|x)^2]}{n}\right).
	\end{equation}
	This can be obtained by choosing a prior distribution of $r|x,a$ defined on $\cR(\sigma^2,R_{\max})$, such that for each $x$,
	$\E(r|x,\cdot)=\R_{\max}(x,\cdot)$ or $\E(r|x,\cdot)=0$, each with probability half.

DR has a smaller MSE than the IPS estimator and thus in some sense ``breaks'' the lower bound in the last section through the use of the oracle information. However, the DR still depends on $\E_\mu \rho^2\sigma^2$ therefore sensitive to large importance weights.
In addition, DR is clearly a suboptimal way of using the oracle, because using $\hat{v}_{\mathrm{DM}}$ directly achieves much smaller MSE\footnote{$\hat{v}_{\mathrm{DM}}$ is in fact optimal (up to a constant) for each $\cD_x,\pi$ separately, since it is the MLE and $\sum_{a} \pi(a|x)f(x,a)$ is bounded.}.

The same observation was made in \citet[Theorem 2]{jiang2016doubly} in the context of reinforcement learning, which states that even when a perfect value function is used, the variance of the DR estimator still depends on the second moment of the importance weight and the conditional variance.